\documentclass{article}

   \PassOptionsToPackage{numbers, compress}{natbib}
 \usepackage[final]{neurips_2020}

\usepackage[utf8]{inputenc} % allow utf-8 input
\usepackage[T1]{fontenc}    % use 8-bit T1 fonts
\usepackage{hyperref}       % hyperlinks
\usepackage{url}            % simple URL typesetting
\usepackage{booktabs}       % professional-quality tables
\usepackage{amsfonts}       % blackboard math symbols
\usepackage{nicefrac}       % compact symbols for 1/2, etc.
\usepackage{microtype}      % microtypography

\usepackage{amsmath}
\usepackage{macro_commands}
\usepackage{enumitem}
\usepackage{xcolor}

\usepackage{color}

\usepackage[compact]{titlesec}
\title{A Finite-Time Analysis of Two Time-Scale Actor-Critic Methods}

\author{%
  Yue Wu \\
  Department of Computer Science \\
  University of California, Los Angeles \\
  Los Angeles, CA 90095 \\
  \texttt{ywu@cs.ucla.edu}
  \And 
  Weitong Zhang \\
  Department of Computer Science \\
  University of California, Los Angeles \\
  Los Angeles, CA 90095 \\
  \texttt{weightzero@cs.ucla.edu}
  \And 
  Pan Xu \\
  Department of Computer Science \\
  University of California, Los Angeles \\
  Los Angeles, CA 90095 \\
  \texttt{panxu@cs.ucla.edu}
  \And
  Quanquan Gu \\
  Department of Computer Science \\
  University of California, Los Angeles \\
  Los Angeles, CA 90095 \\
  \texttt{qgu@cs.ucla.edu}
}

\begin{document}

\maketitle

\begin{abstract}
Actor-critic (AC) methods have exhibited great empirical success compared with other reinforcement learning algorithms, where the actor uses the policy gradient to improve the learning policy and the critic uses temporal difference learning to estimate the policy gradient. Under the two time-scale learning rate schedule, the asymptotic convergence of AC has been well studied in the literature.
However, the non-asymptotic convergence and finite sample complexity of actor-critic methods are largely open.
In this work, we provide a non-asymptotic analysis for two time-scale actor-critic methods under non-i.i.d. setting. We prove that the actor-critic method is guaranteed to find a first-order stationary point (i.e., $\|\nabla J(\bm{\theta})\|_2^2 \le \epsilon$) of the non-concave performance function $J(\bm{\theta})$, with $\mathcal{\tilde{O}}(\epsilon^{-2.5})$ sample complexity. 
To the best of our knowledge, this is the first work providing finite-time analysis and sample complexity bound for two time-scale actor-critic methods.
\end{abstract}

\section{Introduction}

Actor-Critic (AC) methods \citep{barto1983neuronlike,konda2000actor}
aim at combining the advantages of actor-only methods and critic-only methods, and have achieved great empirical success in reinforcement learning \citep{wang2016sample, bahdanau2016actor}.
Specifically, actor-only methods, such as policy gradient \citep{sutton2000policy} and trust region policy optimization \citep{schulman2015trust}, utilize a parameterized policy function class and improve the policy by optimizing the parameters of some performance function using gradient ascent, whose exact form is characterized by the Policy Gradient Theorem \citep{sutton2000policy}. Actor-only methods can be naturally applied to continuous setting but suffer from high variance when estimating the policy gradient. On the other hand, critic-only methods, such as temporal difference learning \citep{sutton1988learning} and Q-learning \citep{watkins1992q}, focus on learning a value function (expected cumulative rewards), and determine the policy based on the value function, which is recursively approximated based on the Bellman equation. Although the critic-only methods can efficiently learn a satisfying policy under tabular setting \citep{jin2018q}, they can diverge with function approximation under continuous setting \citep{wiering2004convergence}. Therefore, it is natural to combine actor and critic based methods to achieve the best of both worlds. The principal idea behind actor-critic methods is simple: the critic tries to learn the value function, given the policy from the actor, while the actor can estimate the policy gradient based on the approximate value function provided by the critic.
 
If the actor is fixed, the policy remains unchanged throughout the updates of the critic. Thus one can use policy evaluation algorithm such as temporal difference (TD) learning \citep{sutton2018reinforcement} to estimate the value function (critic). 
After many steps of the critic update, one can expect a good estimation of the value function, which in turn enables an accurate estimation of the policy gradient for the actor. 
A more favorable implementation is the so-called two time-scale actor-critic algorithm, where the actor and the critic are updated simultaneously at each iteration except that the actor changes more slowly (with a small step size) than the critic (with a large step size). In this way, one can hope the critic will be well approximated even after one step of update. From the theoretical perspective, the asymptotic analysis of two time-scale actor-critic methods has been established in \cite{borkar1997actor,konda2000actor}. 
In specific, under the assumption that the ratio of the two time-scales goes to infinity (i.e. $\lim_{t \rightarrow \infty} \beta_t / \alpha_t = \infty$), the asymptotic convergence is guaranteed through the lens of the two time-scale ordinary differential equations(ODE), where the slower component is fixed and the faster component converges to its stationary point. This type of analysis was also applied in the context of generic two time-scale stochastic approximation \citep{borkar1997stochastic}. 

However, finite-time analysis (non-asymptotic analysis) of two-time scale actor-critic is still largely missing in the literature, which is important because it can address the questions that how many samples are needed for two time-scale actor-critic to converge,  
and how to appropriately choose the different learning rates for the actor and the critic. 
Some recent work has attempted to provide the finite-time analysis for the ``decoupled'' actor-critic methods \citep{kumar2019sample,qiu2019finite}. 
The term ``decoupled'' means that before updating the actor at the $t$-th iteration, the critic starts from scratch to estimate the state-value (or Q-value) function. At each iteration, the ``decoupled'' setting requires the critic to perform multiple sampling and updating (often from another new sample trajectory). As we will see in the later comparison, this setting is sample-inefficient or even impractical.
Besides, their analyses are based on either the i.i.d.\ assumption \citep{kumar2019sample} or the partially i.i.d.\ assumption \citep{qiu2019finite} (the actor receives i.i.d.\ samples), which is unrealistic in practice.
In this paper, we present the first finite-time analysis on the convergence of the two time-scale actor-critic algorithm. 
We summarize our contributions as follows:

\begin{itemize}[leftmargin=*]
\item We prove that, the actor in the two time-scale actor critic algorithm converges to an $\epsilon$-approximate stationary point of the non-concave performance function $J$ after accessing at most $\tilde{\cO}(\epsilon^{-2.5})$ samples. Compared with existing finite-time analysis of actor-critic methods \citep{kumar2019sample,qiu2019finite}, the algorithm we analyzed is based on two time-scale update and therefore more practical and efficient than the ``decoupled'' version.
Moreover, we do not need any i.i.d. data assumptions in the convergence analysis as required by \citet{kumar2019sample,qiu2019finite}, which do not hold in real applications.

\item From the technical viewpoint, we also present a new proof framework that can tightly characterize the estimation error in two time-scale algorithms. Compared with the proof technique used in \cite{xu2019two}, we remove the extra artificial factor $\cO(t^{\xi})$ in the convergence rate introduced by their ``iterative refinement'' technique. Therefore, our new proof technique may be of independent interest for analyzing the convergence of other two time-scale algorithms to get sharper rates. 

\end{itemize}

% The remainder of this paper is organized as follows: In Section \ref{sec:related}, we review the related work in the literature. In Section \ref{sec:preliminary}, we introduce preliminaries and the two time-scale actor-critic algorithm. We analyze the algorithm in detail in Section \ref{subsec:alg}. In Sections \ref{sec:theory} and \ref{sec:proof} we present our main theory and the corresponding the proof sketch respectively. Finally, we conclude our paper in Section \ref{sec:conclusion}.

\noindent\textbf{Notation} We use lower case letters to denote scalars, and use lower and upper case bold face letters to denote vectors and matrices respectively. For two sequences $\{a_n\}$ and $\{b_n\}$, we write $a_n = \cO(b_n)$ if there exists an absolute constant $C$ such that $a_n\le C b_n$. We use $\tilde{\cO}(\cdot)$ to further hide logarithm factors. Without other specification, $\|\cdot\|$ denotes the $\ell_2$ norm of Euclidean vectors.
$d_{TV}(P,Q)$ is the total variation norm between two probability measure $P$ and $Q$, which is defined as $d_{TV}(P,Q) = 1/2 \int_{\cX}|P(dx)-Q(dx)|$.

\section{Related work} \label{sec:related}

In this section, we briefly review and discuss existing work, which is mostly related to ours.

\textbf{Stochastic bias characterization}
The main difficulty in analyzing reinforcement learning algorithms under non-i.i.d.\ data assumptions is that the samples and the trainable parameters are correlated, which makes the noise term biased. \citet{bhandari2018finite} used information-theoretical techniques to bound the Markovian bias and provide a simple and explicit analysis for the temporal difference learning. Similar techniques were also established in \cite{srikant2019finite} through the lens of stochastic approximation methods.  \citet{gupta2019finite,xu2019two} applied such methods to deriving the non-asymptotic convergence of two time-scale temporal difference learning algorithms (TDC). \citet{zou2019finite,chen2019performance,xu2019finite} further applied these analysis methods to on-policy learning algorithms including SARSA and Q-learning. In addition, \citet{hu2019characterizing} formulated a family of TD learning algorithms as  Markov jump linear systems and analyzed the evolution of the mean and covariance matrix of the estimation error. \citet{cai2019provably} studied TD learning with neural network approximation, and proved its global convergence.

\noindent\textbf{Two time-scale reinforcement learning}
The two time-scale stochastic approximation can be seen as a general framework for analyzing reinforcement learning \citep{borkar1997stochastic, tadic2003asymptotic, konda2004convergence}. 
Recently, the finite-time analysis of two time-scale stochastic approximation has gained much interest. \citet{dalal2017finite} proved convergence rate for the two time-scale linear stochastic approximation under i.i.d.\ assumption. \citet{gupta2019finite} also provided finite-time analysis for the two time-scale linear stochastic approximation algorithms. Both can be applied to analyze two time-scale TD methods like GTD, GTD2 and TDC. \citet{xu2019two} proved convergence rate and sample complexity for the TDC algorithm over Markovian samples.
\citep{kaledin2020finite} further improved the convergence rate of two time-scale linear stochastic approximation and removed the projection step.
However, since the update rule for the actor is generally not linear, we cannot apply these results to the actor-critic algorithms.

\noindent\textbf{Analysis for actor-critic methods}
The asymptotic analysis of actor-critic methods has been well established. \citet{konda2000actor} proposed the actor-critic algorithm, and established the asymptotic convergence for the two time-scale actor-critic, with TD($\lambda$) learning-based critic. \citet{bhatnagar2009natural} proved the convergence result for the original actor-critic and natural actor-critic methods. \citet{castro2010convergent} proposed a single time-scale actor-critic algorithm and proved its convergence. 
Recently, \citep{zhang2019provably} proved convergence of two time-scale off-policy actor-critic with function approximation. 
Recently, there has emerged some works concerning the finite-time behavior of actor-critic methods.
\citet{yang2019global} studied the global convergence of actor-critic algorithms under the Linear Quadratic Regulator. 
%\citet{liu2019neural} and 
\citet{yang2018finite} analyzed the finite-sample performance of batched actor-critic, where all samples are assumed i.i.d.\ and the critic performs several empirical risk minimization (ERM) steps. 
\citet{qiu2019finite} treated the actor-critic algorithms as a bilevel optimization problem and established a finite sample analysis under the ``average-reward'' setting, assuming that the actor has access to independent samples. 
Similar result has also been established by \citet{kumar2019sample}, where they considered the sample complexity for the ``decoupled'' actor-critic methods under i.i.d.\ assumption.  \citet{wang2020neural} also proved the global convergence of actor-critic algorithms with both actor and critic being approximated by overparameterized neural networks.

When we were preparing this work, we noticed that there is a concurrent and independent work \citep{xu2020non} which also analyzes the non-asymptotic convergence of two time-scale actor-critic algorithms and achieves the same sample complexity, i.e., $\mathcal{\tilde{O}}(\epsilon^{-2.5})$. However, there are two key differences between their work and ours. First, the two time-scale algorithms analyzed in both papers are very different. We analyze the classical two time-scale algorithm described in \citep{sutton2018reinforcement}, where both actor and critic take one step update in each iteration. It is very easy to implement and has been widely used in practice, while the update rule in \cite{xu2020non} for the critic needs to call a sub-algorithm, which involves generating a fresh episode to estimate the Q-function. Second, the analysis in \cite{xu2020non} relies on the compatible function approximation \citep{sutton2000policy}, which requires the critic to be a specific linear function class, while our analysis does not require such specific approximation, and therefore is more general. This makes our analysis potentially extendable to non-linear function approximation such as neural networks \citep{cai2019provably}.

%(1) the analysis in \citet{xu2020non} relies on the compatible function approximation, which enables analyzing natural policy gradient descent, but it is required that the critic must use the specific linear function class, while our function class could be chosen according to the specific RL problem, and potentially extended to non-linear function classes; (2) compatible function approximation requires a more complicated update rule for the critic: each iteration in the algorithm of \cite{xu2020non} involves restarting a new episode to estimate the current Q-function value, while our Algorithm \ref{alg:2ts_ac} in Section \ref{sec:preliminary} only requires one sample trajectory and enjoys a simple and practical implementation.

\section{Preliminaries} \label{sec:preliminary}
In this section, we present the background of the two time-scale actor-critic algorithm.

\subsection{Markov decision processes} 
Reinforcement learning tasks can be  modeled as a discrete-time Markov Decision Process (MDP)  $\cM=\{\cS,\cA,\cP,r\}$, where $\cS$ and $\cA$ are the state and action spaces respectively. In this work we consider the finite action space $|\cA| < \infty$. $\cP(s'|s,a)$ is the transition probability that the agent transits to state $s'$ after taking action $a$ at state $s$. Function $r:\cS\times\cA\rightarrow[-U_r,U_r]$ emits a bounded reward after the agent takes action
$a$ at state $s$, where $U_r>0$ is a constant. 
A policy parameterized by $\btheta$ at state $s$ is a probability function $\pi_{\btheta}(a|s)$ over action space $\cA$. $\mu_{\btheta}$ denotes the stationary distribution induced by the policy $\pi_{\btheta}$.

In this work we consider the ``average reward'' setting \citep{sutton2000policy}, where under the ergodicity assumption, the average reward over time eventually converges to the expected reward under the stationary distribution:
\begin{align*}
    r(\btheta)
    & := 
    \lim_{N \rightarrow \infty} \frac{\sum_{t=0}^{N} r(s_t, a_t)}{N}
    =
    \EE_{s \sim \mu_{\btheta}, a \sim \pi_{\btheta}}\big[r(s,a)\big].
\end{align*}

To evaluate the overall rewards given a starting state $s_0$ and the behavior policy $\pi_{\btheta}$, we define the state-value function  as 
\begin{align*}
    V^{\pi_{\btheta}}(\cdot) 
    & := 
    \EE\bigg[\sum_{t=0}^{\infty} 
    \big( 
    r(s_t, a_t) - r(\btheta) 
    \big) | s_0 = \cdot\bigg],
\end{align*}
where the action follows the policy $a_t \sim \pi_{\btheta}(\cdot|s_t)$ and the next state follows the transition probability $s_{t+1} \sim \cP(\cdot|s_t,a_t)$. Another frequently used function is the state-action value function, also called Q-value function:
\begin{align*}
    Q^{\pi_{\btheta}}(s,a) 
    : & =
    \EE \bigg[\sum_{t=0}^{\infty} 
    \big( 
    r(s_t, a_t) - r(\btheta) 
    \big) |s_0 = s, a_0 = a\bigg] \\
    & =
    r(s, a) - r(\btheta)  + \EE\big[V^{\pi_{\btheta}}(s')\big],
\end{align*}
where the expectation is taken over $s' \sim \cP(\cdot | s, a)$.

Throughout this paper, we use $O$ to denote the tuple $O = (s, a, s')$, some variants are like ${O}_t = ({s}_t, {a}_t, {s}_{t+1})$ and $\tilde{O}_t = (\tilde{s}_t, \tilde{a}_t, \tilde{s}_{t+1})$. 

\subsection{Policy gradient theorem}

We define the performance function associated with policy $\pi_{\btheta}$ naturally as the expected reward under the stationary distribution $\mu_{\btheta}$ induced by  $\pi_{\btheta}$, which takes the form
\begin{align}\label{Eq:J-function}
J(\btheta)  : &=  r(\btheta).
\end{align}

To maximize the performance function with respect to the policy parameters, \citet{sutton2000policy} proved the following policy gradient theorem.
\begin{lemma}[Policy Gradient] \label{lemma:policy-gradient}
Consider the performance function defined in \eqref{Eq:J-function}, its gradient takes the form
\begin{align*}
    \nabla J(\btheta) 
    & =  
    \EE_{s \sim \mu_{\btheta}(\cdot)}
    \bigg[
    \sum_{a \in \cA}
    Q^{\pi_{\btheta}}(s,a)\nabla \pi (a | s)
    \bigg].
\end{align*}
\end{lemma}
The policy gradient also admits a neat form in expectation: 
\begin{align*}
    \nabla J(\btheta)
    & = \EE_{s \sim \mu_{\btheta}(\cdot), a \sim \pi_{\btheta}(\cdot | s)} 
    \big[
    Q^{\pi_{\btheta}}(s,a) 
    \nabla \log \pi_{\btheta}(a|s)
    \big].
\end{align*}
A typical way to estimate the policy gradient $\nabla J(\btheta)$ is by Monte Carlo method, namely using the summed return along the trajectory as the estimated Q-value, which is known as the ``REINFORCE'' method \citep{williams1992simple}.

\begin{remark}
The problem formulation in this paper is what \citet{sutton2000policy} had defined as ``average-reward'' formulation. An alternative formulation is the ``start-state'' formulation, which avoids estimating the average reward, but gives a more complicated form for the policy-gradient algorithm and the AC algorithm.  
\end{remark}

\subsection{REINFORCE with a baseline}

Note that for any function $b(s)$ depending only on the state, which is usually called ``baseline'' function, we have 
\begin{align*}
    \sum_{a \in \cA} b(s) \nabla \pi_{\btheta}(a|s) = b(s) \nabla \bigg( \sum_{a \in \cA} \pi_{\btheta}(a|s) \bigg) = 0.
\end{align*}
 So we also have 
\begin{align*}
    \nabla J(\btheta)
    & = 
    \EE
    \bigg[
    \sum_{a \in \cA}
    \big(
    Q^{\pi_{\btheta}}(s,a) - b(s)
    \big)
    \nabla \pi_{\btheta}(a|s)
    \bigg].
\end{align*}
A popular choice of $b(s)$ is $b(s) = V^{\pi_{\btheta}}(s)$ and $\Delta^{\pi_{\btheta}}(s,a) = Q^{\pi_{\btheta}}(s,a) - V^{\pi_{\btheta}}(s)$ is viewed as the advantage of taking a specific action $a$, compared with the expected reward at state $s$.  Also note that the expectation form still holds:
\begin{align*}
    \nabla J(\btheta)
    & =  
    \EE_{s,a} 
    \big[
    \Delta^{\pi_{\btheta}}(s,a) 
    \nabla \log \pi_{\btheta}(a|s)
    \big].
\end{align*}
Based on this fact, \citet{williams1992simple} also proposed a corresponding policy gradient algorithm named ``REINFORCE with a baseline'' which performs better due to the reduced variance.

In practice the policy gradient method could suffer from high variance.
An alternative approach is to introduce another trainable model to approximate the state-value function, which is called the actor-critic methods.

\subsection{The two time-scale actor-critic algorithm} \label{subsec:alg}
In previous subsection, we have seen how the policy gradient theorem appears in the form of the advantage value instead of the Q-value. 
Assume the critic uses linear function approximation $\hat{V}(\cdot; \bomega) = \bphi^{\top}(\cdot) \bomega$, and is updated by TD(0) algorithm, then this gives rise to Algorithm \ref{alg:2ts_ac} that we are going to analyze.

Algorithm \ref{alg:2ts_ac} has been proposed in many literature, and is clearly introduced in \cite{sutton2018reinforcement} as a classic on-line one-step actor-critic algorithm. It uses the advantage (namely temporal difference error) to update the critic and the actor simultaneously. Based on its on-line nature, this algorithm can be implemented both under episodic and continuing setting.
In practice, the asynchronous variant of this algorithm, called Asynchronous Advantage Actor-Critic(A3C), is an empirically very successful parallel actor-critic algorithm. 

Sometimes, 
Algorithm \ref{alg:2ts_ac} is also called Advantage Actor-Critic (A2C) because it is the synchronous version of A3C and the name indicates its use of advantage instead of Q-value \citep{mnih2016asynchronous}. 
\begin{algorithm}[htbp!]
\caption{Two Time-Scale Actor-Critic 
\label{alg:2ts_ac}} 
\begin{algorithmic}[1]
\STATE \textbf{Input:} initial actor parameter $\btheta_0$, initial critic parameter $\bomega_0$, initial average reward estimator $\eta_0$, step size $\alpha_{t}$ for actor, $\beta_{t}$ for critic and $\gamma_{t}$ for the average reward estimator.

\STATE Draw $s_{0}$ from some initial distribution
\FOR {$t=0,1,2,\dots$}
\STATE Take the action $a_t \sim \pi_{\btheta_t}(\cdot|s_t)$
\STATE Observe next state $s_{t+1} \sim \cP(\cdot|s_t,a_t)$ and the reward $r_t = r(s_t,a_t)$
\STATE $\delta_t = r_t - \eta_t + \bphi(s_{t+1})^{\top} \bomega_{t} - \bphi(s_t)^{\top} \bomega_{t}$
\STATE $\eta_{t+1} = \eta_t + \gamma_t (r_t - \eta_t)$
\STATE $\bomega_{t+1} = \Pi_{R_{\bomega}}\big( \bomega_{t} + \beta_{t} \delta_t \bphi(s_t) \big)$\label{algline:critic_update}
\STATE $\btheta_{t+1} = \btheta_{t} + \alpha_t \delta_t \nabla_{\btheta} \log \pi_{\btheta_{t}}(a_t|s_t)$\label{algline:actor_update}
\ENDFOR 
\end{algorithmic} 
\end{algorithm}
 
In Line 6 of Algorithm \ref{alg:2ts_ac}, the temporal difference error $\delta_t$ can be calculated based on the critic's estimation of the value function $\bphi(\cdot)^{\top} \bomega_t$, where $\bomega_t\in\RR^d$ and $\phi(\cdot):\cS\rightarrow\RR^d$ is a known feature mapping.
Then the critic will be updated using the semi-gradient from TD(0) method. Line 8 in Algorithm \ref{alg:2ts_ac} also contains a projection operator. This is required to control the algorithm's convergence which also appears in some other literature \citep{bhandari2018finite,xu2019two}.
The actor uses the advantage $\delta_t$ (estimated by critic) and the samples to get an estimation of the policy gradient.

Algorithm \ref{alg:2ts_ac} is more general and practical than the algorithms analyzed in many previous work \citep{qiu2019finite,kumar2019sample}. In our algorithm, there is no need for independent samples or samples from the stationary distribution. There is only one naturally generated sample path. Also, the critic inherits from last iteration and continuously updates its parameter, without requiring a restarted sample path (or a new episode).

\section{Main theory} \label{sec:theory}
In this section, we first discuss on some standard assumptions used in the literature for deriving the convergence of reinforcement learning algorithms and then present our theoretical results for two time-scale actor-critic methods.

\subsection{Assumptions and propositions}
We consider the setting where the critic uses TD \citep{sutton2018reinforcement} with linear function approximation to estimate the state-value function, namely $\hat{V}(\cdot; \bomega) = \bphi^{\top}(\cdot) \bomega$. We assume that the feature mapping has bounded norm $\norm{\bphi(\cdot)} \le 1$. Denote by $\bomega^*(\btheta)$ the limiting point of TD(0) algorithms under the behavior policy $\pi_{\btheta}$, and define $\Ab$ and $\bbb$ as:
\begin{align*}
    \Ab &:= \EE_{s,a,s'} \big[ \bphi(s) \big( \bphi(s') - \bphi(s)\big)^{\top} \big], \\
    \bbb &:= \EE_{s,a,s'} [(r(s,a)- r(\btheta)) \bphi(s)],
\end{align*}
where $s \sim \mu_{\btheta}(\cdot), a \sim \pi_{\btheta}(\cdot | s), s' \sim \cP(\cdot | s, a)$.
It is known that the TD limiting point satisfies:
\begin{align*}
    \Ab \bomega^*(\btheta) + \bbb
    & = 
    \mathbf{0}.
\end{align*}
In the sequel, when there is no confusion, we will use a shorthand notation $\bomega^*$ to denote $\bomega^*(\btheta)$.
Based on the complexity of the feature mapping, the approximation error of this function class can vary. 
The approximation error of the linear function class is defined as follows:
\begin{align*}
    \epsilon_{\text{app}}(\btheta) := 
    \sqrt{
    \EE_{s \sim \mu_{\btheta}} \big( \bphi(s)^{\top} \bomega^*(\btheta) - V^{\pi_{\btheta}}(s) \big)^2
    }.
\end{align*} 
Throughout this paper, we assume the approximation error for all potential policies is uniformly bounded, 
\begin{align*}
     \forall \btheta, \epsilon_{\text{app}}(\btheta) \le \epsilon_{\text{app}},
\end{align*}
for some constant $\epsilon_{\text{app}}\geq0$.
 
In the analysis of TD learning, the following assumption is often made to ensure the uniqueness of the limiting point of TD and the problem's solvability. 
\begin{assumption} \label{assum:negative-definite}
    For all potential policy parameters $\btheta$, the matrix $\Ab$ defined above is negative definite and has the maximum eigenvalues as $- \lambda$.
    % \begin{align*}
    %     \Ab \preceq - \lambda \Ib.
    % \end{align*}
\end{assumption}
Assumption \ref{assum:negative-definite} is often made to guarantee the problem's solvability \citep{bhandari2018finite,zou2019finite, xu2019two}. Note that 
Algorithm~\ref{alg:2ts_ac} contains a projection step at Line 8. To guarantee convergence it is required all $\bomega^*$ lie within this projection radius $R_{\bomega}$.  Assumption~\ref{assum:negative-definite} indicates that a sufficient condition is to set $R_{\omega} = 2U_r / \lambda$ because $\norm{\bbb} \le 2 U_r$ and $\norm{\Ab^{-1}} \le \lambda^{-1}$.

The next assumption, first adopted by \citet{bhandari2018finite} in TD learning, addresses the issue of Markovian noise. 
\begin{assumption}[Uniform ergodicity] \label{assum:ergodicity}
     For a fixed $\btheta$, denote $\mu_{\btheta}(\cdot)$ as the stationary distribution induced by the policy $\pi_{\btheta}(\cdot|s)$ and the transition probability measure $\cP(\cdot|s,a)$. Consider a Markov chain generated by the rule $a_t \sim \pi_{\btheta}(\cdot | s_t), s_{t+1} \sim \cP(\cdot | s_t, a_t)$. Then there exists $m > 0$ and $\rho \in (0,1)$ such that:
    \begin{align*}
        d_{TV}\big(\PP(s_{\tau} \in \cdot | s_0 = s), \mu_{\btheta}(\cdot)\big) \le m \rho^{\tau}, \forall \tau \ge 0, \forall s \in \cS.
    \end{align*}
\end{assumption}
%With this assumption, we can remove the unrealistic assumption that each tuple is drawn from the stationary-state distribution, at the cost of only logarithm factors in the convergence rate.  

We also need some regularity assumptions on the policy.
\begin{assumption} \label{assum:policy-lipschitz-bounded}
Let $\pi_{\btheta}(a|s)$ be a policy parameterized by $\btheta$. There exist constants $L,B,L_l>0$ such that for all given state $s$ and action $a$ it holds 
\begin{enumerate}
\item[(a)] $\big\|\nabla \log \pi_{\btheta}(a|s) \big\| \le B$, $\forall \btheta \in \RR^d$,
\item[(b)] $\big\|\nabla \log \pi_{\btheta_1}(a|s) - \nabla \log \pi_{\btheta_2}(a|s) \big\| \le L_{l} \norm{\btheta_1 - \btheta_2}$, $\forall \btheta_1,\btheta_2 \in \RR^d$, 
\item[(c)] $\big|\pi_{\btheta_1}(a|s) - \pi_{\btheta_2}(a|s) \big| \le L \norm{\btheta_1 - \btheta_2}$, $\forall \btheta_1,\btheta_2 \in \RR^d$.
\end{enumerate}
\end{assumption}
The first two inequalities are regularity conditions to guarantee actor's convergence in the literature of policy gradient \citep{papini2018stochastic,zhang2019global, kumar2019sample,xu2019improved,xu2020sample}. The last inequality in Assumption \ref{assum:policy-lipschitz-bounded} is also adopted by \citet{zou2019finite} when analyzing SARSA. %Note that this assumption is less restrictive in our setting than in \citet{zou2019finite}, because here the Lipschitzness depends on the function class of the policy $\pi_{\btheta}$. For example, a ReLU neural network with bounded parameters, followed by softmax mapping will suffice. Meanwhile, for SARSA, the policy will choose the action with highest Q-value with high probability, and the optimal action can change even after a small change of the parameter. 

An important fact arises from our assumptions is that the limiting point $\bomega^*$ of TD(0) , which can be viewed as a mapping of the policy's parameter $\btheta$, is Lipschitz.
\begin{proposition} \label{prop:optimal-lipschitz}
    Under Assumptions \ref{assum:negative-definite} and \ref{assum:ergodicity}, there exists a constant $L_*>0$ such that
\begin{align*}
    \big \|
    \bomega^{*}(\btheta_{1}) - \bomega^{*}(\btheta_{2})
    \big\|
    \le L_{*} \norm{\btheta_{1} - \btheta_{2}}, \forall \btheta_{1}, \btheta_{2}\in\RR^d. 
\end{align*}
\end{proposition}
Proposition \ref{prop:optimal-lipschitz} states that the target point $\bomega^*$ moves slowly compared with the actor's update on $\btheta$. This is an observation pivotal to the two time-scale analysis. Specifically, the two time-scale analysis can be informally described as ``the actor moves slowly while the critic chases the slowly moving target determined by the actor''.

%\section{Main Theory}
Now we are ready to present the convergence result of two time-scale actor-critic methods. We first define an integer that depends on the learning rates $\alpha_t$ and $\beta_t$.
\begin{align}\label{eq:def_mixing_time}
    \tau_t & := 
    \min 
    \big \{
    i \ge 0 |
    m \rho^{i-1} \le
    \min \{ \alpha_t, \beta_t \}
    \big \},
\end{align}
where $m,\rho$ are defined as in Assumption \ref{assum:ergodicity}. By definition, $\tau_t$ is a mixing time of an ergodic Markov chain. We will use $\tau_t$ to control the Markovian noise encountered in the training process.

\subsection{Convergence of the actor}
At the $k$-th iteration of the actor's update, $\bomega_k$ is the critic parameter estimated by Line 7 of Algorithm \ref{alg:2ts_ac} and $\bomega_k^*$ is the unknown parameter of value function $V^{\pi_{\btheta_k}}(\cdot)$ defined in Assumption \ref{assum:negative-definite}. The following theorem gives the convergence rate of the actor %\ref{thm:actor} isolates the influence on the actor from the critic, by 
when the averaged mean squared error between $\bomega_k$ and $\bomega_k^*$ and the error between $\eta_k$ and $r(\btheta_k)$ from $k=\tau_t$ to $k=t$ are small.
\begin{theorem} \label{thm:actor}
Suppose Assumptions \ref{assum:negative-definite}-\ref{assum:policy-lipschitz-bounded} hold and we choose $\alpha_{t} = c_{\alpha}/(1+t)^{\sigma}$ in Algorithm \ref{alg:2ts_ac}, where $\sigma\in(0,1)$ and $c_\alpha>0$ are  constants. If we assume at the $t$-th iteration, the critic satisfies
\begin{align}\label{eq:def_average_critic}
    \frac{8}{t} 
    \sum_{k=1}^{t}
    \EE \|\bomega_k - \bomega^*_k\|^2 
    +
    \frac{2}{t} 
    \sum_{k=1}^{t}
    \EE \big(\eta_k - r(\btheta_k)\big)^2 &=
    \cE(t),
\end{align}
where $\cE(t)$ is a bounded sequence,
then we have
\begin{align*}
    \min_{0\leq k \le t} \EE \big\|\nabla J(\btheta_{k})\big\|^2 
    &=
    \cO(\epsilon_{\text{app}})
    +
    \cO\bigg(\frac{1}{t^{1- \sigma}}\bigg)
    +
    \cO\bigg(\frac{\log^2 t}{t^{\sigma}}\bigg)
    +
    \cO \big(\cE(t) \big),
\end{align*}
where $\cO(\cdot)$ hides constants, whose exact forms can be found in the detailed proof in Appendix \ref{subsec:proof-actor}. 
\end{theorem}
Note that $\cE(t)$ in Theorem~\ref{thm:actor} is the averaged estimation error made by the critic throughout the learning process, which will be bounded in the next Theorem~\ref{thm:critic}. %Later in Theorem~\ref{thm:critic}, we will bound $\cE(t)$ in \eqref{eq:critic_converge_a} and \eqref{eq:critic_converge_b}. 
%\begin{remark}
%In \eqref{eq:def_average_critic}, we assume the averaged mean squared error of critic from time steps $\tau_t$ to $t$ is bounded, which we will prove in the next section. The choice of the starting point as $\tau_t$ instead of $0$ might seem artificial compared to a natural choice of $0$.This is purely for a neat presentation of the proof. Compared to the scale of $t$, $\tau_t = \cO(\log t)$ is insignificant. Actually it is shown (in Supplementary \ref{subsec:proof-equi-asym}) that for any bounded sequence $\{ a_k \}$, the two means are equivalent.
%\end{remark}

\begin{remark}
Theorem \ref{thm:actor} recovers the results for the decoupled case \citep{qiu2019finite,kumar2019sample} by setting $\sigma =1/2$. Nevertheless, we are considering a much more practical and challenging case where the actor and critic are simultaneously updated under Markovian noises. It is worth noting that the non-i.i.d. data assumption leads to an additional logarithm term, which is also observed in \cite{bhandari2018finite,zou2019finite,srikant2019finite,chen2019performance}. 
% In the decoupled case, the actor-critic algorithms in \citet{kumar2019sample,qiu2019finite} need to use extra samples at each iteration of the actor's update in order to ensure a small estimation error of the critic, which will essentially significantly increases the sample complexity. More details will be discussed in Section \ref{sec:theory_complexity}. 
\end{remark}

\subsection{Convergence of the critic}
The condition in \eqref{eq:def_average_critic} is guaranteed by the following theorem that  characterizes the convergence of the critic.
\begin{theorem} \label{thm:critic}
Suppose Assumptions \ref{assum:negative-definite}-\ref{assum:policy-lipschitz-bounded} hold and we choose $\alpha_{t} = c_{\alpha}/(1+t)^{\sigma}$ and $\beta_{t} = c_{\beta}/(1+t)^{\nu}$ in Algorithm \ref{alg:2ts_ac}, where $0 < \nu < \sigma < 1$, $c_{\alpha}$ and $c_{\beta} \le \lambda^{-1}$ are positive constants. Then we have 
\begin{align}
    \frac{1}{1+t-\tau_t} \sum_{k=\tau_t}^{t} \EE \norm{\bomega_k - \bomega^*_k}^2
    & = 
    \cO
    \bigg(\frac{1}{t^{1-\nu}}
    \bigg)
    +
    \cO
    \bigg(\frac{\log t}{t^{\nu}}
    \bigg)
    +
    \cO
    \bigg(\frac{1}{t^{2(\sigma - \nu)}}
    \bigg), \label{eq:critic_converge_a}\\
    \frac{1}{1+t-\tau_t} \sum_{k=\tau_t}^{t} \EE \big(\eta_k - r(\btheta_k)\big)^2 
    & = 
    \cO
    \bigg(\frac{1}{t^{1-\nu}}
    \bigg)
    +
    \cO
    \bigg(\frac{\log t}{t^{\nu}}
    \bigg)
    +
    \cO
    \bigg(\frac{1}{t^{2(\sigma - \nu)}}
    \bigg),\label{eq:critic_converge_b}
\end{align}
where $\cO(\cdot)$ hides constants, whose exact forms can be found in the detailed proof in Appendix \ref{subsec:proof-eta} and \ref{subsec:proof-critic}.
\end{theorem}

%\begin{remark}
\begin{remark} 
The first term $\cO(t^{\nu-1})$ on the right hand side of \eqref{eq:critic_converge_a} and \eqref{eq:critic_converge_b} comes from loosely bounding the error's norm, and can be removed by applying the ``iterative refinement'' technique used in \citet{xu2019two}. Using this technique, we can obtain a bound (also holds for $\eta_t$) $\EE \norm{\bomega_t - \bomega^*_t}^2 = \cO(\log t/t^{\nu}) + \cO(1/t^{2(\sigma-\nu)-\xi})$, where $\xi > 0$ is an arbitrarily small constant. The constant $\xi$ is an artifact due to the the ``iterative refinement'' technique. Similar simplification can be done for \eqref{eq:critic_converge_b}.
Nevertheless, if we plug \eqref{eq:critic_converge_a} and \eqref{eq:critic_converge_b} (after some transformation) into the result of Theorem \ref{thm:actor}, it is easy to see that the term $\cO(1/t^{1-\nu})$ is actually dominated by the term $\cO(1/t^{1-\sigma})$. Thus this term makes no difference in the total sample complexity of Algorithm \ref{alg:2ts_ac} and we choose not to complicate the proof or introduce the extra artificial parameter $\xi$ in the result of Theorem \ref{thm:critic}. 

%Based on the technique in \citet{xu2019two}, we can also get a bound on $\EE \norm{\bomega_t - \bomega^*_t}^2 = \cO(\frac{\log t}{t^{\nu}}) + \cO(\frac{1}{t^{2(\sigma-\nu)-\xi}})$, where $\xi > 0$ is an arbitrarily small constant. This suggests the term $\cO(t^{-\nu-1})$ is indeed removable. The only drawback of their technique is the small constant $\xi$ in their bound, which is unnatural but also unavoidable due to the `iterative refinement' technique used by \citet{xu2019two}. Our proof can remove the unnatural small constant $\xi$, by taking a different framework.
%\end{remark}

The second term in both \eqref{eq:critic_converge_a} and \eqref{eq:critic_converge_b} comes from the Markovian noise and the variance of the semi-gradient. The third term in these two equations comes from the slow drift of the actor. These two terms together can be interpreted as follows: if the actor moves much slower than the critic (i.e., $\sigma - \nu \gg \nu$), then the error is dominated by the Markovian noise and gradient variance; if the actor moves not too slowly compared with the critic (i.e. $\sigma - \nu \ll \nu$), then the critic's error is dominated by the slowly drifting effect of the actor.
\end{remark}

\subsection{Convergence rate and sample complexity}\label{sec:theory_complexity}

Combining Theorems \ref{thm:actor} and \ref{thm:critic} leads to the following convergence rate and sample complexity for Algorithm \ref{alg:2ts_ac}. The detailed proof is in Appendix \ref{subsec:proof-samp-comp}.
\begin{corollary} \label{col:sample-complexity}
Under the same assumptions of Theorems \ref{thm:actor} and \ref{thm:critic}, we have 
\begin{align*} 
    \min_{0 \le k \le t} \EE \norm{\nabla J(\btheta_{k})}^2 
    &=
    \cO(\epsilon_{\text{app}})
    +
    \cO\bigg(\frac{1}{t^{1- \sigma}}\bigg)
    +
    \cO
    \bigg(\frac{\log t}{t^{\nu}}
    \bigg)
    +
    \cO
    \bigg(\frac{1}{t^{2(\sigma - \nu)}}
    \bigg).
\end{align*}
If we set $\sigma = 3/5, \nu = 2/5$, leading to the actor step size $\alpha_t=O(1/t^{3/5})$ and the critic step size $\beta_t=O(1/t^{2/5})$, %we can achieve the convergence rate $\cO\big(\frac{\log t}{t^{2/5}}\big)$, and 
Algorithm \ref{alg:2ts_ac} can find an $\epsilon$-approximate stationary point of $J(\cdot)$ within $T$ steps, namely,
\begin{align*}
    \min_{0 \le k \le T} \EE \big\|\nabla J(\btheta_{k})\big\|^2
    & \le 
    \cO(\epsilon_{\text{app}})
    +
    \epsilon, 
\end{align*}
where $T =\tilde{\cO}(\epsilon^{-2.5})$ is the total iteration number.
\end{corollary}
Corollary \ref{col:sample-complexity} combines the results of Theorems \ref{thm:actor} and \ref{thm:critic} and shows that the convergence rate of Algorithm \ref{alg:2ts_ac} is $\tilde{\cO}(t^{-2/5})$. Since the per iteration sample is $1$, the sample complexity of two time-scale actor-critic is $\tilde{\cO}(\epsilon^{-2.5})$.

\begin{remark}
We compare our results with existing results on the sample complexity of actor-critic methods in the literature. \citet{kumar2019sample} provided a general result that after $T = \cO(\epsilon^{-2})$ updates for the actor, the algorithm can achieve $\min_{0 \le k \le T} \EE \norm{\nabla J(\btheta_{k})}^2 \le \epsilon$ , as long as the estimation error of the critic can be bounded by $\cO(t^{-1/2})$ at the $t$-th actor's update. However, to ensure such a condition on the critic, they need to draw $t$ samples to estimate the critic at the $t$-th actor's update. Therefore, the total number of samples drawn from the whole training process by the actor-critic algorithm in \cite{kumar2019sample} is $\cO(T^2)$, 
%in terms of sample complexity, a reasonable measure should be the number of samples drawn throughout the whole training process, rather than the number of actor's update. Since \citet{kumar2019sample} require $t$ samples for the critic at the $t$-th actor's update (which means $\cO(T^2)$ samples altogether), 
yielding a $\cO(\epsilon^{-4})$ sample complexity. 
%Moreover, When the critic is specified to use TD(0), the sample complexity goes to $\cO(\epsilon^{-8})$ because $T=\cO(\epsilon^{-4})$ updates for the actor are required.
Under the similar setting, \citet{qiu2019finite} proved the same sample complexity %provided a similar analysis of actor-critic methods. When the critic uses TD(0), they only require $\tilde{\cO}(\epsilon^{-2})$ updates for the actor, but they also require the critic perform $T$ times sampling at the $T$-th update for the actor. Again, the sample complexity is 
$\tilde{\cO}(\epsilon^{-4})$ when TD(0) is used for estimating the critic. Thus Corollary \ref{col:sample-complexity} suggests that the sample complexity of Algorithm \ref{alg:2ts_ac} is significantly better than the sample complexity presented in \cite{kumar2019sample,qiu2019finite} by a factor of $\cO(\epsilon^{-1.5})$. 
\end{remark}
\begin{remark}
The gap between the ``decoupled'' actor-critic and the two time-scale actor-critic seems huge. Intuitively, this is due to the inefficient usage of the samples. At each iteration, the critic in the ``decoupled'' algorithm starts over to evaluate the policy's value function and discards the history information, regardless of the fact that the policy might only changed slightly. The two time-scale actor-critic keeps the critic's parameter and thus takes full advantage of each samples in the trajectory. 

%The ``decoupled'' setting has certain appeal because it circumvents the complicated scenario where the policy to be evaluated is constantly changing, and the actor and critic both are affected by Markovian noises.
%One of this work's contribution lies in building a clear framework for analysing two time-scale setting with Markovian noise and dynamically changing transition probability.
\end{remark}
\begin{remark}
According to \cite{papini2018stochastic}, the sample complexity of policy gradient methods such as REINFORCE is $\cO(\epsilon^{-2})$. As a comparison, if the critic converges faster than $\cO(t^{-1/2})$, namely $\cE(t)=\cO(t^{-1/2})$, then Theorem \ref{thm:actor} combined with Corollary \ref{col:sample-complexity} implies that the complexity of two time-scale actor-critic is $\tilde{\cO}(\epsilon^{-2})$, which matches the result of policy gradient methods \citep{papini2018stochastic} up to logarithmic factors.
%the exact value function with $\cO(1)$ cost, which means the estimation error for critic is $\cE(t)=0$. Then our analysis in \eqref{Eq:exact-actor} implies that the complexity of two time-scale actor-critic is $\tilde{\cO}(\epsilon^{-2})$, which matches the result in \citet{papini2018stochastic} up to log factors. However, in general, 
Nevertheless, as we have discussed in the previous remarks, a smaller estimation error for critic often comes at the cost of more samples needed for the critic update \citep{qiu2019finite,kumar2019sample}, which eventually increases the total sample complexity. Therefore, the $\tilde{\cO}(\epsilon^{-2.5})$ sample complexity in Corollary \ref{col:sample-complexity} is indeed the lowest we can achieve so far for classic two time-scale actor-critic methods. However, it is possible to further improve the sample complexity by using policy evaluation algorithms better than vanilla TD(0), such as GTD and TDC methods.
\end{remark}

\section{Conclusion and discussion} \label{sec:conclusion}
In this paper, we provided the first finite-time analysis of the two time-scale actor-critic methods, with non-i.i.d. Markovian samples and linear function approximation. The algorithm we analyzed is an on-line, one-step actor-critic algorithm which is practical and efficient. We proved its non-asymptotic convergence rate as well as its sample complexity. Our proof technique can be potentially extended to analyze other two time-scale reinforcement learning algorithms.

As one of the anonymous reviewers suggested, the compatible features are useful tools to address the function approximation error of the critic \citep{konda2000actor}. This can leads to finite-time analysis for the natural actor-critic algorithm \citep{xu2020non}, which also relates to the more general natural policy gradient methods \citep{cen2020fast}.
Another possible improvement is to use regularization( e.g. ridge) for the critic to ensure the boundedness of the critic and remove the assumption on the maximum eigenvalue.
The analysis can also be applied to the infinite-horizon discounted MDP, where the framework of analysis essentially remains the same.

\section*{Broader impact}
This work could positively impact the industrial application of actor-critic algorithms and other reinforcement learning algorithms.
The theorem exhibits the sample complexity of actor-critic algorithms, which could be used to estimate required training time of reinforcement learning models. Another direct application of our result is to set the learning rate according to the finite-time bound, by optimizing the constant factors of the dominant terms.
In this sense, the result could potentially reduce the overhead of hyper-parameter tuning, thus saving both human and computational resources. Moreover, the new analysis in this paper can potentially help people in different fields to understand the broader class of two-time scale algorithms, in addition to actor-critic methods.
To our knowledge, this algorithm and theory studied in our paper do not have any ethical issues.

\section*{Acknowledgement}
We would like to thank the anonymous reviewers for
their helpful comments. We also thank Xuyang Chen and Lin Zhao for pointing out a bug caused by the notation inconsistency in the proof of Theorem~\ref{thm:actor} and Lemma~\ref{lemma:Gamma-term1} in the previous version. This research was sponsored
in part by the National Science Foundation IIS-1904183 and Adobe Data Science Research Award. The views and conclusions contained
in this paper are those of the authors and should not be
interpreted as representing any funding agencies.

\bibliographystyle{plainnat}
\bibliography{RL}

\newpage
\appendix
\section{Proof Sketch} \label{sec:proof}
In this section, we provide the proof roadmap of the main theory. Detailed proofs  can be found in Appendix \ref{sec:proof_of_maintheory}.
\subsection{Proof Sketch of Theorem \ref{thm:actor}} 
%\textbf{Step 1: XXX} 
The following lemma is important in that it enables the analysis of policy gradient method:
\begin{lemma}[\cite{zhang2019global}]
For the performance function defined in \eqref{Eq:J-function}, there exists a constant $L_J>0$ such that for all $\btheta_1,\btheta_2\in\RR^d$, it holds that
\begin{align*}
    \big\| 
    \nabla {J(\btheta_1)} - \nabla {J(\btheta_2)}
    \big \|
    \le L_{J} \norm{\btheta_1 - \btheta_2},
\end{align*}
which by the definition of smoothness \citep{nesterov2018lectures} is also equivalent to 
\begin{align*}
    J(\btheta_2) \ge 
    J(\btheta_1) + \big\la \nabla J(\btheta_1),\btheta_2 - \btheta_1\big\ra - \frac{L_{J}}{2} \norm{\btheta_1 - \btheta_2}^2.
\end{align*}
\end{lemma}
This lemma enables us to perform a gradient ascent style analysis on the non-concave function $J(\btheta)$:
% \begin{align}\label{eq:sketch_decomp}
%     J(\btheta_{t+1}) 
%     & \ge
%     J(\btheta_{t}) 
%     + \alpha_{t} \big \la \nabla J (\btheta_{t}),\delta_t \nabla \log \pi_{\btheta_{t}}(a_t | s_t)\big \ra
%     - L_{J} \alpha_{t}^{2} 
%     \big\|\delta_{t} \nabla \log {\pi_{\btheta_t}}(a_t | s_t) \big\|^2 
%     \notag \\ 
%     & =
%     J(\btheta_{t}) 
%     + \alpha_{t} \big \la \nabla J (\btheta_{t}), \Delta h(O_t, \eta_t, \bomega_t, \btheta_t) \big \ra
%     + \alpha_{t} \big\la \nabla J (\btheta_{t}), \Delta h'(O_t, \btheta_t) \big\ra
%     \notag\\
%     & \qquad
%     + \alpha_{t}  \Gamma(O_t,  \btheta_t)
%     + \alpha_{t} \big \|\nabla J (\btheta_{t}) \big\|^2 
%     - L_{J} \alpha_{t}^{2} \big\| \delta_{t} \nabla \log {\pi_{\btheta_t}}(a_t | s_t) \big\|^2,
% \end{align}
\begin{align}\label{eq:sketch_decomp}
    J(\btheta_{t+1}) 
    & \ge
    J(\btheta_{t}) 
    + \alpha_{t} \big \la \nabla J (\btheta_{t}),\delta_t \nabla \log \pi_{\btheta_{t}}(a_t | s_t)\big \ra
    - L_{J} \alpha_{t}^{2} 
    \big\|\delta_{t} \nabla \log {\pi_{\btheta_t}}(a_t | s_t) \big\|^2 
    \notag \\
    & =
    J(\btheta_{t}) 
    + \alpha_{t} \big \la \nabla J (\btheta_{t}), \Delta h(O_t, \eta_t, \bomega_t, \btheta_t) \big \ra
    + \alpha_{t} \big\la \nabla J (\btheta_{t}),  \EE_{O'} [\Delta h'(O', \btheta_t)] \big\ra
    \notag\\
    & \qquad
    + \alpha_{t}  \Gamma(O_t,  \btheta_t)
    + \alpha_{t} \big \|\nabla J (\btheta_{t}) \big\|^2 
    - L_{J} \alpha_{t}^{2} \big\| \delta_{t} \nabla \log {\pi_{\btheta_t}}(a_t | s_t) \big\|^2,
\end{align}
where $O_t = (s_t, a_t, s_{t+1})$ is a tuple of observations. The second term $\Delta h(O_t, \eta_t, \bomega_t, \btheta_t)$ on the right hand side of \eqref{eq:sketch_decomp} is the bias introduced by the critic $\bomega_t$ and the reward estimate $\eta_t$.
The third term $\Delta h'(O_t, \btheta_t)$ is from the linear approximation error.
The fourth term $\Gamma(O_t, \btheta_t)$ is due to the Markovian noise.
The last term can be viewed as the variance of the stochastic gradient update. Please refer to \eqref{def:actor-term} for the definition of each notation.

Now we bound each term's expectation in \eqref{eq:sketch_decomp} respectively.

First, we have
\begin{align*}
    \EE \big\la \nabla J (\btheta_{t}), \Delta h(O_t, \eta_t, \bomega_t, \btheta_t) \big\ra
    & \ge 
    - B \sqrt{\EE \big\| \nabla J (\btheta_{t})\big\|^2}  \sqrt{8\EE \norm{\zb_t}^2 + 2\EE[y_t^2]},
\end{align*}
where $\zb_t := \bomega_t - \bomega^*_t$ and $y_t := \eta_t - \eta^*_t$, and the inequality is due to Cauchy inequality and Lemma~\ref{lemma:critic-bias}.

Second, taking expectation over the approximation error term containing $\Delta h'$, we have
\begin{align*}
    \EE \big\la \nabla J (\btheta_{t}), \Delta h'(O_t, \btheta_t) \big\ra
    & \ge 
    - G_{\btheta}  \sqrt{
    \EE \big\| \Delta h'(O_t, \btheta_t) \big\|^2
    }
    \\
    & \ge
    - G_{\btheta} \cdot 2B \sqrt{
    \EE \big( \phi(s)^{\top} \bomega^*_t - V^{\pi_{\btheta_t}}(s) \big)^2
    }
    \\
    & \ge 
    -2B G_{\btheta} \epsilon_{\text{app}},
\end{align*}

Third, we have
\begin{align*}
    \EE [\Gamma(O_t,  \btheta_t)]
    & \ge 
    -G_{\btheta} \bigg(D_{1} (\tau + 1) \sum_{k=t-\tau+1}^t \EE \norm{\btheta_k - \btheta_{k-1}}
    +D_{2} m \rho^{\tau - 1}\bigg), \\
    & \ge 
    - G_{\btheta}
    \bigg(
    D_1 (\tau + 1) G_{\btheta} \sum_{k=t-\tau+1}^{t-1} \alpha_k 
    + 
    D_2 m \rho^{\tau - 1}
    \bigg),
\end{align*}
where the first inequality is due to Lemma~\ref{lemma:actor-markovian}, and the second inequality is due to %the fact that $\|\btheta_t-\btheta_{t-\tau}\|\leq\sum_{k=t-\tau}^{t-1}\|\btheta_{k+1}-\btheta_{k}\|$ and that
$\big\| \delta_{t} \nabla \log {\pi_{\btheta_t}}(a_t | s_t) \big\| \le G_{\btheta}$ by Lemma~\ref{lemma:actor-markovian}.

Taking the expectation of \eqref{eqn:actor-all-term}, plugging the above terms back into it and rearranging give
\begin{align*}
    \EE \big\| \nabla J(\btheta_t) \big\|^2
    & \le 
    \frac{1}{\alpha_t}
    \big(
    \EE [J(\btheta_{t+1})] - \EE [J(\btheta_{t})]
    \big)
    + 
    B \sqrt{\EE \big\| \nabla J(\btheta_t) \big\|^2} \sqrt{8\EE \norm{\zb_t}^2 + 2 \EE[y_t^2]} \\
    & \qquad + 
    D_1 G_{\btheta}^2 (\tau + 1) \sum_{k=t-\tau}^{t-1} \alpha_k  
    +
    D_2 G_{\btheta} m \rho^{\tau - 1} 
    +
    L_{J} G_{\btheta}^2 \alpha_t.
\end{align*}
Setting $\tau = \tau_{t}$ and summing over each term, and further
dividing $(1 + t - \tau_{t})$ at both sides and assuming $t > 2\tau_t-1$, we can express the result as 
\begin{align} \label{eqn:sketch-actor-bone}
    \frac{1}{1+t-\tau_{t}}\sum_{k=\tau_{t}}^{t} \EE \big\| \nabla J(\btheta_t) \big\|^2 &\le 
    \cO
    \bigg( 
    \frac{1}{t^{1-\sigma}} 
    \bigg)
    + 
    \cO
    \bigg(
    \frac{(\log t)^2}{t^{\sigma}} 
    \bigg)
    +
    \cO(\epsilon_{\text{app}})
    \notag \\
    & \qquad +
    \frac{2B}{1+t-\tau_{t}}
     \sum_{k=\tau_{t}}^{t}
    \sqrt{\EE \big\| \nabla J(\btheta_t) \big\|^2} \sqrt{8\EE \norm{\zb_t}^2 + 2 \EE[y_t^2]}
\end{align}
By Cauchy-Schwartz inequality, we have
\begin{align*}
    &\frac{1}{1+t-\tau_{t}}
    \sum_{k=\tau_{t}}^{t}
    \sqrt{\EE \big\| \nabla J(\btheta_t) \big\|^2} \sqrt{8\EE \norm{\zb_t}^2 + 2 \EE[y_t^2]}\\
    & \qquad \le 
    \bigg(
    \frac{1}{1+t-\tau_{t}}\sum_{k=\tau_{t}}^{t}
    \EE \big\| \nabla J(\btheta_t) \big\|^2
    \bigg)^{\frac{1}{2}}
    \bigg(
    \frac{1}{1+t-\tau_{t}}
    \sum_{k=\tau_{t}}^{t}
    \big(
    8\EE \norm{\zb_t}^2 + 2 \EE[y_t^2]
    \big)
    \bigg)^{\frac{1}{2}}.
\end{align*}
Now, denote $F(t) := 1/(1+t-\tau_{t})\sum_{k=\tau_{t}}^{t} \EE \norm{\nabla J(\btheta_k)}^2$ and $Z(t) := 1/(1+t-\tau_{t}) \sum_{k=\tau_{t}}^{t} \big( 8\EE \norm{\zb_t}^2 + 2 \EE[y_t^2] \big)$, and putting them back to~\eqref{eqn:sketch-actor-bone}  ($\cO$-notation for simplicity):
\begin{align*}
    F(t)
    & \le 
    \cO
    \bigg( 
    \frac{1}{t^{1-\sigma}} 
    \bigg)
    + 
    \cO
    \bigg(
    \frac{(\log t)^2}{t^{\sigma}} 
    \bigg)
    +
    \cO(\epsilon_{\text{app}})
    +
    2B \sqrt{F(t)} \cdot \sqrt{Z(t)}
    ,
\end{align*}
which further gives
\begin{align*}
    \big(
    \sqrt{F(t)} - B \sqrt{Z(t)}
    \big)^2 
    & \le 
    \cO
    \bigg( 
    \frac{1}{t^{1-\sigma}} 
    \bigg)
    + 
    \cO
    \bigg(
    \frac{(\log t)^2}{t^{\sigma}} 
    \bigg)
    +
    \cO(\epsilon_{\text{app}})
    +
    B^2 Z(t)
    .
\end{align*}
Note that for a general function $H(t) =  A(t) + B(t)$(with each positive), we have 
\begin{align*}
    H^2(t) & =
    \cO \big(A^2(t) \big)
    +
    \cO \big(B^2(t) \big), \\ 
    \sqrt{H(t)} 
    & =
    \cO \big(\sqrt{A(t)} \big)
    +
    \cO \big(\sqrt{B(t)} \big).
\end{align*}
This means
\begin{align*}
    \min_{0 \le k \le t} 
    \EE \big\|\nabla J(\btheta_k)\big\|^2 
    & \le 
    \frac{1}{1+t-\tau_{t}}\sum_{k=\tau_{t}}^{t} \EE \big\|\nabla J(\btheta_k)\big\|^2 
    \\
    & = 
    \cO
    \bigg( 
    \frac{1}{t^{1-\sigma}}
    \bigg)
    +
    \cO
    \bigg( 
    \frac{1}{t^{\sigma}}
    \bigg)
    +
    \cO
    (\epsilon_{\text{app}})
    +
    \cO
    \big( 
    \cE(t)
    \big).
\end{align*}

\subsection{Proof Sketch of Theorem \ref{thm:critic}}
The proof of Theorem \ref{thm:critic} can be divided into the following two parts.
\subsubsection{Estimating the Average Reward $\eta_k$}
We denote $y_k := \eta_k - r(\btheta_k)$.
First, we shall mention that many components in this step is uses the same framework and partial result as the proof regarding $\bomega_t$ in the next part. 
Also, part of the proof is intriguingly similar with the proof of Theorem \ref{thm:actor}. 
For simplicity, here we only present the final result regarding $\eta_k$. Please refer to Section \ref{subsec:proof-eta} for the detailed proof. By setting $\gamma_k = (1+t)^{-\nu}$, we have that
\begin{align*}
    \sum_{k=\tau_t}^{t} \EE[y_k^2 ]
    & = 
    \cO(t^{\nu})
    +
    \cO(\log t \cdot t^{1-\nu})
    +
    \cO(t^{1- 2(\sigma - \nu)}).
\end{align*}

\subsubsection{Approximating the TD Fixed Point}
\textbf{Step 1: decomposition of the estimation error.}
For simplicity, we denote $\zb_t := \bomega_t - \bomega^*_t$, where the $\bomega^*_t$ denotes the exact parameter under policy $\pi_{\btheta_t}$. By the critic update in Line 7 of Algorithm \ref{alg:2ts_ac}, we have 
\begin{align}\label{Eq:error-decomposition}
    \norm{\zb_{t+1}}^2 
    & =  
    \norm{\zb_t}^2 
    + 2\beta_t \big \la \zb_t,\bar{g}( \bomega_t, \btheta_t) \big \ra 
    + 2\beta_t \Lambda(O_t, \bomega_t, \btheta_t)
    + 2\beta_t \big \la \zb_t, \Delta g(O_t, \eta_t, \btheta_t) \big \ra
    \notag \\
    & \qquad
    + 2 \dotp{\zb_t}{\bomega_{t}^{*} - \bomega_{t+1}^{*}} 
    + \big \| \beta_t (g(O_t, \bomega_t, \btheta_t)+ \Delta g(O_t, \eta_t, \btheta_t)) 
    + (\bomega_{t}^{*} - \bomega_{t+1}^{*}) \big \|^2.
\end{align}
where $O_t:= (s_t, a_t, s_{t+1})$ is a tuple of observations, 
$g(O_t, \bomega_t)$ and $\bar{g}(\btheta_t, \bomega_t)$
are the estimated gradient and the true gradient respectively. $\Lambda(O_t, \bomega_t, \btheta_t) : = \dotp{\bomega_t - \bomega^*_t}{g(O_t, \bomega_t) - \bar{g}(\btheta_t, \bomega_t)}$ can be seen as the error induced by the Markovian noise. Please refer to \eqref{def:critic-term} for formal definition of each notation.

The second term on the right hand side of~\eqref{Eq:error-decomposition} can be bounded by $-2 \lambda \beta_t \norm{\zb_t}^2$ due to Assumption \ref{assum:negative-definite}. The third term is a bias term caused by the Markovian noise. The fourth term $\Delta g(O_t, \eta_t, \btheta_t)$ is another bias term caused by inaccurate average reward estimator $\eta_t$. The fifth term is caused by the slowly drifting policy parameter $\btheta_t$. And the last term can be considered as the variance term. 

Rewriting \eqref{Eq:error-decomposition} and telescoping from $\tau = \tau_t$ to $t$, we have
\begin{align} \label{Eq:critic-sum}
    2 \lambda \sum_{k=\tau_{t}}^{t} \EE \norm{\zb_k}^2
     &\le 
    \underbrace{\sum_{k=\tau_{t}}^{t} \frac{1}{\beta_k} \big(\EE \norm{\zb_k}^2 - \EE \norm{\zb_{k+1}}^2  \big)}_{I_1}
    + 
    2 \underbrace{\sum_{k=\tau_{t}}^{t}  \EE \Lambda(\btheta_k, \bomega_k, O_k)}_{I_2}
    \notag \\ 
    & \qquad 
    + 
    2 L_{*} G_{\btheta} 
    \underbrace{\sum_{k=\tau_{t}}^{t}   \frac{\alpha_k}{\beta_k} \sqrt{\EE \norm{\zb_k}}}_{I_3} 
    +
    \underbrace{\sum_{k=\tau_{t}}^{t}   \sqrt{\EE [y_k^2]} \cdot \sqrt{\EE \norm{\zb_k}}}_{I_4}
    + 
    C_{q} \underbrace{\sum_{k=\tau_{t}}^{t} \beta_{k}}_{I_5} .
\end{align}
We will see that the Markovian noise $I_2$, the ``slowly drifting policy" term $I_3$ and the estimation bias $I_4$ from $\eta_t$ are significant, and bounding the Markovian term is another challenge.

\noindent\textbf{Step 2: bounding the Markovian bias.}  We first decompose $\Lambda(\btheta_t, \bomega_t, O_t)$ as follows.
\begin{align}\label{Eq:Markov-decomposition}
    \Lambda(\btheta_t, \bomega_t, O_t)
    &=
    \big(\Lambda(\btheta_t, \bomega_t, O_t) - \Lambda(\btheta_{t-\tau}, \bomega_t, O_t) \big) 
    + \big(\Lambda(\btheta_{t-\tau}, \bomega_t, O_t) - \Lambda(\btheta_{t-\tau}, \bomega_{t-\tau}, O_t) \big) 
    \notag \\
    &\qquad+
    \big(\Lambda(\btheta_{t-\tau}, \bomega_{t-\tau}, O_t) - \Lambda(\btheta_{t-\tau}, \bomega_{t-\tau}, \tilde{O}_t) \big)
    + 
    \Lambda(\btheta_{t-\tau}, \bomega_{t-\tau}, \tilde{O}_t). 
\end{align}
The motivation is to employ the uniform ergodicity defined by Assumption \ref{assum:ergodicity}. This technique was first introduced by \citet{bhandari2018finite} to address the Markovian noise in policy evaluation. \citet{zou2019finite} extended to the Q-learning setting where the parameter itself both keeps updated and determines the behavior policy. 
In this work we take one step further to consider that the policy parameter $\btheta_t$ is changing, and the evaluation parameter $\bomega_t$ is updated. The analysis relies on the auxiliary Markov chain constructed by \citet{zou2019finite}, which is obtained by repeatedly applying policy $\pi_{\btheta_{t-\tau}}$:
\begin{align*}
    s_{t-\tau} \xrightarrow{\btheta_{t-\tau}} 
    a_{t-\tau} \xrightarrow{\cP} 
    s_{t-\tau+1} \xrightarrow{\btheta_{t-\tau}} 
    \tilde{a}_{t-\tau+1} \xrightarrow{\cP} 
    \tilde{s}_{t-\tau+2}
    \xrightarrow{\btheta_{t-\tau}} 
    \tilde{a}_{t-\tau+2} \xrightarrow{\cP} 
    \cdots 
    \xrightarrow{\cP}
    \tilde{s}_{t} \xrightarrow{\btheta_{t-\tau}} 
    \tilde{a}_{t} \xrightarrow{\cP} 
    \tilde{s}_{t+1}.
\end{align*}
For reference, recall that the original Markov chain is given by:
\begin{align*}
    s_{t-\tau} \xrightarrow{\btheta_{t-\tau}} 
    a_{t-\tau} \xrightarrow{\cP} 
    s_{t-\tau+1} \xrightarrow{\btheta_{t-\tau+1}} 
    {a}_{t-\tau+1} \xrightarrow{\cP} 
    {s}_{t-\tau+2} \xrightarrow{\btheta_{t-\tau+2}} 
    {a}_{t-\tau+2} \xrightarrow{\cP} 
    \cdots 
    \xrightarrow{\cP}
    {s}_{t} \xrightarrow{\btheta_{t}} 
    {a}_{t} \xrightarrow{\cP} 
    {s}_{t+1}.
\end{align*}
By Lipschitz conditions, we can bound the first two terms in \eqref{Eq:Markov-decomposition}. The third term will be bounded by the total variation between $s_k$ and $\tilde{s}_k$, which is achieved by recursively bounding total variation between $s_{k-1}$ and $\tilde{s}_{k-1}$. 

In fact, the Markovian noise $\Gamma(O_t, \btheta_t)$ in Section \ref{subsec:proof-actor} is obtained in a similar way. Due to the space limit, we only present how to bound the more complicated $\Lambda(\btheta_t, \bomega_t, O_t)$. \\
We have the final form as:
\begin{align}
    \Lambda(\btheta_t, \bomega_t, O_t)
    & \le 
    C_{1}(\tau + 1) \norm{\btheta_{t} - \btheta_{t-\tau}} + C_{2} m \rho^{\tau - 1} 
    + C_{3} \norm{\bomega_t - \bomega_{t-\tau}},
    \label{Eq:Lambda-bound}
\end{align}
where $C_{1} = 2 U_{\delta}^2 |\cA| L (1 + \lceil \log_{\rho} m^{-1}\rceil + 1/(1- \rho) ) + 2 U_{\delta} L_{*}, C_{2} = 2 U_{\delta}^2, C_{3} = 4 U_{\delta}$ are constants. 

\noindent\textbf{Step 3: integrating the results.}
By some calculation, terms $I_1$, $I_2$ and $I_4$ can be respectively bounded as follows (set $\tau = \tau_t$ defined in \eqref{eq:def_mixing_time}). The detailed derivation can be found in Appendix~\ref{subsec:proof-critic},
\begin{align*}
    I_1 
    & = 
    4R_{\bomega}^2 \frac{1}{\beta_t}
    = \cO(t^{\nu}), \\
    I_2 
    & \le 
    C_1 G_{\btheta} (\tau_{t} + 1)^2 \sum_{k=0}^{t-\tau_{t}}\alpha_{k}
    +
    C_2 (t-\tau_{t}+1) \alpha_t
    +
    C_3 U_{\delta} \tau_{t} \sum_{k=0}^{t-\tau_{t}} \beta_k 
    \\ 
    & = 
    \cO
    \big(
    (\log t)^2 t^{1-\sigma}
    \big)
    +
    \cO(t^{1-\sigma})
    +
    \cO
    \big(
    (\log t) t^{1-\nu}
    \big) \\ 
    & = 
    \cO\big((\log t) t^{1-\nu}\big), \\
    I_5 
    & = 
    \sum_{k=0}^{t-\tau_{t}} \beta_k = \cO(t^{1-\nu}).
\end{align*}
The $\log t$ comes from $\tau_t = \cO(\log t)$. Performing the same technique on $I_3$ as in Step 3 in the proof sketch of Theorem \ref{thm:actor}, we have 
\begin{align*}
    I_3 
    & \le 
    \bigg(
    \sum_{k=0}^{t-\tau_{t}}   \frac{\alpha_k^2}{\beta_k^2}
    \bigg)^{\frac{1}{2}}
    \bigg(
    \sum_{k=\tau_{t}}^{t}  
    \EE \norm{\zb_k}^2
    \bigg)^{\frac{1}{2}}, \\ 
    I_4
    & \le 
    \bigg(
    \sum_{k=\tau_{t}}^{t}  
    \EE [y_k^2]
    \bigg)^{\frac{1}{2}}
    \bigg(
    \sum_{k=\tau_{t}}^{t}  
    \EE \norm{\zb_k}^2
    \bigg)^{\frac{1}{2}}.
\end{align*}
After plugging each term into \eqref{Eq:critic-sum}, we have that 
\begin{align*}
    2 \lambda \sum_{k=\tau_{t}}^{t} \EE \norm{\zb_k}^2
    & \le
    \cO(t^{\nu})
    +
    \cO \big((\log t) t^{1-\nu} \big)
    \\
    & \qquad +
    2 L_* G_{\btheta}
    \bigg(
    \sum_{k=0}^{t-\tau_{t}}   \frac{\alpha_k^2}{\beta_k^2}
    \bigg)^{\frac{1}{2}}
    \bigg(
    \sum_{k=\tau_{t}}^{t}  
    \EE \norm{\zb_k}^2
    \bigg)^{\frac{1}{2}}
    +
    \bigg(
    \sum_{k=0}^{t-\tau_{t}}   \EE [y_k^2]
    \bigg)^{\frac{1}{2}}
    \bigg(
    \sum_{k=\tau_{t}}^{t}  
    \EE \norm{\zb_k}^2
    \bigg)^{\frac{1}{2}}.
\end{align*}
This inequality actually resembles \eqref{eqn:sketch-actor-bone}. Following the same procedure as the proof of Theorem~\ref{thm:actor}, starting from~\eqref{eqn:sketch-actor-bone}, we can finally get
\begin{align*}
    \frac{1}{1+t-\tau_t} \sum_{k=\tau_t}^{t} \EE \norm{\zb_k}^2 = 
    \cO
    \bigg(\frac{1}{t^{1-\nu}}
    \bigg)
    +
    \cO
    \bigg(\frac{\log t}{t^{\nu}}
    \bigg)
    +
    \cO
    \bigg(\frac{1}{t^{2(\sigma - \nu)}}
    \bigg).
\end{align*}
Note that this requires the step sizes $\gamma_t$ and $\beta_t$ should be of the same order $\cO(t^{-\nu})$.

\section{Preliminary Lemmas} \label{sec:prel-lemma}
These useful lemmas are frequently applied throughout the proof.
\subsection{Probabilistic Lemmas}
The first two statements in the following lemma come from \citet{zou2019finite}.
\begin{lemma} %[\cite{zou2019finite}] 
\label{lemma:prob-mixing}
For any $\btheta_1$ and $\btheta_2$, it holds that
\begin{align*}
d_{TV}(\mu_{\btheta_1}, \mu_{\btheta_2}) 
& \le 
|\cA|L \bigg( \lceil \log_{\rho}m^{-1} \rceil+ \frac{1}{1-\rho} \bigg) \norm{\btheta_1 - \btheta_2}, \\
d_{TV}(\mu_{\btheta_1} \otimes \pi_{\btheta_1}, \mu_{\btheta_2}\otimes \pi_{\btheta_2}) 
& \le 
|\cA|L \bigg(1 + \lceil \log_{\rho}m^{-1} \rceil + \frac{1}{1-\rho}\bigg) \norm{\btheta_1 - \btheta_2}, \\
d_{TV}(\mu_{\btheta_1} \otimes \pi_{\btheta_1} \otimes \cP, \mu_{\btheta_2}\otimes \pi_{\btheta_2} \otimes \cP) 
& \le 
|\cA|L \bigg( 1 + \lceil \log_{\rho}m^{-1} \rceil + \frac{1}{1-\rho} \bigg) \norm{\btheta_1 - \btheta_2}.
\end{align*}

\begin{proof}
The proof of the first two inequality is exactly the same as  Lemma A.3 in \citet{zou2019finite}, which mainly depends on Theorem 3.1 in \citet{mitrophanov2005sensitivity}. Here we provide the proof of  the third inequality. Note that 
\begin{align}
    & d_{TV}(\mu_{\btheta_1} \otimes \pi_{\btheta_1} \otimes \cP, \mu_{\btheta_2}\otimes \pi_{\btheta_2} \otimes \cP) 
    \notag \\
    & = 
    \frac{1}{2}
    \int_{\cS} \sum_{\cA} \int_{\cS} 
    \big|
    \mu_{\btheta_1}(ds) \pi_{\btheta_1}(a|s) \cP(ds'|s,a)
    -
    \mu_{\btheta_2}(ds) \pi_{\btheta_2}(a|s) \cP(ds'|s,a)
    \big| 
    \notag \\
    & =
    \frac{1}{2}
    \int_{\cS} \sum_{\cA} \int_{\cS} 
    \cP(ds'|s,a)
    \big|
    \mu_{\btheta_1}(ds) \pi_{\btheta_1}(a|s) 
    -
    \mu_{\btheta_2}(ds) \pi_{\btheta_2}(a|s) 
    \big| 
    \notag \\
    & =
    \frac{1}{2}
    \int_{\cS} \sum_{\cA}
    \big|
    \mu_{\btheta_1}(ds) \pi_{\btheta_1}(a|s) 
    -
    \mu_{\btheta_2}(ds) \pi_{\btheta_2}(a|s) 
    \big| 
    \notag \\
    & = d_{TV}(\mu_{\btheta_1} \otimes \pi_{\btheta_1}, \mu_{\btheta_2}\otimes \pi_{\btheta_2}),
    \label{eqn:peeling}
\end{align}
so it has the same upper bound as the second inequality.
\end{proof}
\end{lemma}

\begin{lemma} \label{lemma:auxiliary-chain}
Given time indexes $t$ and $\tau$ such that $t \ge \tau > 0$, consider the auxiliary Markov chain starting from $s_{t-\tau}$. Conditioning on $s_{t-\tau+1}$ and $\btheta_{t-\tau}$, the Markov chain is obtained by repeatedly applying policy $\pi_{\btheta_{t-\tau}}$.
\begin{align*}
    s_{t-\tau} \xrightarrow{\btheta_{t-\tau}} 
    a_{t-\tau} \xrightarrow{\cP} 
    s_{t-\tau+1} \xrightarrow{\btheta_{t-\tau}} 
    \tilde{a}_{t-\tau+1} \xrightarrow{\cP} 
    \tilde{s}_{t-\tau+2} \xrightarrow{\btheta_{t-\tau}} 
    \tilde{a}_{t-\tau+2} \xrightarrow{\cP} 
    \cdots 
    \xrightarrow{\cP}
    \tilde{s}_{t} \xrightarrow{\btheta_{t-\tau}} 
    \tilde{a}_{t} \xrightarrow{\cP} 
    \tilde{s}_{t+1}.
\end{align*}
For reference, recall that the original Markov chain is given as:
\begin{align*}
    s_{t-\tau} \xrightarrow{\btheta_{t-\tau}} 
    a_{t-\tau} \xrightarrow{\cP} 
    s_{t-\tau+1} \xrightarrow{\btheta_{t-\tau+1}} 
    {a}_{t-\tau+1} \xrightarrow{\cP} 
    {s}_{t-\tau+2} \xrightarrow{\btheta_{t-\tau+2}} 
    {a}_{t-\tau+2} \xrightarrow{\cP} 
    \cdots 
    \xrightarrow{\cP}
    {s}_{t} \xrightarrow{\btheta_{t}} 
    {a}_{t} \xrightarrow{\cP} 
    {s}_{t+1}.
\end{align*}
Throughout this lemma, we always condition the expectation on $s_{t-\tau+1}$ and $\btheta_{t-\tau}$ and omit this in order to simplify the presentation. Under the setting introduced above, we have:
\begin{align}
    d_{TV} \big( \PP(s_{t+1} \in \cdot), \PP(\tilde{s}_{t+1} \in \cdot) \big) 
    & \le 
    d_{TV} \big( \PP(O_{t} \in \cdot), \PP(\tilde{O}_{t} \in \cdot) \big), \label{eqn:aux1}\\
    d_{TV} \big( \PP(O_{t} \in \cdot), \PP(\tilde{O}_{t} \in \cdot) \big) 
    & = 
    d_{TV} \big( \PP((s_t, a_t) \in \cdot), \PP((\tilde{s}_t, \tilde{a}_t) \in \cdot) \big), 
    \label{eqn:aux2}\\
    d_{TV} \big( \PP((s_t, a_t) \in \cdot), \PP((\tilde{s}_t, \tilde{a}_t) \in \cdot) \big) 
    & \le 
    d_{TV} \big( \PP(s_t \in \cdot), \PP((\tilde{s}_t \in \cdot) \big)
    +
    \frac{1}{2} |\cA| L \EE \big[\norm{\btheta_{t} - \btheta_{t-\tau}}\big]. \label{eqn:aux3}
\end{align}
\begin{proof} [Proof of \eqref{eqn:aux1}]
By the Law of Total Probability,
\begin{align*}
    \PP(s_{t+1} \in \cdot)
    & = 
    \int_{\cS} \sum_{\cA} 
    \PP(s_{t} = ds , a_{t} = a, s_{t+1} \in \cdot),
\end{align*}
and a similar argument also holds for $\tilde{O}_t$. Then we have 
\begin{align*}
    & 2d_{TV} \big( \PP(s_{t+1} \in \cdot), \PP(\tilde{s}_{t+1} \in \cdot) \big) \\
    & = 
    \int_{\cS} \bigg| \int_{\cS} \sum_{\cA} 
    \PP(s_{t} = ds , a_{t} = a, s_{t+1} = ds')
    -
    \int_{\cS} \sum_{\cA} 
    \PP(s_{t} = ds , a_{t} = a, s_{t+1} = ds') 
    \bigg| \\
    & \le 
    \int_{\cS} \int_{\cS} \sum_{\cA} 
    \big|  
    \PP(s_{t} = ds , a_{t} = a, s_{t+1} = ds')
    -
    \PP(s_{t} = ds , a_{t} = a, s_{t+1} = ds') \big| \\
    & = 
    \int_{\cS} \int_{\cS} \sum_{\cA} 
    \big|  
    \PP(O_t = (ds , a, ds'))
    -
    \PP(\tilde{O}_t = (ds , a, ds')) \big| \\ 
    & = 
    2d_{TV} \big( \PP(O_{t} \in \cdot), \PP(\tilde{O}_{t} \in \cdot) \big).
\end{align*}
The last equality requires exchange of integral, which should be guaranteed by the regularity. 
\end{proof}
\begin{proof} [Proof of \eqref{eqn:aux2}]
\begin{align*}
    & 2d_{TV} \big( \PP(O_{t} \in \cdot), \PP(\tilde{O}_{t} \in \cdot) \big) \\
    &=  
    \int_{\cS} \sum_{\cA} \int_{\cS}
    \big|  
    \PP(O_t  = (ds , a, ds'))
    -
    \PP(\tilde{O}_t  = (ds , a, ds')) \big| \\
    &=  
    \int_{\cS} \sum_{\cA} \int_{\cS}
    \big|  
    \cP(ds' | s, a) \PP((s_t, a_t)  = (ds, a))
    -
    \cP(ds' | s, a) \PP((\tilde{s}_t, \tilde{a}_t)  = (ds, a)) \big| \\
    &=  
    \int_{\cS} \sum_{\cA} \int_{\cS}
    \cP(ds' | s, a) \big|  
    \PP((s_t, a_t)  = (ds, a))
    -
    \PP((\tilde{s}_t, \tilde{a}_t)  = (ds, a)) \big| \\
    &=  
    \int_{\cS} \sum_{\cA} 
    \big|  
    \PP((s_t, a_t)  = (ds, a))
    -
    \PP((\tilde{s}_t, \tilde{a}_t)  = (ds, a)) \big| \\
   & =  
    2d_{TV} \big( \PP((s_t, a_t) \in \cdot), \PP((\tilde{s}_t, \tilde{a}_t) \in \cdot) \big).
\end{align*}
\end{proof}
\begin{proof}[Proof of \eqref{eqn:aux3}]
Because $\btheta_{t}$ is also dependent on $s_{t}$, we make it clear here that
\begin{align*}
    \PP\big((s_t, a_t) = (ds, a)\big)
    &=
    \int_{\btheta\in\RR^d}
    \PP(s_t = ds) \PP(\btheta_t = d\btheta | s_t = ds) \PP(a_t = a | s_t = ds, \btheta_t = d\btheta) \\
    & =
    \int_{\btheta\in\RR^d}
    \PP(s_t = ds) \PP(\btheta_t = d\btheta | s_t = ds) \pi_{\btheta_t}(a|ds) \\
    & = 
    \PP(s_t = ds)
    \int_{\btheta\in\RR^d}
    \PP(\btheta_t = d\btheta | s_t = ds) \pi_{\btheta_t}(a|ds) \\
    & = 
    \PP(s_t = ds) \EE\big[\pi_{\btheta_t}(a|ds) | s_t = ds\big].
\end{align*}
Therefore, the total variance can be bounded as
\begin{align*}
    & 2d_{TV} \big( \PP((s_t, a_t) \in \cdot), \PP((\tilde{s}_t, \tilde{a}_t) \in \cdot) \big) \\
   & =  
    \int_{\cS} \sum_{\cA} 
    \big |
    \PP(s_t = ds) \EE[ \pi_{\btheta_t}(a|ds) | s_t = ds]
    -
    \PP(\tilde{s}_t = ds) \pi_{\btheta_{t-\tau}}(a|ds)
    \big | \\
   & =  
    \int_{\cS} \sum_{\cA} 
    \big |
    \PP(s_t = ds) \EE[ \pi_{\btheta_t}(a|ds) | s_t = ds]
    -
    \PP(s_t = ds) \pi_{\btheta_{t-\tau}}(a|ds)
    \big | \\
    & \qquad+
    \int_{\cS} \sum_{\cA} 
    \big |
    \PP(s_t = ds) \pi_{\btheta_{t-\tau}}(a|ds)
    -
    \PP(\tilde{s}_t = ds) \pi_{\btheta_{t-\tau}}(a|ds)
    \big |\\
   & =  
    \int_{\cS} 
    \PP(s_t = ds) \sum_{\cA} 
    \big |
     \EE[ \pi_{\btheta_t}(a|ds) | s_t = ds]
    -
     \pi_{\btheta_{t-\tau}}(a|ds)
    \big | \\
    & \qquad +2d_{TV} \big( \PP(s_t \in \cdot), \PP((\tilde{s}_t \in \cdot) \big) \\
   & \le  
    |\cA| L \EE \big[\norm{\btheta_{t} - \btheta_{t-\tau}} \big] 
    +2d_{TV} \big( \PP(s_t \in \cdot), \PP((\tilde{s}_t \in \cdot) \big),
\end{align*}
where the inequality holds due to the Lipschitz continuity of the policy as in Assumption \ref{assum:policy-lipschitz-bounded}.
\end{proof}
\end{lemma}

\subsection{Lipschitzness of the Optimal Parameter} \label{subsec:proof-Lipschitz}
This section is used to present the proof of Proposition \ref{prop:optimal-lipschitz}.

\begin{proof}[Proof of Proposition \ref{prop:optimal-lipschitz}]
\citet{sutton2018reinforcement} has proved in Chapter 9 the fact that the linear TD(0) will converge to the optimal point (w.r.t. Mean Square Projected Bellman Error) which satisfies
\begin{align*}
    \Ab_i \bomega^*(\btheta_i) = \bbb_i,
\end{align*}
where $\Ab_i := \EE [\bphi(s)(\bphi(s) - \bphi(s'))^{\top}]$ and $\bbb_i := \EE [(r(s,a)-r(\btheta_i)) \bphi(s)]$. The expectation is taken over the stationary distribution $s \sim \mu_{\btheta_i}$, the action $a \sim \pi_{\btheta_i}(\cdot|s)$ and the transition probability kernel $s' \sim \cP(\cdot|s,a)$.

Now we denote $\bomega^*_1, \bomega^*_2, \hat{\bomega}_1$ as the unique solutions of the following equations respectively:
\begin{align*}
     \Ab_1 \bomega^*_1 & = \bbb_1, \\
     \Ab_2 \hat{\bomega}_1 & = \bbb_1, \\
     \Ab_2 \bomega^*_2 & = \bbb_2.
\end{align*}
First we bound $\norm{\bomega^*_1 - \hat{\bomega}_1}$. By definition, we have
\begin{align*}
    \norm{\bomega^*_1 - \hat{\bomega}_1} 
    & \le
    \norm{\Ab_{1}^{-1} - \Ab_{2}^{-1}}
    \norm{\bbb_1}.
\end{align*}
It can be easily shown that
\begin{align*}
    \Ab_{1}^{-1} - \Ab_{2}^{-1}
    & = 
    \Ab_{1}^{-1}(\Ab_{2} - \Ab_{1}) \Ab_{2}^{-1},
\end{align*}
which further gives
\begin{align*}
    \norm{\bomega^*_1 - \hat{\bomega}_1} 
    & \le
    \norm{\Ab_{1}^{-1}}
    \norm{\Ab_{1} - \Ab_{2}}
    \norm{\Ab_{2}^{-1}}
    \norm{\bbb_1}.
\end{align*}
Then we bound $\norm{\hat{\bomega}_1 - \bomega^*_2}$,
\begin{align*}
    \norm{\hat{\bomega}_1 - \bomega^*_2}
    & \le 
    \norm{\Ab_{2}^{-1}}
    \norm{\bbb_1 - \bbb_2}.
\end{align*}
By Assumption \ref{assum:negative-definite},
the eigenvalues of $\Ab_i$ are bounded from below by $\lambda > 0$, therefore $\norm{\Ab_i^{-1}} \le \lambda^{-1}$. Also $\norm{\bbb_1} \le U_r$ due to the assumption that $|r(s,a)| \le U_r$ and $\norm{\bphi(s)} \le 1$. To bound $\norm{\Ab_{1} - \Ab_{2}}$ and $\norm{\bbb_1 - \bbb_2}$, we first note that 
\begin{align*}
    \norm{\Ab_{1} - \Ab_{2}}_2
    & \le 
    \sup_{s,s'\in\cS} \big \|\bphi(s)(\bphi(s) -  \bphi(s'))^{\top}
    \big \|_2
    \cdot 
    2 d_{TV} \big(
    \PP(O^1 \in \cdot), \PP(O^2 \in \cdot) \big),
    \\
    & \le 
    4 d_{TV} \big(
    \PP(O^1 \in \cdot), \PP(O^2 \in \cdot) \big) \\
    \norm{\bbb_1 - \bbb_2}
    & \le
    \big\|
    \EE[r(s^{1},a^{1})\bphi(s^1)]
    -
    \EE[r(s^{2},a^{2})\bphi(s^2)]
    \big\|
    +
    \big\|
    r(\btheta_1)\EE[\bphi(s^1)]
    -
    r(\btheta_2)\EE[\bphi(s^2)]
    \big\|
    \\
    & \le 
    6 U_r d_{TV} \big(
    \PP(O^1 \in \cdot), \PP(O^2 \in \cdot) \big)
    ,
\end{align*}
where $O^i$ is the tuple obtained by $s^i \sim \mu_{\btheta_i}(\cdot)$, $a^i \sim \pi_{\btheta_i}(\cdot|s^i)$ and $(s')^i \sim \cP(\cdot| s^i, a^i)$. And the total variation norm can be bounded by Lemma \ref{lemma:prob-mixing} as:
\begin{align*}
    d_{TV} \big(
    \PP(O^1 \in \cdot), \PP(O^2 \in \cdot) \big)
    & \le 
    |\cA|L \bigg(1 + \lceil \log_{\rho}m^{-1} \rceil + \frac{1}{1-\rho} \bigg) \norm{\btheta_1 - \btheta_2}.
\end{align*}
Collecting the results above gives
\begin{align*}
    \norm{\bomega^*_1 - \bomega^*_2}
    & \le 
    \norm{\bomega^*_1 - \hat{\bomega}_1}
    +
    \norm{\hat{\bomega}_1 - \bomega^*_2} \\ 
    & \le 
    (2\lambda^{-2} U_r  + 3\lambda^{-1} U_r) |\cA|L \bigg(1 + \lceil \log_{\rho}m^{-1} \rceil + \frac{1}{1-\rho} \bigg) \norm{\btheta_1 - \btheta_2},
\end{align*}
and we set $L_* := (2\lambda^{-2} U_r  + 3\lambda^{-1} U_r) |\cA|L (1 + \lceil \log_{\rho}m^{-1} \rceil + 1/(1-\rho) )$ to obtain the final result.
\end{proof}

\subsection{Asymptotic Equivalence} \label{subsec:proof-equi-asym}
\begin{lemma} \label{lemma:equi-asym}
Suppose $\{ a_i \}$ is a non-negative, bounded sequence, $\tau := C_1 + C_2 \log t (C_2>0)$, then for any large enough $t$ such that $t \ge \tau > 0$, we have:
\begin{align*}
    \frac{1}{1+t-\tau} \sum_{k=\tau}^{t} a_i
    & = 
    \cO
    \bigg(
    \frac{1}{t} \sum_{k=1}^{t} a_i
    \bigg), \\
    \frac{1}{t} \sum_{k=1}^{t} a_i
    & = 
    \cO\bigg(
    \frac{\log t}{t}
    \bigg)
    +
    \cO
    \bigg(
    \frac{1}{1+t-\tau} \sum_{k=\tau}^{t} a_i
    \bigg).
\end{align*}
\end{lemma}
\begin{proof}
We know that $\tau = \cO(\log t)$ and the sequence is bounded: $0 < a_i < B$. For the first equation, we have
\begin{align*}
    \frac{1}{1+t-\tau} \sum_{k=\tau}^{t} a_i
    & \le 
    \frac{1}{1+t-\tau} \sum_{k=1}^{t} a_i
    \le 
    \frac{t}{1+t-\tau}
    \cdot
    \frac{1}{t}
    \sum_{k=1}^{t} a_i
    \le
    \cO
    \bigg(
    \frac{1}{t}
    \sum_{k=1}^{t} a_i
    \bigg),
\end{align*}
and further assuming $t \ge 2 \tau - 2$ gives a constant $2$.
For the second equation, we have
\begin{align*}
    \frac{1}{t} \sum_{k=1}^{t} a_i
    & \le 
    \frac{1}{t} 
    \bigg(
    (\tau-1) B +
    \sum_{k=\tau}^{t} a_i
    \bigg)
    = 
    \frac{\tau-1}{t}
    B
    +
    \frac{1}{t}
    \sum_{k=\tau}^{t} a_i
    =
    \cO \bigg(
    \frac{\log t}{t}
    \bigg)
    +
    \cO \bigg( 
    \frac{1}{1+t-\tau} \sum_{k=\tau}^{t} a_i
    \bigg).
\end{align*}
\end{proof}

\section{Proof of Main Theorems and Propositions}\label{sec:proof_of_maintheory}

\subsection{Proof of Theorem \ref{thm:actor}} \label{subsec:proof-actor}
We first define several notations to clarify the dependence:
\begin{align} \label{def:actor-term}
    O_t 
    : & = 
    (s_t, a_t, s_{t+1}), 
    \notag \\
    \eta^*
    : & = 
    r(\btheta)
    = \EE_{s\sim \mu_{\btheta}, a\sim \pi_{\btheta}(\cdot|s)}[r(s,a)] 
    \notag \\
    \Delta h(O, \eta, \bomega, \btheta) 
    :& = 
    \Big(
    r(\btheta) - \eta +
    \big(\bphi(s') - \bphi(s) \big)^{\top} (\bomega - \bomega^*(\btheta)) 
    \Big)
    \nabla \log \pi_{\btheta}(a|s),
    \notag \\
    \Delta h'(O, \btheta)
    : & = 
    \Big(
    \big( \bphi(s')^{\top} \bomega^*(\btheta) - V^{\pi_{\btheta}}(s') \big)
    -
    \big( \bphi(s)^{\top} \bomega^*(\btheta) - V^{\pi_{\btheta}}(s) \big)
    \Big)  
    \nabla \log \pi_{\btheta}(a|s),
    \notag \\
    h(O, \btheta) 
    :& = 
    {\big( r(s, a) - r(\btheta) +  \bphi(s')^{\top} \bomega^*(\btheta) - \bphi(s)^{\top} \bomega^*(\btheta) \big) \nabla \log \pi_{\btheta}(a|s)} ,
    \notag \\
    \Gamma(O, \btheta) 
    :& =
    {\big \la \nabla J(\btheta), h(O,\btheta) - \EE_{s \sim \mu_{\btheta}, a \sim \pi_{\btheta}, s' \sim \PP}
    \big[ 
    h(O', \btheta)
    \big]
    \big \ra}.
\end{align}
{In the following proof, we also denote $\eta^*_t = r(\btheta_t)$. When the context is clear, $\bomega^*$ denotes $\bomega^*(\btheta)$.}
Note that $\Delta h$, $\Delta h'$ and $h - \Delta h'$ together give a decomposition of the actor update {($\Delta h + h$)} we use in Algorithm~\ref{alg:2ts_ac}. They respectively correspond to the error caused by the critic $\bomega_t$ and $\eta_t$, the approximation error of the linear class, and the stochastic policy gradient.

$\Gamma(O, \btheta)$ is the Markovian noise for $h(O, \btheta)$. 
Here $O' = (s,a,s')$ is a shorthand for an independent sample from $s \sim \mu_{\btheta}, a \sim \pi_{\btheta}, s' \sim \PP$.
Using a more compact notation $\EE_{O'}[ \cdot]$, it is clear we have 
\begin{align}
    \EE_{O'}[h(O', \btheta) - \Delta h'(O', \btheta) ]
    & =
    \EE_{O'}
    \Big[
    \big(
    r(s,a)
    -
    r(\btheta)
    +
    V^{\pi_{\btheta}}(s')
    -
    V^{\pi_{\btheta}}(s)
    \big)  
    \nabla \log \pi_{\btheta}(a|s)
    \Big]
    \notag \\
    & = \nabla J(\btheta), \label{eqn:diff-h-dhp}
\end{align}
and $\EE_{O'} \| \Delta h'(O, \btheta) \|^2 \le 4B^2 \epsilon_{\text{app}}^2 $ because
\begin{align*}
    \EE_{O'} \| \Delta h'(O, \btheta) \|^2
    & = 
    \EE_{O'} 
    \bigg \|
    \Big(
    \big( \bphi(s')^{\top} \bomega^* - V^{\pi_{\btheta}}(s') \big)
    -
    \big( \bphi(s)^{\top} \bomega^* - V^{\pi_{\btheta}}(s) \big)
    \Big)  
    \nabla \log \pi_{\btheta}(a|s)
    \bigg \|^2
    \\
    &  \le 
    \EE_{O'} 
    \bigg [
    B^2 
    \Big(
    \big( \bphi(s')^{\top} \bomega^* - V^{\pi_{\btheta}}(s') \big)
    -
    \big( \bphi(s)^{\top} \bomega^* - V^{\pi_{\btheta}}(s) \big)
    \Big)^2
    \bigg ]
    \\
    &  
    \le 
    \EE_{O'} 
    \Big [
    2B^2 
    \big( \bphi(s')^{\top} \bomega^* - V^{\pi_{\btheta}}(s') \big)^2
    +
    \big( \bphi(s)^{\top} \bomega^* - V^{\pi_{\btheta}}(s) \big)^2
    \Big ]
    \\
    & =
    4 B^2 
    \EE_{O'} \Big[ 
    \big( \bphi(s)^{\top} \bomega^*(\btheta) - V^{\pi_{\btheta}}(s) \big)^2
    \Big]
    \\
    & =
    4 B^2 \epsilon^2_{\text{app}}.
\end{align*}

There are several lemmas that will be used in the proof.
\begin{lemma}\label{lemma:J-smooth}
For the performance function defined in \eqref{Eq:J-function}, there exists a constant $L_J>0$ such that for all $\btheta_1,\btheta_2\in\RR^d$, it holds that
\begin{align*}
    \big\| \nabla {J(\btheta_1)} - \nabla {J(\btheta_2)} \big\|
    \le L_{J} \norm{\btheta_1 - \btheta_2},
\end{align*}
which by the definition of smoothness \citep{nesterov2018lectures} implies
\begin{align*}
    J(\btheta_2) \ge 
    J(\btheta_1) + \big \la\nabla J(\btheta_1), \btheta_2 - \btheta_1\big \ra - \frac{L_{J}}{2} \norm{\btheta_1 - \btheta_2}^2. 
\end{align*}
\end{lemma}
The following two lemmas characterize the bias introduced by the critic's approximation and the Markovian noise.
\begin{lemma} \label{lemma:critic-bias}
For any $t \ge 0$,
\begin{align*}
    \big\|
    \Delta h(O_t, \eta_t, \bomega_t, \btheta_t) 
    \big\|^2
    \le 
    B^2 
    \big( 
    8\norm{\bomega_t - \bomega^*_t}^2 + 2(\eta_t - \eta^*_t)^2
    \big) .
\end{align*}

\end{lemma}

\begin{lemma} \label{lemma:actor-markovian}
For any $\btheta\in\RR^d$, we have $\|\delta \nabla \log \pi_{\btheta}(a|s)\|  \le G_{\btheta} := U_\delta \cdot B$, where $U_{\delta}=2U_r+2R_{\bomega}$. Furthermore, for any $t \ge 0$, it holds that
\begin{align*}
    \EE\big[\Gamma(O_t, \btheta_t)\big]
   & \ge 
    -G_{\btheta} \big(D_{1} (\tau + 1) \sum_{k=t-\tau+1}^t \EE \norm{\btheta_k - \btheta_{k-1}}
    +D_{2} m \rho^{\tau - 1}\big),
\end{align*}
where {$D_1:= \max \{ 2 L_J +3 L_h, 2U_{\delta} B  |\cA| L \}$} and $D_2 = 4 U_{\delta} B$. {Here $L_h = U_{\delta} L_{l}
    +
    (2+ 2 \lambda^{-2} + 3 \lambda^{-1}) B U_r|\cA|L \big(1 + \lceil \log_{\rho}m^{-1} \rceil + 1/(1-\rho)\big)$.}
\end{lemma}

\begin{proof}[Proof of Theorem \ref{thm:actor}]

Under the update rule of Algorithm \ref{alg:2ts_ac}, we have 
\begin{align}
    J(\btheta_{t+1}) 
    & \ge
    J(\btheta_{t}) 
    + \alpha_{t} \big \la \nabla J (\btheta_{t}),\delta_t \nabla \log \pi_{\btheta_{t}}(a_t | s_t)\big \ra
    - L_{J} \alpha_{t}^{2} 
    \big\|\delta_{t} \nabla \log {\pi_{\btheta_t}}(a_t | s_t) \big\|^2 
    \notag \\ 
    & =
    J(\btheta_{t}) 
    + \alpha_{t} \big\la \nabla J (\btheta_{t}), \Delta h(O_t, \eta_t, \bomega_t, \btheta_t) \big\ra
    %+ \alpha_{t} \big\la \nabla J (\btheta_{t}), \Delta h'(O_t, \btheta_t) \big\ra
    \notag\\
    & \qquad
    + \alpha_{t} \big \la \nabla J (\btheta_{t}), h(O_{t}, \btheta_t) \big \ra
    - L_{J} \alpha_{t}^{2} \big \| \delta_{t} \nabla \log {\pi_{\btheta_t}}(a_t | s_t) \big\|^2 
    \notag \\
    & =
    J(\btheta_{t}) 
    + \alpha_{t} \big \la \nabla J (\btheta_{t}), \Delta h(O_t, \eta_t, \bomega_t, \btheta_t) \big \ra
    %+ \alpha_{t} \big\la \nabla J (\btheta_{t}), \EE_{O'} [\Delta h'(O', \btheta_t)] \big\ra
    \notag\\
    & \qquad
    +
    \alpha_{t} \big \la \nabla J (\btheta_{t}), \EE_{O'} [h(O', \btheta_t)] \big \ra
    + \alpha_{t}  \Gamma(O_t,  \btheta_t)
    - L_{J} \alpha_{t}^{2} \big\| \delta_{t} \nabla \log {\pi_{\btheta_t}}(a_t | s_t) \big\|^2
    \notag \\
    & =
    J(\btheta_{t}) 
    + \alpha_{t} \big \la \nabla J (\btheta_{t}), \Delta h(O_t, \eta_t, \bomega_t, \btheta_t) \big \ra
    + \alpha_{t} \big\la \nabla J (\btheta_{t}),  \EE_{O'} [\Delta h'(O', \btheta_t)] \big\ra
    \notag\\
    & \qquad
    + \alpha_{t} \big \|\nabla J (\btheta_{t}) \big\|^2 
    + \alpha_{t}  \Gamma(O_t,  \btheta_t)
    - L_{J} \alpha_{t}^{2} \big\| \delta_{t} \nabla \log {\pi_{\btheta_t}}(a_t | s_t) \big\|^2.
    \label{eqn:actor-all-term}
\end{align}
The first inequality is by Lemma \ref{lemma:J-smooth} (we discard the $1/2$ in front of the square-norm term). The first equality is by the definitions in~\eqref{def:actor-term}; the second equality is by the definition of $\Gamma(O_t, \btheta_t)$ in ~\eqref{def:actor-term}. 
{
The last equality is due to \eqref{eqn:diff-h-dhp}.
Here $O' = (s,a,s')$ is a shorthand for an independent sample from $s \sim \mu_{\btheta_t}, a \sim \pi_{\btheta_t}, s' \sim \cP(\cdot|s,a)$.
}

We will bound the expectation of each term on the right hand side of \eqref{eqn:actor-all-term} as follows. First, we have
\begin{align*}
    \EE \big\la \nabla J (\btheta_{t}), \Delta h(O_t, \eta_t, \bomega_t, \btheta_t) \big\ra
    & \ge 
    - B \sqrt{\EE \big\| \nabla J (\btheta_{t})\big\|^2}  \sqrt{8\EE \norm{\zb_t}^2 + 2\EE[y_t^2]},
\end{align*}
where $\zb_t := \bomega_t - \bomega^*_t$ and $y_t := \eta_t - \eta^*_t$, and the inequality is due to Cauchy inequality and Lemma~\ref{lemma:critic-bias}.

Second, we have
\begin{align*}
    \EE [\Gamma(O_t,  \btheta_t)]
    & \ge 
    -G_{\btheta} \bigg(D_{1} (\tau + 1) \sum_{k=t-\tau+1}^t \EE \norm{\btheta_k - \btheta_{k-1}}
    +D_{2} m \rho^{\tau - 1}\bigg), \\
    & \ge 
    - G_{\btheta}
    \bigg(
    D_1 (\tau + 1) G_{\btheta} \sum_{k=t-\tau+1}^{t-1} \alpha_k 
    + 
    D_2 m \rho^{\tau - 1}
    \bigg),
\end{align*}
where the first inequality is due to Lemma~\ref{lemma:actor-markovian}, and the second inequality is due to %the fact that $\|\btheta_t-\btheta_{t-\tau}\|\leq\sum_{k=t-\tau}^{t-1}\|\btheta_{k+1}-\btheta_{k}\|$ and that
$\big\| \delta_{t} \nabla \log {\pi_{\btheta_t}}(a_t | s_t) \big\| \le G_{\btheta}$ by Lemma~\ref{lemma:actor-markovian}. 

Third, by the remarks under~\eqref{def:actor-term} regarding $\Delta h'$, we have
\begin{align*}
    \big\la \nabla J (\btheta_{t}), \EE_{O'} [\Delta h'(O_t, \btheta_t)] \big\ra
    & \ge 
    - G_{\btheta}  \sqrt{
    \big \|
    \EE_{O'} [\Delta h'(O_t, \btheta_t)]
    \big \|^2
    }
    \\
    & \ge 
    - G_{\btheta}  \sqrt{
    \EE_{O'} 
    \big \| 
    \Delta h'(O_t, \btheta_t)
    \big \|^2
    }
    \\
    & \ge 
    -2B G_{\btheta} \epsilon_{\text{app}},
\end{align*}

Taking the expectation of \eqref{eqn:actor-all-term} and plugging the above terms back into it gives 
\begin{align*}
    \EE [J(\btheta_{t+1})]
    &\ge 
    \EE [J(\btheta_{t}) ]
    - \alpha_{t} 
    B \sqrt{\EE \big\| \nabla J (\btheta_{t}) \big\|^2} \sqrt{8\EE \norm{\zb_t}^2 + 2 \EE [y_t^2]}
    -
    2BG_{\btheta} \epsilon_{\text{app}}
    \alpha_{t}
    \\
    & \qquad
    - \alpha_{t} 
    G_{\btheta} 
    \bigg(
    D_1 (\tau + 1) G_{\btheta} \sum_{k=t-\tau}^{t-1} \alpha_k 
    + 
    D_2 m \rho^{\tau - 1}
    \bigg) 
    + \alpha_{t} \EE \norm{\nabla J (\btheta_{t})}^2 
    - L_{J} G_{\btheta}^2 \alpha_{t}^{2}.
\end{align*}
Rearranging the above inequality gives
\begin{align*}
    \EE \big\| \nabla J(\btheta_t) \big\|^2
    & \le 
    \frac{1}{\alpha_t}
    \big(
    \EE [J(\btheta_{t+1})] - \EE [J(\btheta_{t})]
    \big)
    + 
    B \sqrt{\EE \big\| \nabla J(\btheta_t) \big\|^2} \sqrt{8\EE \norm{\zb_t}^2 + 2 \EE[y_t^2]} \\
    & \qquad 
    +
    2BG_{\btheta} \epsilon_{\text{app}}
    + 
    D_1 G_{\btheta}^2 (\tau + 1) \sum_{k=t-\tau}^{t-1} \alpha_k  
    +
    D_2 G_{\btheta} m \rho^{\tau - 1} 
    +
    L_{J} G_{\btheta}^2 \alpha_t.
\end{align*}
By setting $\tau = \tau_{t}$, we get 
\begin{align*}
    \EE \big\| \nabla J(\btheta_t) \big\|^2
    & \le 
    \frac{1}{\alpha_t}
    \big(
    \EE \big[J(\btheta_{t+1})\big] - \EE\big[ J(\btheta_{t})\big]
    \big)
    + 
    B \sqrt{\EE \big\| \nabla J(\btheta_t) \big\|^2} \sqrt{8\EE \norm{\zb_t}^2 + 2 \EE[y_t^2]}
    \\
    & \qquad + 
    2BG_{\btheta} \epsilon_{\text{app}}
    +
    D_1 G_{\btheta}^2 (\tau_{t} + 1)^2   \alpha_{t - \tau_{t}}
    + 
    D_2 G_{\btheta} \alpha_t 
    + 
    L_{J} G_{\btheta}^2 \alpha_t.
\end{align*}
Summing over $k$ from $\tau_{t}$ to $t$ gives
\begin{align*}
    \sum_{k=\tau_{t}}^{t} \EE \big\| \nabla J(\btheta_t) \big\|^2
    & \le 
    \underbrace{\sum_{k=\tau_{t}}^{t} \frac{1}{\alpha_k}
    \big(
    \EE[ J(\btheta_{k+1})] - \EE [J(\btheta_{k})]
    \big)}_{I_1} 
    \\
    & \qquad + 
    B \sum_{k=\tau_{t}}^{t}
    \sqrt{\EE \big\| \nabla J(\btheta_t) \big\|^2} \sqrt{8\EE \norm{\zb_t}^2 + 2 \EE[y_t^2]}\\
    & \qquad + 
    \underbrace{
    \sum_{k=\tau_{t}}^{t}
    D_1 G_{\btheta}^2 (\tau_{t} + 1)^2   \alpha_{k - \tau_{t}}  + 
    \sum_{k=\tau_{t}}^{t}
    (D_2 G_{\btheta}+L_{J} G_{\btheta}^2/2) \alpha_k}_{I_2}
    \\
    & \qquad + 2BG_{\btheta} \epsilon_{\text{app}}(t- \tau_t + 1)
    .
\end{align*}
For the term $I_1$, we have,
\begin{align*}
    \sum_{k=\tau_{t}}^{t} \frac{1}{\alpha_k}
    \big(
    J(\btheta_{k+1}) - J(\btheta_{k})
    \big)
    & =
    \sum_{k=\tau_{t}}^{t}
    \bigg(
    \frac{1}{\alpha_{k-1}} - \frac{1}{\alpha_{k}}
    \bigg) \EE [J(\btheta_k)]
    - \frac{1}{\alpha_{\tau_{t}-1}} \EE [J(\btheta_{\tau_{t}})]
    + \frac{1}{\alpha_{t}} \EE[J(\btheta_{t+1})] \\ 
    & \le 
    \sum_{k=\tau_{t}}^{t}
    \bigg(\frac{1}{\alpha_{k}} - \frac{1}{\alpha_{k-1}}\bigg) 
    U_r
    + \frac{1}{\alpha_{\tau_{t}-1}}
    U_r
    +
    \frac{1}{\alpha_{t}} U_r
     \\
    & =
    U_r
    \bigg[ 
    \sum_{k=\tau_{t}}^{t}
    \bigg( \frac{1}{\alpha_{k}} - \frac{1}{\alpha_{k-1}} \bigg)
    +
    \frac{1}{\alpha_{\tau_{t}-1}}
    +
    \frac{1}{\alpha_{t}}
    \bigg] \\
    & = 
    2U_r \alpha_{t}^{-1},
\end{align*}
where the inequality holds due to $|\EE [ J(\btheta)] | \le U_r$.

\noindent For the term $I_2$, we have
\begin{align*}
    \sum_{k=\tau_{t}}^{t}
    D_1 G_{\btheta}^2 (\tau_{t} + 1)^2   \alpha_{k - \tau_{t}} 
    & = 
    D_1 G_{\btheta}^2 (\tau_{t} + 1)^2
    \sum_{k=\tau_{t}}^{t}
    \alpha_{k - \tau_{t}}  \\ 
    & = 
    D_1 G_{\btheta}^2 (\tau_{t} + 1)^2
    \sum_{k=0}^{t-\tau_{t}}
    \alpha_{k} \\
    & = 
    D_1 G_{\btheta}^2 (\tau_{t} + 1)^2
    c_{\alpha} \sum_{k=0}^{t-\tau_{t}}
    \frac{1}{(1+k)^{\sigma}} ,
\end{align*}
and
\begin{align*}
    \sum_{k=\tau_{t}}^{t}
    (D_2 G_{\btheta}+L_{J} G_{\btheta}^2) \alpha_{k} 
    & = 
    (D_2 G_{\btheta}+L_{J} G_{\btheta}^2/2)
    \sum_{k=\tau_{t}}^{t}
    \alpha_{k}  \\ 
    & \le 
    (D_2 G_{\btheta}+L_{J} G_{\btheta}^2/2)
    \sum_{k=0}^{t-\tau_{t}}
    \alpha_{k} \\
    & = 
    (D_2 G_{\btheta}+L_{J} G_{\btheta}^2/2)
    c_{\alpha} \sum_{k=0}^{t-\tau_{t}}
    \frac{1}{(1+k)^{\sigma}}.
\end{align*}
Note that both upper bounds rely on the summation $\sum_{k=0}^{t-\tau_{t}}
    1/(1+k)^{\sigma} \le \int_{0}^{t-\tau_{t}+1} x^{-\sigma} dx = 1/(1-\sigma) (t-\tau_{t}+1)^{1-\sigma}$. Combining the results for terms $I_1$ and $I_2$, we have
\begin{align*}
    \sum_{k=\tau_{t}}^{t} \EE \big\| \nabla J(\btheta_t) \big\|^2
    & \le 
    \frac{2U_r}{ c_{\alpha}} (1+t)^{\sigma} \\
    & \qquad + 
    \big(
    D_1 G_{\btheta}^2 (\tau_{t} + 1)^2
    +
    D_2 G_{\btheta}+L_{J} G_{\btheta}^2
    \big) 
   \frac{c_{\alpha}}{1-\sigma} (t-\tau_{t}+1)^{1-\sigma}
    \\
    & \qquad +
    B \sum_{k=\tau_{t}}^{t}
    \sqrt{\EE \big\| \nabla J(\btheta_t) \big\|^2} \sqrt{8\EE \norm{\zb_t}^2 + 2 \EE[y_t^2]}
    \\
    & \qquad +
    2BG_{\btheta} \epsilon_{\text{app}} (t- \tau_t + 1).
\end{align*}
Dividing $(1 + t - \tau_{t})$ at both sides and assuming $t > 2\tau_t-1$, we can express the result as 
\begin{align}
    \frac{1}{1+t-\tau_{t}}\sum_{k=\tau_{t}}^{t} \EE \big\| \nabla J(\btheta_t) \big\|^2 &\le 
    \frac{4U_r}{c_\alpha}
    \frac{1}
    {(t+1)^{1 - \sigma}} 
    \notag\\
    & \qquad 
    + 
    \big(
    D_1 G_{\btheta}^2 (\tau_{t} + 1)^2
    +
    D_2 G_{\btheta}+L_{J} G_{\btheta}^2
    \big) 
   \frac{c_{\alpha}}{1-\sigma} \frac{1}{(t-\tau_{t}+1)^{\sigma}}
    \notag\\
    & \qquad +
    \frac{2B}{1+t-\tau_{t}}
     \sum_{k=\tau_{t}}^{t}
    \sqrt{\EE \big\| \nabla J(\btheta_t) \big\|^2} \sqrt{8\EE \norm{\zb_t}^2 + 2 \EE[y_t^2]}
    \notag\\
    & \qquad +
    2BG_{\btheta} \epsilon_{\text{app}}
    .
    \label{eqn:actor-bone}
\end{align}
By Cauchy-Schwartz inequality, we have
\begin{align*}
    &\frac{1}{1+t-\tau_{t}}
    \sum_{k=\tau_{t}}^{t}
    \sqrt{\EE \big\| \nabla J(\btheta_t) \big\|^2} \sqrt{8\EE \norm{\zb_t}^2 + 2 \EE[y_t^2]}\\
    & \qquad \le 
    \bigg(
    \frac{1}{1+t-\tau_{t}}\sum_{k=\tau_{t}}^{t}
    \EE \big\| \nabla J(\btheta_t) \big\|^2
    \bigg)^{\frac{1}{2}}
    \bigg(
    \frac{1}{1+t-\tau_{t}}
    \sum_{k=\tau_{t}}^{t}
    \big(
    8\EE \norm{\zb_t}^2 + 2 \EE[y_t^2]
    \big)
    \bigg)^{\frac{1}{2}}.
\end{align*}
Now, denote $F(t) := 1/(1+t-\tau_{t})\sum_{k=\tau_{t}}^{t} \EE \norm{\nabla J(\btheta_k)}^2$ and $Z(t) := 1/(1+t-\tau_{t}) \sum_{k=\tau_{t}}^{t} \big( 8\EE \norm{\zb_t}^2 + 2 \EE[y_t^2] \big)$, and putting them back to \eqref{eqn:actor-bone} ($\cO$-notation for simplicity):
\begin{align*}
    F(t)
    & \le 
    \cO
    \bigg( 
    \frac{1}{t^{1-\sigma}} 
    \bigg)
    + 
    \cO
    \bigg(
    \frac{(\log t)^2}{t^{\sigma}} 
    \bigg)
    +
    \cO(\epsilon_{\text{app}})
    +
    2B \sqrt{F(t)} \cdot \sqrt{Z(t)}
    ,
\end{align*}
which further gives
\begin{align} \label{eqn:squared}
    \big(
    \sqrt{F(t)} - B \sqrt{Z(t)}
    \big)^2 
    & \le 
    \cO
    \bigg( 
    \frac{1}{t^{1-\sigma}} 
    \bigg)
    + 
    \cO
    \bigg(
    \frac{(\log t)^2}{t^{\sigma}} 
    \bigg)
    +
    \cO(\epsilon_{\text{app}})
    +
    B^2 Z(t)
    .
\end{align}
Note that for a general function $H(t) \le  A(t) + B(t)$(with each positive), we have 
\begin{align*}
    H^2(t) & \le
    2A^2(t)
    +
    2B^2(t), \\ 
    \sqrt{H(t)} 
    & \le
    \sqrt{A(t)}
    +
    \sqrt{B(t)}.
\end{align*}
This means \eqref{eqn:squared} implies
\begin{align*}
    \sqrt{F(t)} - B \sqrt{Z(t)}
    & \le 
    \sqrt{A(t)}
    +
    B
    \sqrt{Z(t)}
    , \\
    \sqrt{F(t)}
    & \le 
    \sqrt{A(t)}
    +
    2B
    \sqrt{Z(t)}, \\
    F(t)
    & \le 
    2A(t)
    +
    8B^2 Z(t).
\end{align*}
By Lemma \ref{lemma:equi-asym}, assuming $t \ge 2 \tau_t - 1$, it holds that
\begin{align*}
    Z(t)
    & = 
    \frac{1}{1+t-\tau_{t}}\sum_{k=\tau_{t}}^{t} 8 \EE \norm{\zb_k}^2 
    +
    2 \EE [y_t^2]
    \le 
    \frac{2}{t} \sum_{k=1}^{t} 8 \EE \norm{\zb_k}^2
    +
    2 \EE [y_t^2]
    =
    2
    \cE(t).
\end{align*}
And finally, we have
\begin{align*}
    \min_{0 \le k \le t} 
    \EE \big\|\nabla J(\btheta_k)\big\|^2 
    & \le 
    \frac{1}{1+t-\tau_{t}}\sum_{k=\tau_{t}}^{t} \EE \big\|\nabla J(\btheta_k)\big\|^2 \\
    & \le 
    \frac{8U_r}{c_\alpha}
    \frac{1}
    {(t+1)^{1 - \sigma}} 
    \notag\\
    & \qquad 
    + 
    \big(
    D_1 G_{\btheta}^2 (\tau_{t} + 1)^2
    +
    D_2 G_{\btheta}+L_{J} G_{\btheta}^2
    \big) 
   \frac{2c_{\alpha}}{1-\sigma} \frac{1}{(t-\tau_{t}+1)^{\sigma}}
    \\
    & \qquad +
    4BG_{\btheta} \epsilon_{\text{app}}
    \\
    & \qquad +
    16B^2 \cE(t)
    \\
    & = 
    \cO
    \bigg( 
    \frac{1}{t^{1-\sigma}}
    \bigg)
    +
    \cO
    \bigg( 
    \frac{1}{t^{\sigma}}
    \bigg)
    +
    \cO
    (\epsilon_{\text{app}})
    +
    \cO
    \big( 
    \cE(t)
    \big).
\end{align*}
\end{proof}

\subsection{Proof of Theorem \ref{thm:critic}: Estimating the Average Reward} \label{subsec:proof-eta}
We define several notations to clarify the probabilistic dependency. 
\begin{align}\label{eq:def_notation_Ot}
\begin{split}
    O_t 
    : & = 
    (s_t, a_t, s_{t+1}),
    \\
    \eta^*_t 
    : & =
    r(\btheta_t),
    \\
    y_t
    : & = 
    \eta_t - \eta^*_t,
    \\
    \Xi(O_t, \eta_t, \btheta_t) 
    : & = 
    y_t(r(s_t,a_t) - \eta^*_t). 
\end{split}    
\end{align}
We also write $J(\btheta_t)=r(\btheta_t)$ sometimes in the proof.

\begin{lemma} \label{lemma:J-Lipschitz}
For any $\btheta_1, \btheta_2$, we have
\begin{align*}
    \big|J(\btheta_1) - J(\btheta_2) \big|
    & \le 
    C_J \norm{\btheta_1 - \btheta_2},
\end{align*}
where $C_J = 2 U_r |\cA|L (1 + \lceil \log_{\rho}m^{-1} \rceil + 1/(1-\rho))$.
\end{lemma}

\begin{lemma} \label{lemma:eta-Markovian}
Given the definition of $\Xi(O_t, \eta_t, \btheta_t)$, for any $t > 0$, we have 
\begin{align*}
    \EE [\Xi(O_t, \eta_t, \btheta_t)]
    & \le 
    4U_r C_J \norm{\btheta_t - \btheta_{t-\tau}}
    + 
    2 U_r |\eta_t - \eta_{t-\tau}|
    \\
    & \qquad 
    + 2 U_{r}^2 |\cA| L \sum_{i=t-\tau}^{t} \EE \norm{\btheta_i - \btheta_{t-\tau}}.
    + 4 U_r^2  m \rho^{\tau - 1}
    .
\end{align*}
\end{lemma}

\begin{proof}
From the definition, $\eta_t$ is the average reward estimator, $\eta_t^* = J(\btheta_t) = \EE[r(s,a)]$ is the average reward under the stationary distribution $\mu_{\btheta_t} \otimes \pi_{\btheta_t}$, and $y_t = \eta_t - \eta_t^*$. From the algorithm we have the update rule as
\begin{align*}
    \eta_{t+1} & :=
    \eta_t + \gamma_t \big(r(s_t,a_t) - \eta_t \big),
\end{align*}
where we leave the step size $\gamma_t$ unspecified for now. Unrolling the recursive definition we have 
\begin{align*}
    y_{t+1}^2
    & =
    \big(y_t + \eta^*_t - \eta^*_{t+1} + \gamma_t(r_t - \eta_t)
    \big)^2 \\ 
    & \le
    y_t^2 
    + 2 \gamma_t y_t (r_t - \eta_t)
    + 2 y_t (\eta^*_t - \eta^*_{t+1})
    + 2 (\eta^*_t - \eta^*_{t+1})^2
    +
    2 \gamma_t^2 (r_t - \eta_t)^2 \\ 
    & = 
    (1- 2 \gamma_t) y_t^2
    + 2 \gamma_t y_t (r_t - \eta^*_t)
    + 2 y_t (\eta^*_t - \eta^*_{t+1})
    + 2 (\eta^*_t - \eta^*_{t+1})^2
    +
    2 \gamma_t^2 (r_t - \eta_t)^2 \\
    & = 
    (1- 2 \gamma_t) y_t^2
    + 2 \gamma_t \Xi(O_k, \eta_k, \btheta_k)
    + 2 y_t (\eta^*_t - \eta^*_{t+1})
    + 2 (\eta^*_t - \eta^*_{t+1})^2
    +
    2 \gamma_t^2 (r_t - \eta_t)^2.
\end{align*}
Rearranging and summing from $\tau_t$ to $t$, we have
\begin{align*}
    \sum_{k=\tau_t}^{t} \EE[y_k^2 ]
    & \le 
    \underbrace{
    \sum_{k=\tau_t}^{t}
    \frac{1}{2 \gamma_k}
    \EE(y_k^2 - y_{k+1}^2)}_{I_1}
    +
    \underbrace{
    \sum_{k=\tau_t}^{t}
    \EE [\Xi(O_k, \eta_k, \btheta_k)]}_{I_2} \\
    & \qquad +
    \underbrace{
    \sum_{k=\tau_t}^{t}
    \frac{1}{\gamma_k} \EE[y_k (\eta^*_k - \eta^*_{k+1})]
    }_{I_3}
    +
    \underbrace{
    \sum_{k=\tau_t}^{t}
    \frac{1}{\gamma_k}
    \EE[(\eta^*_k - \eta^*_{k+1})^2]}_{I_4}
    +
    \underbrace{
    \sum_{k=\tau_t}^{t}
    \gamma_k \EE [(r_k - \eta_k)^2]}_{I_5}.
\end{align*}
For $I_1$, following the Abel summation formula, we have
\begin{align*}
    I_1 
    & = 
    \sum_{k=\tau_t}^{t}
    \frac{1}{2 \gamma_k}
    (y_k^2 - y_{k+1}^2) \\ 
    & = 
    \sum_{k=\tau_t}^{t}
    \bigg(\frac{1}{2 \gamma_k} - \frac{1}{2 \gamma_{k-1}}\bigg)
    y_k^2 
    +
    \frac{1}{2 \gamma_{\tau_t - 1}}
    y_{\tau_t}^2
    -
    \frac{1}{2 \gamma_t}
    y_{t+1}^2 \\ 
    & \le 
     \frac{2 U_r^2}{\gamma_t}.
\end{align*}
For $I_2$, from Lemma \ref{lemma:eta-Markovian}, we have
\begin{align*}
    \EE [\Xi(O_t, \eta_t, \btheta_t)]
    & \le 
    4U_r C_J \norm{\btheta_t - \btheta_{t-\tau}}
    + 
    2 U_r |\eta_t - \eta_{t-\tau}|
    \\
    & \qquad 
    + 2 U_{r}^2 |\cA| L \sum_{i=t-\tau}^{t} \EE \norm{\btheta_i - \btheta_{t-\tau}}.
    + 4 U_r^2  m \rho^{\tau - 1}
    \\ 
    & \le 
    4U_r C_J G_{\btheta} \tau \alpha_{t-\tau}
    + 
    4 U_r^2 \tau \gamma_{t-\tau}
    + 2 U_{r}^2 |\cA| L \tau (\tau+1) G_{\btheta} \alpha_{t-\tau}
    + 4 U_r^2  m \rho^{\tau - 1}
    \\
    & \le 
    C_1 \tau^2 \alpha_{t-\tau} + C_2 \tau \gamma_{t-\tau} + C_3  m \rho^{\tau-1}.
\end{align*}
By the choice of $\tau_t$, we have
\begin{align*}
    I_2 & = \sum_{k=\tau_t}^{t} \EE [\Xi(O_k, \eta_k, \btheta_k)]
    \le 
    (C_1
    \tau_t^2+C_3) \sum_{k=\tau_t}^{t} \alpha_{k} 
    +
    C_2
    \tau_t
    \sum_{k=\tau_t}^{t} \gamma_{k}.
\end{align*}
For $I_3$, we have
\begin{align*}
    I_3 & \le 
    \bigg( 
    \sum_{k=\tau_t}^{t} \EE[y_k^2]
    \bigg)^{1/2}
    \bigg(
    C_J^2 G_{\btheta}^2 
    \sum_{k=\tau_t}^{t} 
    \frac{\alpha_k^2}{\gamma_k^2}
    \bigg)^{1/2}, 
\end{align*}
which is because by Lemma \ref{lemma:J-Lipschitz}, $(\eta^*_k - \eta^*_{k+1})$ can be linearly bounded by $\norm{\btheta_k - \btheta_{k+1}} \le G_{\btheta} \cdot \alpha_{k}$.\\
For $I_4$, by the same argument it holds that
\begin{align*}
    I_4 & = 
    \sum_{k=\tau_t}^{t} 
    \frac{1}{\gamma_k}
    \EE[(\eta^*_k - \eta^*_{k+1})^2]
    \\
    & =
    \sum_{k=\tau_t}^{t} 
    \frac{1}{\gamma_k}
    \EE \big[\big(J(\btheta_k) - J(\btheta_{k+1}) \big)^2
    \big]
    \\
    & \le 
    \sum_{k=\tau_t}^{t} 
    \frac{1}{\gamma_k}
    C_J^2 
    \| \btheta_k - \btheta_{k+1} \|^2
    \\
    & \le 
    \sum_{k=\tau_t}^{t} 
    \frac{1}{\gamma_k}
    C_J^2 G_{\btheta}^2
    \alpha_k^2
    \\
    & =
    \cO 
    \bigg(
    \sum_{k=\tau_t}^{t} 
    \frac{\alpha_k^2}{\gamma_k}
    \bigg). 
\end{align*}
For $I_5$, we have 
\begin{align*}
    I_5 
    & =
    \sum_{k=\tau_t}^{t}
    \gamma_k \EE [(r_k - \eta_k)^2]
    \\
    & \le 
    \sum_{k=\tau_t}^{t} 4U_r^2
    \gamma_k 
    \\
    & = 
    \cO 
    \bigg(
    \sum_{k=\tau_t}^{t} 
    {\gamma_k}
    \bigg),
\end{align*}
by bounding the expectation uniformly. 

Now, we set $\gamma_k =1/(1+t)^{\nu}$ and combine all the terms together to get
\begin{align*}
    \sum_{k=\tau_t}^{t} \EE[y_k^2 ]
    & \le 
    2U_r^2 (1+t)^{\nu}
    +
    (C_1
    \tau_t^2+C_3)c_{\alpha} \sum_{k=\tau_t}^{t} (1+k)^{-\sigma} 
    +
    C_2
    \tau_t
    \sum_{k=\tau_t}^{t} (1+k)^{-\nu} \\
    & \qquad +
    C_J G_{\btheta} c_{\alpha}
    \bigg( 
    \sum_{k=\tau_t}^{t} \EE[y_k^2]
    \bigg)^{1/2}
    \bigg(
    \sum_{k=\tau_t}^{t} 
    (1+k)^{-2 (\sigma - \nu) }
    \bigg)^{1/2}
    \\
    & \qquad
    +
    C_J^2 G_{\btheta}^2 c^2_{\alpha}
    \sum_{k=\tau_t}^{t} 
    (1+k)^{\nu - 2 \sigma}
    +
    4U_r^2
    \sum_{k=\tau_t}^{t} 
    (1+k)^{-\nu}
    \\
    & \le 
    2U_r^2 (1+t)^{\nu}
    +
    \big[
    (C_1 \tau^2 + C_3) c_{\alpha} +C_2 \tau_t + C_J^2 G_{\btheta}^2 c^2_{\alpha}
    +
    4U_r^2
    \big] \sum_{k=\tau_t}^{t} 
    (1+k)^{-\nu}
    \\
    & \qquad +
    C_J G_{\btheta} c_{\alpha}
    \bigg( 
    \sum_{k=\tau_t}^{t} \EE[y_k^2]
    \bigg)^{1/2}
    \bigg(
    \sum_{k=\tau_t}^{t} 
    (1+k)^{-2 (\sigma - \nu) }
    \bigg)^{1/2} 
    \\
    & \le 
    2U_r^2 (1+t)^{\nu}
    +
    \big[
    (C_1 \tau^2 + C_3) c_{\alpha} +C_2 \tau_t + C_J^2 G_{\btheta}^2 c^2_{\alpha}
    +
    4U_r^2
    \big] 
    \frac{(1+t-\tau_t)^{1-\nu}}{1-\nu}
    \\
    & \qquad +
    C_J G_{\btheta} c_{\alpha}
    \bigg( 
    \sum_{k=\tau_t}^{t} \EE[y_k^2]
    \bigg)^{1/2}
    \bigg(
    \frac{(1+t-\tau_t)^{1-2(\sigma-\nu)}}{1-2(\sigma-\nu)}
    \bigg)^{1/2}
\end{align*}

By applying the squaring technique already stated in the proof of Theorem \ref{thm:actor}, we have that 
\begin{align}\label{eqn:eta-constant}
    \sum_{k=\tau_t}^{t} \EE[y_k^2 ]
    &
    \le 
    4U_r^2 (1+t)^{\nu}
    +
    2\big[
    (C_1 \tau_t^2 + C_3) c_{\alpha} +C_2 \tau_t + C_J^2 G_{\btheta}^2 c^2_{\alpha}
    +
    4U_r^2
    \big] 
    \frac{(1+t-\tau_t)^{1-\nu}}{1-\nu}
    \notag \\
    & \qquad +
    8C_J^2 G_{\btheta}^2 c_{\alpha}^2
    \frac{(1+t-\tau_t)^{1-2(\sigma-\nu)}}{1-2(\sigma-\nu)}
    \\
    & = 
    \cO(t^{\nu})
    +
    \cO(\log^2 t \cdot t^{1-\nu})
    +
    \cO(t^{1- 2(\sigma - \nu)}). \notag
\end{align}
\end{proof}

\subsection{Proof of Theorem \ref{thm:critic}: Approximating the TD Fixed Point}
\label{subsec:proof-critic}
Now we deal with the critic's parameter $\bomega_t$. The two time-scale analysis with Markovian noise and moving behavior policy can be complicated, so we define some useful notations here that could hopefully clarify the probabilistic dependency.  Note that $J(\btheta) := r(\btheta)$ is the average reward under $\pi_{\btheta}$ and $r(s,a)$ is the one-step reward specified by the state $s$ and action $a$.
\begin{align} \label{def:critic-term}
    O_t 
    : & = 
    (s_t, a_t, s_{t+1}), 
    \notag \\
    g(O, \bomega, \btheta) : &= [r(s,a) - J(\btheta) + ( \bphi(s')-\bphi(s))^{\top}\bomega] \bphi(s), 
    \notag \\
    \Delta g(O, \eta, \btheta)
    : &= [J(\btheta) - \eta] \bphi(s), 
    \notag \\
    \bar g(\bomega, \btheta) 
    : & = 
    \EE_{s \sim \mu_{\btheta}, a \sim \pi_{\btheta}, s' \sim {\cP}}
    \Big[
    \big[
    r(s,a) - J(\btheta) + \big( \bphi(s')-\bphi(s)\big)^{\top}\bomega
    \big] \bphi(s)
    \Big], 
    \notag \\
    \bomega^{*}_{t} 
    : & = 
    \bomega^{*}(\btheta_{t}), 
    \notag \\
    \eta^*_t 
    : & =
    \eta^*(\btheta_t) = J(\btheta_t) 
    \notag \\
    \Lambda(O, \bomega, \btheta) 
    : & = 
    \big \la \bomega - \bomega^*(\btheta), g(O, \bomega, \btheta) - \bar g(\bomega, \btheta) \big \ra, 
    \notag \\
    \zb_t 
    : & = 
    \bomega_t - \bomega^{*}_t 
    \notag \\ 
    y_t
    : & = 
    \eta_t - \eta^*_t.
\end{align}
A bounded lemma is used frequently in this section.
\begin{lemma} \label{lemma:critic-bounded}
Under Assumption \ref{assum:policy-lipschitz-bounded}, for any $\btheta$, $\bomega$, $O=(s,a,s')$ such that $\norm{\bomega} \le R_{\bomega}$,
\begin{align*}
    \big\| g(O,\bomega, \btheta) \big\| & \le U_{\delta} := 2 U_r + 2R_{\bomega}, \\
    \big\|\Delta g(O, \eta, \btheta)\big\|
    & \le 
    2 U_r, \\
    %\big\|\delta \nabla \log \pi_{\btheta}(a|s)\big\| & \le G_{\btheta} := U_\delta \cdot B, \\
    \big |\Lambda(O, \bomega, \btheta)\big | & \le 2R_{\bomega}\cdot 2U_{\delta}  \le 2 U_{\delta}^2.
\end{align*}
\end{lemma}

The following lemma is used to control the bias due to Markovian noise.
\begin{lemma} \label{lemma:critic-Markovian}
Given the definition of $\Lambda(\btheta_t, \bomega_t, O_t)$,
for any $0 \le \tau \le t$, we have
\begin{align*}
    \EE [\Lambda(O_t, \bomega_t, \btheta_t)]
    & \le 
    C_{1}(\tau + 1) \norm{\btheta_{t} - \btheta_{t-\tau}} + C_{2} m \rho^{\tau - 1} + C_{3} \norm{\bomega_t - \bomega_{t-\tau}},
\end{align*}
where $C_{1} = 2 U_{\delta}^2 |\cA| L (1 + \lceil \log_{\rho} m^{-1}\rceil +1/(1- \rho)) + 2 U_{\delta} L_{*}, C_{2} = 2 U_{\delta}^2, C_{3} = 4 U_{\delta}$ are constants.
\end{lemma}

\begin{proof} [Proof of Theorem \ref{thm:critic}]
By the updating rule of $\bomega_t$ in Algorithm \ref{alg:2ts_ac}, unrolling and decomposing the squared error gives
\begin{align*}
    \norm{\zb_{t+1}}^2 
    & = 
    \Big \| 
    \Pi_{R_{\bomega}} \Big(
    \bomega_t + \beta_t 
    \big(
    g(O_t, \bomega_t, \btheta_t)+ \Delta g(O_t, \eta_t, \btheta_t)
    \big) 
    \Big) 
    -
    \bomega^*_{t+1}
    \Big \|
    \\
    & \le 
    \big \| 
    \bomega_t + \beta_t 
    \big(
    g(O_t, \bomega_t, \btheta_t)+ \Delta g(O_t, \eta_t, \btheta_t)
    \big) 
    -
    \bomega^*_{t+1}
    \big \|
    \\
    & = 
    \big\| \zb_{t} + \beta_t (g(O_t, \bomega_t, \btheta_t)+ \Delta g(O_t, \eta_t, \btheta_t)) 
    + (\bomega_{t}^{*} - \bomega_{t+1}^{*}) \big\|^2 
    \\
    & =  
    \norm{\zb_t}^2 
    + 2\beta_t \big\la \zb_t,g(O_t, \bomega_t, \btheta_t) \big\ra 
    + 2\beta_t \big\la \zb_t, \Delta g(O_t, \eta_t, \btheta_t) \big\ra 
    \\
    & \qquad
    + 2 \dotp{\zb_t}{\bomega_{t}^{*} - \bomega_{t+1}^{*}} 
    + \big\| \beta_t (g(O_t, \bomega_t, \btheta_t)+ \Delta g(O_t, \eta_t, \btheta_t)) 
    + (\bomega_{t}^{*} - \bomega_{t+1}^{*}) \big\|^2 
    \\
    & =  
    \norm{\zb_t}^2 
    + 2\beta_t \big\la\zb_t, \bar{g}( \bomega_t, \btheta_t) \big\ra 
    + 2\beta_t \Lambda(O_t, \bomega_t, \btheta_t)
    + 2\beta_t \big\la\zb_t,\Delta g(O_t, \eta_t, \btheta_t)\big\ra 
    \\
    & \qquad
    + 2 \dotp{\zb_t}{\bomega_{t}^{*} - \bomega_{t+1}^{*}} 
    + \big\|\beta_t (g(O_t, \bomega_t, \btheta_t)+ \Delta g(O_t, \eta_t, \btheta_t)) 
    + (\bomega_{t}^{*} - \bomega_{t+1}^{*}) \big\|^2 \\
    & \le  
    \norm{\zb_t}^2 
    + 2\beta_t \big \la \zb_t, \bar{g}( \bomega_t, \btheta_t) \big \ra 
    + 2\beta_t \Lambda(O_t, \bomega_t, \btheta_t)
    + 2\beta_t \big \la \zb_t, \Delta g(O_t, \eta_t, \btheta_t) \big \ra 
    \\ 
    & \qquad+ 2 \dotp{\zb_t}{\bomega_{t}^{*} - \bomega_{t+1}^{*}} 
    + 2\beta_{t}^{2} 
    \big\|g(O_t, \bomega_t, \btheta_t)+ \Delta g(O_t, \eta_t, \btheta_t) \big\|^2
    + 2\norm{\bomega_{t}^{*} - \bomega_{t+1}^{*}}^2\\
    &\le  
    \norm{\zb_t}^2 
    + 2\beta_t \big\la \zb_t, \bar{g}( \bomega_t, \btheta_t) \big\ra
    + 2\beta_t \Lambda(O_t, \bomega_t, \btheta_t)
    + 2\beta_t \big\la \zb_t, \Delta g(O_t, \eta_t, \btheta_t) \big\ra 
    \\ 
    &\qquad + 2 \dotp{\zb_t}{\bomega_{t}^{*} - \bomega_{t+1}^{*}} 
    + 2 U_{\delta}^2  \beta_t^2 
    + 2\norm{\bomega_{t}^{*} - \bomega_{t+1}^{*}}^2,
\end{align*}
where the first inequality holds because $\bomega_{t+1}^*$ is assumed to be within the $R_{\bomega}$-ball so the projection only reduces the distance; the second one is due to $\norm{\xb + \yb}^2 \le 2 \norm{\xb}^2 + 2 \norm{\yb}^2$ and the third one is due to $\norm{g(O_t, \bomega_t, \btheta_t)+ \Delta g(O_t, \eta_t, \btheta_t)} \le U_{\delta}$. 

First, note that due to Assumption \ref{assum:negative-definite}, we have 
\begin{align*}
    \big\la \zb_t, \bar{g}(\bomega_t, \btheta_t) \big\ra
    & = 
    \big\la \zb_t, \bar{g}(\bomega_t, \btheta_t) - \bar{g}(\bomega^*_t, \btheta_t) \big\ra \\
    & = 
    \Big\la \zb_t, \EE \big[ \big( \bphi(s')-\bphi(s) \big)^{\top}(\bomega_t - \bomega^*_t)  \bphi(s) \big] \Big\ra \\
    & = 
    \zb_t^{\top}
    \EE \big[
    \bphi(s)
    \big( \bphi(s')-\bphi(s) \big)^{\top}
    \big]
    \zb_t \\
    & =
    \zb_t^{\top}
    \Ab
    \zb_t \\
    & \le 
    - \lambda \norm{\zb_t}^2,
\end{align*}
where the first equation is due to the fact that $\bar{g}( \bomega^*, \btheta) = 0$ \citep{sutton2018reinforcement}. Taking expectation up to $s_{t+1}$, we have
\begin{align*}
    \EE \norm{\zb_{t+1}}^2  
    & \le 
    \EE \norm{\zb_t}^2 
    + 2\beta_t \EE \big\la \zb_t, \bar{g}( \bomega_t, \btheta_t) \big\ra 
    + 2\beta_t \EE \Lambda(O_t, \bomega_t, \btheta_t)
    + 2\beta_t \EE \big\la \zb_t, \Delta g(O_t, \eta_t, \btheta_t) \big\ra 
    \\ 
    &\qquad + 2
    \EE \dotp{\zb_t}{\bomega_{t}^{*} - \bomega_{t+1}^{*}} 
    + 2 U_{\delta}^2  \beta_t^2 
    + 2 \EE \norm{\bomega_{t}^{*} - \bomega_{t+1}^{*}}^2
    \\
    & \le 
    (1 - 2 \lambda \beta_t) \EE \norm{\zb_t}^2 
    + 2\beta_t \EE \Lambda(O_t, \bomega_t, \btheta_t)
    + 2\beta_t \EE \big\la \zb_t, \Delta g(O_t, \eta_t, \btheta_t) \big\ra 
    \\ & \qquad + 
    2 \EE \dotp{\zb_t}{\bomega_{t}^{*} - \bomega_{t+1}^{*}} 
    + 2 U_{\delta}^2  \beta_t^2 
    + 2 \EE \norm{\bomega_{t}^{*} - \bomega_{t+1}^{*}}^2.
\end{align*}
Based on the result above, we can further rewrite it as:
\begin{align*}
    \EE \norm{\zb_{t+1}}^2  
    & \le 
    (1 - 2 \lambda \beta_t) \EE \norm{\zb_t}^2 
    + 2\beta_t \EE \Lambda(O_t, \bomega_t, \btheta_t)
    + 2\beta_t \EE \norm{\zb_t} \cdot |y_t|
    \\ 
    & \qquad + 
    2 L_{*} \EE \norm{\zb_t} \cdot \norm{\btheta_{t} - \btheta_{t+1}} 
    + 2 U_{\delta}^2  \beta_t^2 
    + 2 L_{*}^2 \EE \norm{\btheta_{t} - \btheta_{t+1}}^2 \\
    &  \le 
    (1 - 2 \lambda \beta_t) \EE \norm{\zb_t}^2 
    + 2\beta_t \EE \Lambda(O_t, \bomega_t, \btheta_t)
    + 2\beta_t \EE \norm{\zb_t} \cdot |y_t|
    \\ 
    & \qquad
    + 2 L_{*} G_{\btheta} \alpha_t \EE \norm{\zb_t}
    + 2 U_{\delta}^2  \beta_t^2 
    + 2 L_{*}^2 G_{\btheta}^2 \alpha_{t}^2
    \\
    &  \le 
    (1 - 2 \lambda \beta_t) \EE \norm{\zb_t}^2 
    + 2\beta_t \EE \Lambda(O_t, \bomega_t, \btheta_t)
    + 2\beta_t \EE \norm{\zb_t} \cdot |y_t|
    \\ 
    & \qquad +
    2  L_{*} G_{\btheta} \alpha_t \EE \norm{\zb_t}
    + \bigg(2 U_{\delta}^2 + 2 L_{*}^2 G_{\btheta}^2 \Big(\max_t \frac{\alpha_t}{\beta_t} \Big)^2 \bigg)\beta_{t}^2 \\
    & = 
    (1 - 2 \lambda \beta_t) \EE \norm{\zb_t}^2 
    + 2\beta_t \EE \Lambda(O_t, \bomega_t, \btheta_t)
    + 2\beta_t \EE \norm{\zb_t} \cdot |y_t|
    + 2  L_{*} G_{\btheta} \alpha_t \EE \norm{\zb_t}
    + C_{q}\beta_{t}^2, 
\end{align*}
where we denote the constant coefficient before the quadratic stepsize $\beta_{t}^2$ as $C_q$ at the last step. The first inequality is due to Proposition \ref{prop:optimal-lipschitz} and Cauchy-Schwartz inequality. The second inequality is due to the update of $\btheta_t$ is bounded by $G_{\btheta} \alpha_t$. The third inequality is from employing the fact that $\sigma > \nu$ so $\alpha_t / \beta_t$ is bounded. Rearranging the inequality yields
\begin{align*}
    2 \lambda \EE \norm{\zb_t}^2
    & \le 
    \frac{1}{\beta_t} \big(\EE \norm{\zb_t}^2 - \EE \norm{\zb_{t+1}}^2  \big)
    + 
    2  \EE \Lambda(O_t, \bomega_t, \btheta_t)
    + 
    \EE \norm{\zb_t} \cdot |y_t|
    + 
    2  L_{*} G_{\btheta} \frac{\alpha_t}{\beta_t} \EE \norm{\zb_t}
    + 
    C_{q}\beta_{t} \\
    & \le 
    \frac{1}{\beta_t} \big(\EE \norm{\zb_t}^2 - \EE \norm{\zb_{t+1}}^2  \big)
    + 
    2  \EE \Lambda(O_t, \bomega_t, \btheta_t)
    + 
    \sqrt{\EE y_t^2} \cdot \sqrt{\EE \norm{\zb_t}^2}
    \\
    & \qquad 
    + 
    2  L_{*} G_{\btheta} \frac{\alpha_t}{\beta_t} \sqrt{\EE \norm{\zb_t}^2}
    + 
    C_{q}\beta_{t},
\end{align*}
where the second inequality is due to the concavity of square root function. Telescoping from $\tau_{t}$ to $t$ gives: 
\begin{align} \label{eqn:critic-bone}
    2 \lambda \sum_{k=\tau_{t}}^{t} \EE \norm{\zb_k}^2
    & \le 
    \underbrace{\sum_{k=\tau_{t}}^{t} \frac{1}{\beta_k} \big(\EE \norm{\zb_k}^2 - \EE \norm{\zb_{k+1}}^2  \big)}_{I_1}
    + 
    2 \underbrace{\sum_{k=\tau_{t}}^{t}  \EE \Lambda(\btheta_k, \bomega_k, O_k)}_{I_2}
    \notag \\
    & \qquad + 
    2 L_{*} G_{\btheta} 
    \underbrace{\sum_{k=\tau_{t}}^{t}   \frac{\alpha_k}{\beta_k} \sqrt{\EE \norm{\zb_k}^2}}_{I_3}
    +
    \underbrace{\sum_{k=\tau_{t}}^{t}   \sqrt{\EE y_k^2} \cdot \sqrt{\EE \norm{\zb_k}^2}}_{I_4}
    + 
    C_{q} \underbrace{\sum_{k=\tau_{t}}^{t} \beta_{k}}_{I_5} .
\end{align}
From \eqref{eqn:critic-bone}, we can see the proof of the critic again shares the same spirit with the proof of Theorem \ref{thm:actor}. For term $I_1$, we have 
\begin{align*}
    I_1 
    & := 
    \sum_{k=\tau_{t}}^{t} \frac{1}{\beta_k} (\EE \norm{\zb_k}^2 - \EE \norm{\zb_{k+1}}^2  )
    \\
    & = 
    \sum_{k=\tau_{t}}^{t} \bigg(\frac{1}{\beta_k} 
    -
    \frac{1}{\beta_{k-1}} 
    \bigg)
    \EE \norm{\zb_k}^2 
    +
    \frac{1}{\beta_{\tau_{t}-1}} 
    \EE \norm{\zb_{\tau_{t}}}^2 
    -
    \frac{1}{\beta_{t}} 
    \EE \norm{\zb_{t+1}}^2 
    \\ 
    & \le 
    \sum_{k=\tau_{t}}^{t} \bigg(\frac{1}{\beta_k} 
    -
    \frac{1}{\beta_{k-1}} 
    \bigg)
    \EE \norm{\zb_k}^2 
    +
    \frac{1}{\beta_{\tau_{t}-1}} 
    \EE \norm{\zb_{\tau_{t}}}^2 
    \\
    & \le 
    4 R_{\bomega}^2
    \bigg(
    \sum_{k=\tau_{t}}^{t} \bigg(\frac{1}{\beta_k} 
    -
    \frac{1}{\beta_{k-1}} 
    \bigg)
    +
    \frac{1}{\beta_{\tau_{t}-1}} 
    \bigg)
    \\
    & = 
    4R_{\bomega}^2 \frac{1}{\beta_t}
    \\
    & = 
    4R_{\bomega}^2 (1+t)^{\nu}
    = \cO(t^{\nu}),
\end{align*}
where the first inequality is due to discarding the last term, and the second inequality is due to $\EE \norm{\zb_k}^2 \le (R_{\bomega} + R_{\bomega})^2$.\\
For term $I_2$, note that due to Lemma \ref{lemma:critic-Markovian}, we actually have
\begin{align*}
    \Lambda( O_k, \bomega_k, \btheta_k)
    & \le 
    C_{1}(\tau_{t} + 1) \norm{\btheta_{k} - \btheta_{k-\tau_{t}}} + C_{2} m \rho^{\tau_{t} - 1} + C_{3} \norm{\bomega_k - \bomega_{k-\tau_{t}}}
    \\
    & \le 
    C_{1}(\tau_{t} + 1) \sum_{i=k-\tau_{t}}^{k-1} G_{\btheta} \alpha_i
    +
    C_2 m \rho^{\tau_{t} - 1}
    +
    C_3 \sum_{i=k-\tau_{t}}^{k-1}
    U_{\delta} \beta_i
    \\
    & \le 
    C_1 G_{\btheta} (\tau_{t} + 1)^2 \alpha_{k-\tau_{t}} 
    +
    C_2 \alpha_t
    +
    C_3 U_{\delta} \tau_{t} \beta_k,
\end{align*}
and the summation is
\begin{align*}
    I_2 
    & :=
    \sum_{k=\tau_{t}}^{t}  \EE \Lambda( O_k, \bomega_k, \btheta_k)
    \\
    & \le 
    C_1 G_{\btheta} (\tau_{t} + 1)^2 \sum_{k=\tau_{t}}^{t}\alpha_{k-\tau_{t}} 
    +
    C_2 \sum_{k=\tau_{t}}^{t}\alpha_t
    +
    C_3 U_{\delta} \tau_{t} \sum_{k=\tau_{t}}^{t} \beta_k 
    \\ 
    & \le 
    C_1 G_{\btheta} (\tau_{t} + 1)^2 \sum_{k=0}^{t-\tau_{t}}\alpha_{k} 
    +
    C_2 (t-\tau_{t}+1) \alpha_t
    +
    C_3 U_{\delta} \tau_{t} \sum_{k=0}^{t-\tau_{t}} \beta_k 
    \\
    & \le 
    C_1 G_{\btheta} (\tau_{t} + 1)^2 
    c_{\alpha} \frac{(1+t-\tau_t)^{1-\sigma}}{1-\sigma}
    +
    C_2 (t-\tau_{t}+1) c_\alpha (1+t)^{-\sigma}
    +
    C_3 U_{\delta} \tau_{t} \frac{(1+t-\tau_t)^{1-\nu}}{1-\nu}
    \\
    & \le 
    \bigg[
    \frac{C_1 G_{\btheta} (\tau_{t} + 1)^2 
    c_{\alpha} }{1-\sigma}
    +
    C_2 c_\alpha 
    +
    \frac{C_3 U_{\delta} \tau_{t} }{1-\nu}
    \bigg]
    (1+t)^{1-\nu}
    \\ 
    & = 
    \cO \big((\log t)^2 t^{1-\nu} \big),
\end{align*}
where the second inequality is due to the monotonicity of $\alpha_k$ and $\beta_k$. The $\cO(\cdot)$ comes from that $\tau = \cO(\log t)$ and $\sum k^{-\nu} = \cO(t^{1-\nu})$.

\noindent For term $I_3$ and $I_4$,
we will instead show it can be bounded in a different form. Using Cauchy-Schwartz inequality we have
\begin{align*}
    I_3 &:= 
    \sum_{k=\tau_{t}}^{t}   \frac{\alpha_k}{\beta_k} \sqrt{\EE \norm{\zb_k}^2}
     \le 
    \bigg(
    \sum_{k=\tau_{t}}^{t}   \frac{\alpha_k^2}{\beta_k^2}
    \bigg)^{\frac{1}{2}}
    \bigg(
    \sum_{k=\tau_{t}}^{t}  
    \EE \norm{\zb_k}^2
    \bigg)^{\frac{1}{2}}
     \le 
    \bigg(
    \sum_{k=0}^{t-\tau_{t}}   \frac{\alpha_k^2}{\beta_k^2}
    \bigg)^{\frac{1}{2}}
    \bigg(
    \sum_{k=\tau_{t}}^{t}  
    \EE \norm{\zb_k}^2
    \bigg)^{\frac{1}{2}}, \\
    I_4 &: = 
    \sum_{k=\tau_{t}}^{t}   \sqrt{\EE y_k^2} \cdot \sqrt{\EE \norm{\zb_k}^2}
     \le 
    \bigg(
    \sum_{k=\tau_{t}}^{t}   \EE y_k^2
    \bigg)^{\frac{1}{2}}
    \bigg(
    \sum_{k=\tau_{t}}^{t}  
    \EE \norm{\zb_k}^2
    \bigg)^{\frac{1}{2}}
     \le 
    \bigg(
    \sum_{k=0}^{t-\tau_{t}}   \EE y_k^2
    \bigg)^{\frac{1}{2}}
    \bigg(
    \sum_{k=\tau_{t}}^{t}  
    \EE \norm{\zb_k}^2
    \bigg)^{\frac{1}{2}}.
\end{align*}
For term $I_5$, simply bound it as $\sum_{k=0}^{t-\tau_{t}} \beta_k \le (1+t)^{1-\nu}/(1-\nu)$.

\noindent Collecting the upper bounds of the above five terms, and writing them using $\cO(\cdot)$ notation give
\begin{align} \label{eqn:critic-constant}
    2 \lambda \sum_{k=\tau_{t}}^{t} \EE \norm{\zb_k}^2
    & \le
    4R_{\bomega}^2 (1+t)^{\nu}
    +
    2\bigg[
    \frac{C_1 G_{\btheta} (\tau_{t} + 1)^2 
    c_{\alpha} }{1-\sigma}
    +
    C_2 c_\alpha 
    +
    \frac{C_3 U_{\delta} \tau_{t} + C_q }{1-\nu}
    \bigg]
    (1+t)^{1-\nu}
    \notag \\
    & \qquad +
    2 L_* G_{\btheta}
    \bigg(
    \sum_{k=0}^{t-\tau_{t}}   \frac{\alpha_k^2}{\beta_k^2}
    \bigg)^{\frac{1}{2}}
    \big(
    \sum_{k=\tau_{t}}^{t}  
    \EE \norm{\zb_k}^2
    \bigg)^{\frac{1}{2}}
    \notag \\
    &\qquad
    +
    \bigg(
    \sum_{k=0}^{t-\tau_{t}}   \EE y_k^2
    \bigg)^{\frac{1}{2}}
    \bigg(
    \sum_{k=\tau_{t}}^{t}  
    \EE \norm{\zb_k}^2
    \bigg)^{\frac{1}{2}}.
\end{align}
Now, we first divide both sides by $(1+t-\tau_{t})$, and denote 
\begin{align*}
    Z(t) :&= \frac{1}{1+t-\tau_{t}} \sum_{k=\tau_{t}}^{t} \EE \norm{\zb_k}^2, \\
    F(t) : &= \frac{1}{1+t-\tau_{t}} \sum_{k=0}^{t-\tau_{t}}   \frac{\alpha_k^2}{\beta_k^2} 
    \le 
    \frac{t^{-2(\sigma - \nu)}}{ 1-2(\sigma-\nu)} 
    = \cO(t^{-2(\sigma-\nu)}),
    \\
    G(t) :&=
    \frac{1}{1+t-\tau_{t}}
    \sum_{k=0}^{t-\tau_{t}} \EE[y_k^2] = \cO(t^{\nu-1}) + \cO(\log t \cdot t^{-\nu}) + \cO(t^{-2(\sigma - \nu)}),
\end{align*}
 and the rest as $A(t) = \cO(t^{\nu}) + \cO(t^{1-\nu})$. $G(t)$'s constants appear at \eqref{eqn:eta-constant} in exact form.
 
 This simplification leads to
\begin{align*}
    2 \lambda
    \Big(
    \sqrt{Z(t)} 
    - 
    \frac{L_* G_{\btheta}}{2 \lambda}\cdot 
    \sqrt{F(t)}
    -
    \frac{1}{4 \lambda}
    \sqrt{G(t)}
    \Big)^2
    & \le 
    A(t)
    +
    2 \lambda
    \bigg(
    \frac{L_{*} G_{\btheta}}{2 \lambda}
    \sqrt{F(t)}
    +
    \frac{1}{4 \lambda}
    \sqrt{G(t)}
    \bigg)^2
    ,
\end{align*}
which further gives
\begin{align*}
    Z(t) 
    & \le 
    A(t) / \lambda 
    +
    16 F(t)
    +
    16 G(t).
\end{align*}
This is again a similar reasoning as in the end of the proof of Theorem \ref{thm:actor}. We actually show that
\begin{align*}
    \frac{1}{1+t-\tau_{t}} \sum_{k=\tau_{t}}^{t} \EE \norm{\bomega_k - \bomega^*_k}^2
    & = 
    \cO
    \bigg(\frac{1}{t^{1-\nu}}
    \bigg)
    +
    \cO
    \bigg(\frac{\log t}{t^{\nu}}
    \bigg)
    +
    \cO
    \bigg(\frac{1}{t^{2(\sigma - \nu)}}
    \bigg).
\end{align*}
This completes the proof. To obtain the exact constant, please refer to \eqref{eqn:eta-constant} and \eqref{eqn:critic-constant}.
\end{proof}

\subsection{Proof of Corollary \ref{col:sample-complexity}} \label{subsec:proof-samp-comp}
\begin{proof}[Proof of Corollary \ref{col:sample-complexity}]
By Theorem \ref{thm:critic}, we have 
\begin{align*}
    \frac{1}{1+t-\tau_{t}} \sum_{k=\tau_{t}}^{t} \EE \norm{\bomega_k - \bomega^*_k}^2 =
    \cO\bigg(
    \frac{1}{t^{1-\nu}}
    \bigg)
    +
    \cO\bigg(
    \frac{\log t}{t^{\nu}}
    \bigg)
    +
    \cO\bigg(
    \frac{1}{t^{2(\sigma-\nu)}}
    \bigg).
\end{align*}
By Lemma \ref{lemma:equi-asym}, $\cE(t)$ in Theorem \ref{thm:actor} is of the equivalent order:
\begin{align*}
    \cE_1(t) & = \frac{1}{t} \sum_{k=1}^{t} \EE \norm{\bomega_k - \bomega^*_k}^2 \\
    & = 
    \cO \bigg(
    \frac{1}{1+t-\tau_{t}} \sum_{k=\tau_{t}}^{t} \EE \norm{\bomega_k - \bomega^*_k}^2
    \bigg)
    +
    \cO \bigg( 
    \frac{\log t}{t}
    \bigg) \\ 
    & = 
    \cO\bigg(
    \frac{1}{t^{1-\nu}}
    \bigg)
    +
    \cO\bigg(
    \frac{\log t}{t^{\nu}}
    \bigg)
    +
    \cO\bigg(
    \frac{1}{t^{2(\sigma-\nu)}}
    \bigg)
    +
    \cO \bigg( 
    \frac{\log t}{t}
    \bigg) \\ 
    & = 
    \cO\bigg(
    \frac{1}{t^{1-\nu}}
    \bigg)
    +
    \cO\bigg(
    \frac{\log t}{t^{\nu}}
    \bigg)
    +
    \cO\bigg(
    \frac{1}{t^{2(\sigma-\nu)}}
    \bigg).
\end{align*}
The same reasoning also applies to
\begin{align*}
    \cE_2(t) &= \frac{1}{t} \sum_{k=1}^t \EE (\eta_k - r(\btheta_k))^2 
    \\
    & = 
    \cO\bigg(
    \frac{1}{t^{1-\nu}}
    \bigg)
    +
    \cO\bigg(
    \frac{\log t}{t^{\nu}}
    \bigg)
    +
    \cO\bigg(
    \frac{1}{t^{2(\sigma-\nu)}}
    \bigg).
\end{align*}
Plugging the above results into Theorem \ref{thm:actor}, and optimizing over the choice of $\sigma$ and $\nu$ (which gives $\sigma = 3/5$ and $\nu = 2/5$), we have 
\begin{align*}
    \min_{0 \le k \le t}\EE \norm{\nabla J(\btheta_{k})}^2 
    &= 
    \cO\bigg(
    \frac{1}{t^{1-\sigma}}
    \bigg)
    +
    \cO\bigg(
    \frac{\log^2 t}{t^{\sigma}}
    \bigg)
    +
    \cO\bigg(
    \frac{1}{t^{1-\nu}}
    \bigg)
    +
    \cO\bigg(
    \frac{\log t}{t^{\nu}}
    \bigg)
    +
    \cO\bigg(
    \frac{1}{t^{2(\sigma-\nu)}}
    \bigg) 
    +
    \cO
    (\epsilon_{\text{app}})\\ 
    & = 
    \cO\bigg(
    \frac{1}{t^{1-\sigma}}
    \bigg)
    +
    \cO\bigg(
    \frac{\log t}{t^{\nu}}
    \bigg)
    +
    \cO\bigg(
    \frac{1}{t^{2(\sigma-\nu)}}
    \bigg) 
    +
    \cO
    (\epsilon_{\text{app}})\\ 
    & = 
    \cO
    \bigg(
    \frac{\log t}{t^{2/5}}
    \bigg)
    +
    \cO
    (\epsilon_{\text{app}}).
\end{align*}
Therefore, in order to obtain an $\epsilon$-approximate(ignoring the approximation error) stationary point of $J$, namely, 
\begin{align*}
    \min_{0 \le k \le T} \EE \big\|\nabla J(\btheta_{k})\big\|^2 
    &= 
    \cO\bigg(\frac{\log T}{T^{2/5}}\bigg)
    +
    \cO
    (\epsilon_{\text{app}})\le
    \cO
    (\epsilon_{\text{app}}) +\epsilon,
\end{align*}
we need %to set $T = \epsilon^{-2.5} (3\log \epsilon^{-1})^{2.5}$ is enough so we can write it as 
to set $T = \tilde{\cO}(\epsilon^{-2.5})$.
\end{proof}

\section{Proof of Technical Lemmas}
\subsection{Proof of Lemma \ref{lemma:J-smooth}}
\begin{proof}[Proof of Lemma \ref{lemma:J-smooth}]
The first inequality comes from Lemma 3.2 in \citet{zhang2019global}.

The second inequality is well known as a partial result of $[-L, L]$-smoothness of non-convex functions.
\end{proof}

\subsection{Proof of Lemma \ref{lemma:critic-bias}}
\begin{proof}[Proof of Lemma \ref{lemma:critic-bias}]
Applying the definition of $\Delta h()$ and Cauchy-Schwartz inequality immediately yields the result.
\end{proof}

\subsection{Proof of Lemma \ref{lemma:actor-markovian}} \label{subsec:proof-of-lemma-actor-markovian}
The proof of Lemma \ref{lemma:actor-markovian} will be built on the following supporting lemmas. 
{\begin{lemma} \label{lemma:Gamma-term1}
For any $t \ge 0$, 
\begin{align*}
    \big| \Gamma(O_t, \btheta_t) - \Gamma(O_t, \btheta_{t-\tau}) \big|
    \le 
    (2 U_{\delta} B L_J 
    +
    3G_{\btheta} L_h) \norm{\btheta_{t} - \btheta_{t-\tau}},
\end{align*}
where $L_h = U_{\delta} L_{l}
    +
    (2+ 2 \lambda^{-2} + 3 \lambda^{-1}) B U_r|\cA|L \big(1 + \lceil \log_{\rho}m^{-1} \rceil + 1/(1-\rho)\big)$.
\end{lemma}}
\begin{lemma} \label{lemma:Gamma-term2}
For any $t \ge 0$,
\begin{align*}
    \big|\EE[\Gamma(O_t, \btheta_{t-\tau}) - \Gamma(\tilde{O}_t, \btheta_{t-\tau})] \big|
    & \le
    2U_{\delta} B G_{\btheta} |\cA| L \sum_{i=t-\tau}^{t} \norm{\btheta_i - \btheta_{t-\tau}}.
\end{align*}
\end{lemma}
\begin{lemma} \label{lemma:Gamma-term3}
For any $t \ge 0$,
\begin{align*}
    \big|\EE [\Gamma(\tilde{O}_t, \btheta_{t-\tau}) - \Gamma(O'_t, \btheta_{t-\tau})] \big|
    & \le 
    4U_{\delta}B G_{\btheta} m \rho^{\tau - 1}.
\end{align*}
\end{lemma}
\begin{proof}[Proof of Lemma \ref{lemma:actor-markovian}]
First note that
\begin{align*}
   \delta & =  \big| r(s,a) - J(\btheta) + \bphi^{\top}(s') \bomega - \bphi^{\top}(s) \bomega  \big| \\
    & \le
    \big|r(s,a) \big| + \big|J(\btheta) \big|
    + \big|\bphi^{\top}(s') \bomega \big|
    + \big|\bphi^{\top}(s) \bomega \big| \\
    & =  2U_{r} + 2 R_{\bomega}\\
    & =:U_{\delta},
\end{align*}
which immediately implies
\begin{align}
    \big \|
    \delta \nabla \log \pi_{\btheta}(a|s)
    \big \|
    \le 
    |\delta|
    \cdot 
    \big \|
    \nabla \log \pi_{\btheta}(a|s)
    \big \|
    \le 
    U_{\delta} \cdot B
    =: G_{\btheta}, \label{eqn:G-theta}
\end{align}
where the last inequality is due to Assumption \ref{assum:policy-lipschitz-bounded}. We decompose the Markovian bias as
\begin{align*}
    \EE[\Gamma(O_t, \btheta_t)] 
    &=
    \EE[\Gamma(O_t, \btheta_t) - \Gamma(O_t, \btheta_{t-\tau})] +
    \EE[\Gamma(O_t, \btheta_{t-\tau}) - \Gamma(\tilde{O}_t, \btheta_{t-\tau})] \\
    & \qquad+
    \EE[\Gamma(\tilde{O}_t, \btheta_{t-\tau}) - \Gamma(O'_t, \btheta_{t-\tau})] +
    \EE[ \Gamma(O'_t, \btheta_{t-\tau})],
\end{align*}
where $\tilde{O}_t$ is from the auxiliary Markovian chain and $O'_t$ is from the stationary distribution which actually satisfy $\EE[\Gamma(O'_t, \btheta_{t-\tau}) ]= 0$. By collecting the corresponding bounds from Lemmas \ref{lemma:Gamma-term1}, \ref{lemma:Gamma-term2} and~\ref{lemma:Gamma-term3}, we have that 
\begin{align*}
    \EE[\Gamma(O_t, \btheta_t)]
    & \ge 
    - (2 U_{\delta} B L_J 
    +
    3G_{\btheta} L_h)  \EE \norm{\btheta_{t} - \btheta_{t-\tau}}
    -
    2U_{\delta} B G_{\btheta} |\cA| L \sum_{i=t-\tau}^{t} \EE \norm{\btheta_i - \btheta_{t-\tau}}\\
    &\qquad -4U_{\delta}B G_{\btheta} m \rho^{\tau - 1} \\
    & \ge 
    - (2 U_{\delta} B L_J 
    +
    3G_{\btheta} L_h)  
    \sum_{i=t-\tau+1}^t \EE \norm{\btheta_i - \btheta_{i-1}}
    \\
    &\qquad 
    -
    2U_{\delta} B G_{\btheta} |\cA| L 
    \sum_{i=t-\tau+1}^t \sum_{j=t-\tau+1}^i \EE \norm{\btheta_j - \btheta_{j-1}}
    -
    4U_{\delta}B G_{\btheta} m \rho^{\tau - 1}
    \\
    & \ge 
    - (2 U_{\delta} B L_J 
    +
    3G_{\btheta} L_h)
    \sum_{i=t-\tau+1}^t \EE \norm{\btheta_i - \btheta_{i-1}}
    \\
    &\qquad 
    -
    2U_{\delta} B G_{\btheta} |\cA| L 
    \tau \sum_{j=t-\tau+1}^t \EE \norm{\btheta_j - \btheta_{j-1}}
    -
    4U_{\delta}B G_{\btheta} m \rho^{\tau - 1}
    \\
    & \ge 
    -G_{\btheta} \bigg(D_{1} (\tau + 1) \sum_{k=t-\tau+1}^t \EE \norm{\btheta_k - \btheta_{k-1}}
    +D_{2} m \rho^{\tau - 1}\bigg),
\end{align*}
where $D_1:= \max \{ 2 L_J +3 L_h, 2U_{\delta} B  |\cA| L \}$ and $D_2:=4U_{\delta}B$, which completes the proof.
\end{proof}

\subsection{Proof of Lemma \ref{lemma:J-Lipschitz}}
\begin{proof}[Proof of Lemma \ref{lemma:J-Lipschitz}]
By definition, we have
\begin{align*}
    J(\btheta_1) - J(\btheta_2)
    & = 
    \EE [r(s^{(1)}, a^{(1)}) - r(s^{(2)}, a^{(2)})],
\end{align*}
where $s^{(i)} \sim \mu_{\btheta_i}, a^{(i)} \sim \pi_{\btheta_i}$.
Therefore, it holds that
\begin{align*}
    J(\btheta_1) - J(\btheta_2)
    & = 
    \EE [r(s^{(1)}, a^{(1)}) - r(s^{(2)}, a^{(2)})] \\ 
    & \le 
    2 U_r d_{TV}(\mu_{\btheta_1} \otimes \pi_{\btheta_1}, \mu_{\btheta_2}\otimes \pi_{\btheta_2}) 
    \\
    & \le 
    2 U_r |\cA|L \bigg(1 + \lceil \log_{\rho}m^{-1} \rceil + \frac{1}{1-\rho}\bigg) \norm{\btheta_1 - \btheta_2} 
    \\
    & =
    C_J \norm{\btheta_1 - \btheta_2} .
\end{align*}

\end{proof}

\subsection{Proof of Lemma \ref{lemma:eta-Markovian}}
The proof of this lemma depends on several auxiliary lemmas as follows.
\begin{lemma} \label{lemma:eta-term1}
For any $\btheta_1, \btheta_2, eta, O=(s,a,s')$, we have
\begin{align*}
    \big|\Xi(O, \eta, \btheta_1) - \Xi(O, \eta, \btheta_2) \big|
    & \le 
    4 U_r  C_J \norm{\btheta_1 - \btheta_2}.
\end{align*}
\end{lemma}
\begin{lemma} \label{lemma:eta-term2}
For any $\btheta, \eta_1, \eta_2, O$, we have
\begin{align*}
    \big|\Xi(O, \eta_1, \btheta) - \Xi(O, \eta_2, \btheta) \big| 
    & \le
    2 U_r |\eta_1 - \eta_2|.
\end{align*}
\end{lemma}
\begin{lemma} \label{lemma:eta-term3}
Consider original tuples $O_t = (s_t,a_t,s_{t+1})$ and the auxiliary tuples $\tilde{O}_t = (\tilde{s}_t, \tilde{a}_t, \tilde{s}_{t+1})$.
Conditioned on $s_{t-\tau+1}$ and $\btheta_{t - \tau}$, we have
\begin{align*}
    \big|\EE[\Xi(O_t, \eta_{t-\tau}, \btheta_{t-\tau})
    -
    \Xi(\tilde{O}_t, \eta_{t-\tau}, \btheta_{t-\tau})] \big|
    \le
    2 U_{r}^2 |\cA| L \sum_{i=t-\tau}^{t} \EE \norm{\btheta_i - \btheta_{t-\tau}}.
\end{align*}
\end{lemma}
\begin{lemma} \label{lemma:eta-term4}
Conditioned on $s_{t-\tau+1}$ and $\btheta_{t - \tau}$, we have
\begin{align*}
    \EE[\Xi(\tilde{O}_t, \eta_{t-\tau}, \btheta_{t-\tau}) ]
    & \le 4 U_r^2  m \rho^{\tau - 1}.
\end{align*}
\end{lemma}
\begin{proof}
By the Lemma \ref{lemma:eta-term1}, \ref{lemma:eta-term2}, \ref{lemma:eta-term3} and \ref{lemma:eta-term4}, we can collect the corresponding term and  get the bound
\begin{align*}
    \EE [\Xi(O_t, \eta_t, \btheta_t)] 
    & = 
    \EE [\Xi(O_t, \eta_t, \btheta_t) - \Xi(O_t, \eta_t, \btheta_{t-\tau})]
    +
    \EE [\Xi(O_t, \eta_t, \btheta_{t-\tau}) - \Xi(O_t, \eta_{t-\tau}, \btheta_{t-\tau})]
    \\ 
    & \qquad + 
    \EE [\Xi(O_t, \eta_{t-\tau}, \btheta_{t-\tau}) - \Xi(\tilde{O}_t, \eta_{t-\tau}, \btheta_{t-\tau})]
    +
    \EE [\Xi(\tilde{O}_t, \eta_{t-\tau}, \btheta_{t-\tau})] 
    \\ 
    & \le 
    4 U_r  C_J \norm{\btheta_t - \btheta_{t-\tau}}
    +
    2 U_r |\eta_t - \eta_{t-\tau}|
    \\
    & \qquad 
    +
    2 U_{r}^2 |\cA| L \sum_{i=t-\tau}^{t} \EE \norm{\btheta_i - \btheta_{t-\tau}}
    +
    4 U_r^2  m \rho^{\tau - 1}.
\end{align*}
\end{proof}

\subsection{Proof of Lemma \ref{lemma:critic-bounded}}

\begin{proof}[Proof of Lemma \ref{lemma:critic-bounded}]
For the first inequality, apply the property of norm and the Cauchy-Schwartz inequality:
\begin{align*}
    \big\| g(O,\bomega, \btheta) \big\|
    &=
    \big\| (r(s,a) - J(\btheta) + \bphi^{\top}(s') \bomega - \bphi^{\top}(s) \bomega ) \bphi(s) \big\| \\
    & \le
    \big|r(s,a) \big| + \big\|J(\btheta) \big\|
    + \big|\bphi^{\top}(s') \bomega \big|\cdot \big\|\bphi^{\top}(s)\big\|
    + \big|\bphi^{\top}(s) \bomega \big|\cdot \big\|\bphi^{\top}(s)\big\| \\
    & =
    U_{r} + U_{r} 
    + R_{\bomega}
    + R_{\bomega} \le 2U_{r} + 2 R_{\bomega}.
\end{align*}
For the second inequality, we can directly apply Cauchy-Schwartz inequality and obtain the result. For the third inequality, apply Cauchy-Schwartz inequality as we have
\begin{align*}
    \big|\Lambda(O, \bomega, \btheta) \big|
    & = 
    \Big|
    \big\la \bomega - \bomega^*,g(O, \bomega, \btheta) - \bar{g}( \bomega, \btheta) \big\ra
    \Big| \\
    & \le \norm{\bomega - \bomega^*} \cdot \big\|g(O, \bomega, \btheta) - \bar{g}( \bomega, \btheta)\big\| \\ 
    & \le 2 R_{\bomega} \cdot 2 U_{\delta} \le 2 U_{\delta}^2,
\end{align*}
which completes the proof.
\end{proof}

\subsection{Proof of Lemma \ref{lemma:critic-Markovian}}
This Lemma is actually a combination of several auxiliary lemmas listed here:

\begin{lemma} \label{lemma:critic-term1}
For any $\btheta_1, \btheta_2$, $\bomega$ and tuple $O=(s,a,s')$,
\begin{align*}
    \big|\Lambda(O, \bomega, \btheta_1) - \Lambda(O, \bomega, \btheta_2)\big| & \le K_1 \norm{\btheta_1 - \btheta_2},
\end{align*}
where $K_1 = 2U_{\delta}^2 |\cA| L (1 + \lceil \log_{\rho} m ^{-1} \rceil + 1/(1-\rho) ) + 2U_{\delta} L_{*}$.
\end{lemma}

\begin{lemma} \label{lemma:critic-term2}
For any $\btheta$, $\bomega_1, \bomega_2$ and tuple $O=(s,a,s')$, 
\begin{align*}
    \big|\Lambda(O, \bomega_1, \btheta) - \Lambda(O, \bomega_2, \btheta) \big| & \le 6 U_{\delta} \norm{\bomega_1 - \bomega_2}.
\end{align*}
\end{lemma}

\begin{lemma} \label{lemma:critic-term3}
Consider original tuples $O_t = (s_t,a_t,s_{t+1})$ and the auxiliary tuples $\tilde{O}_t = (\tilde{s}_t, \tilde{a}_t, \tilde{s}_{t+1})$.
Conditioned on $s_{t-\tau+1}$ and $\btheta_{t - \tau}$, we have
\begin{align}
    \EE [\Lambda(O_t, \bomega_{t-\tau}, \btheta_{t-\tau})
    -
    \Lambda(\tilde{O}_t, \bomega_{t-\tau}, \btheta_{t-\tau})]
    \le 
    U_{\delta}^2 |\cA| L \sum_{i=t-\tau}^{t} \EE \norm{\btheta_i - \btheta_{t-\tau}}
\end{align}
\end{lemma}
\begin{lemma} \label{lemma:critic-term4}
Conditioned on $s_{t-\tau+1}$ and $\btheta_{t-\tau}$,
\begin{align*}
    \EE [\Lambda(\tilde{O}_t, \bomega_{t-\tau}, \btheta_{t-\tau})]
    \le 
    2 U_{\delta}^2 m \rho^{\tau-1}.
\end{align*}
\end{lemma}

\begin{proof} [Proof of Lemma \ref{lemma:critic-Markovian}]
By the Lemma \ref{lemma:critic-term1}, \ref{lemma:critic-term2}, \ref{lemma:critic-term3} and \ref{lemma:critic-term4}, we can collect the corresponding term and  get the bound
\begin{align*}
    \EE[\Lambda(O_t, \bomega_t, \btheta_t)]
   & = 
    \EE[\Lambda(O_t, \bomega_t, \btheta_t) - \Lambda(O_t, \bomega_t, \btheta_{t-\tau})] 
    + \EE[\Lambda(O_t, \bomega_t, \btheta_{t-\tau}) - \Lambda(O_t, \bomega_{t-\tau}, \btheta_{t-\tau})] \\
    &\qquad+
    \EE[\Lambda(O_t, \bomega_{t-\tau}, \btheta_{t-\tau}) - \Lambda(\tilde{O}_t, \bomega_{t-\tau}, \btheta_{t-\tau})]
    + 
    \EE[\Lambda(\tilde{O}_t, \bomega_{t-\tau}, \btheta_{t-\tau})] \\
    &\le  
    C_{1}(\tau + 1) \norm{\btheta_{t} - \btheta_{t-\tau}} + C_{2} m \rho^{\tau - 1} + C_{3} \norm{\bomega_t - \bomega_{t-\tau}},
\end{align*}
where $C_{1} = 2 U_{\delta}^2 |\cA| L (1 + \lceil \log_{\rho} m^{-1}\rceil + 1/(1- \rho) ) + 2 U_{\delta} L_{*}, C_{2} = 2 U_{\delta}^2, C_{3} = 4 U_{\delta}$.
\end{proof}

\section{Proof of Auxiliary Lemmas}
\subsection{Proof of Lemma \ref{lemma:Gamma-term1}}
\begin{proof}[{Proof of Lemma \ref{lemma:Gamma-term1}}]
Denote $\delta(O_t, \btheta) := r(s_t, a_t) - r(\btheta) + ( \bphi(s_{t+1}) - \bphi(s_t))^{\top} \bomega^{*} $ and we have $h(O_t, \btheta) = \delta(O_t, \btheta) \nabla \log \pi_{\btheta}(a_t|s_t)$.  It can be shown that $\delta(O_t,\btheta_1) - \delta(O_t, \btheta_2) = (\bphi(s_{t+1}) - \bphi(s_t))^{\top}(\bomega^*_1 - \bomega^*_2) - (r(\btheta_1) - r(\btheta_2))$.

Denote $O_t = (s_t,a_t,s_{t+1})$, we have for any $\btheta_1$ and $\btheta_2$, that
\begin{align*}
    \Gamma(O, \btheta_{1}) - \Gamma(O, \btheta_{2})
    & = 
    \big \la \nabla J(\btheta_1), h(O,\btheta_1) - \EE_{\btheta_1}
    \big[ 
    h(O', \btheta_1)
    \big]
    \big \ra
    \\
    & \qquad 
    -
    \big \la \nabla J(\btheta_2), h(O,\btheta_2) - \EE_{\btheta_2}
    \big[ 
    h(O', \btheta_2)
    \big]
    \big \ra,
\end{align*}
where we use shorthand $\EE_{\btheta}$ to denote that $O' = (s,a,s')$ is drawn from $s \sim \mu_{\btheta}, a \sim \pi_{\btheta}, s' \sim \cP(\cdot|s,a)$.
We first exhibit each term here is Lipschitz. We have by Lemma~\ref{lemma:J-smooth} that,
\begin{align*}
    \| \nabla J(\btheta_1) - \nabla J(\btheta_2) \|
    & \le 
    L_J \| \btheta_1 - \btheta_2 \|.
\end{align*}
For $h(O_t, \btheta_1)$ and $h(O_t, \btheta_2)$, we have
\begin{align*}
    & \big\|h(O_t, \btheta_{1}) - h(O_t, \btheta_{2})\big\| \\
    & = 
    \big\|\delta(O_t,\btheta_{1}) \nabla \log \pi_{\btheta_1}(a_t|s_t) - \delta(O_t,\btheta_{2}) \nabla \log \pi_{\btheta_{2}}(a_t|s_t)\big\| \\
    & \le 
    \underbrace{
    \big\|\delta(O_t,\btheta_{1}) \nabla \log \pi_{\btheta_1}(a_t|s_t) - \delta(O_t,\btheta_{1}) \nabla \log \pi_{\btheta_{2}}(a_t|s_t)\big\|}_{I_1} \\
    &\qquad + 
    \underbrace{
    \big\|\delta(O_t,\btheta_{1}) \nabla \log \pi_{\btheta_{2}}(a_t|s_t) - \delta(O_t,\btheta_{2}) \nabla \log \pi_{\btheta_{2}}(a_t|s_t)\big\|}_{I_2}
    \\
    & \le 
    U_{\delta} L_{l} \|\btheta_1 - \btheta_2 \| 
    + 
    \underbrace{
    \big\|\delta(O_t,\btheta_{1}) \nabla \log \pi_{\btheta_{2}}(a_t|s_t) - \delta(O_t,\btheta_{2}) \nabla \log \pi_{\btheta_{2}}(a_t|s_t)\big\|}_{I_2},
\end{align*}
where the first inequality is due to the triangle inequality. The term $I_1$ is easily bounded by the fact that $\delta$ is bounded (see Section~\ref{subsec:proof-of-lemma-actor-markovian}) and Assumption~\ref{assum:policy-lipschitz-bounded}. For $I_2$, we have
\begin{align*}
    I_2 & \le 
    B | \delta(O_t, \btheta_1) - \delta(O_t, \btheta_2) |
    \\
    & \le 
    B \Big(
    \big| r(\btheta_1) - r(\btheta_2) \big|
    +
    \big\| \bphi(s_{t+1}) - \bphi(s_{t}) \big\|
    \cdot 
    \big\| \bomega^*(\btheta_1) - \bomega^*(\btheta_2) \big\|
    \Big), 
\end{align*}
where the first inequality is due to Assumption~\ref{assum:policy-lipschitz-bounded}, and the second is by unrolling the definition of $\delta(O_t,\btheta)$ and invoking the triangle inequality, among them, we know $\bphi$ is within the unit ball and $\bomega^*$ is $L_*$-Lipschitz by Proposition~\ref{prop:optimal-lipschitz} with $L_* := (2\lambda^{-2} U_r  + 3\lambda^{-1} U_r) |\cA|L (1 + \lceil \log_{\rho}m^{-1} \rceil + 1/(1-\rho) )$.

For $|r(\btheta_1) - r(\btheta_2)|$, we have that
\begin{align*}
    \big|r(\btheta_1) - r(\btheta_2)| & =
    |\EE_{s \sim \mu_{\btheta_1}, a \sim \pi_{\btheta_1}}[r(s,a)] - \EE_{s \sim \mu_{\btheta_2}, a \sim \pi_{\btheta_2}}[r(s,a)]
    \big| \\
    & \le 2U_r d_{TV}(\mu_{\btheta_1} \otimes \pi_{\btheta_1}, \mu_{\btheta_2}\otimes \pi_{\btheta_2}) \\
    & \le 
    2U_r|\cA|L \bigg(1 + \lceil \log_{\rho}m^{-1} \rceil + \frac{1}{1-\rho}\bigg) \norm{\btheta_1 - \btheta_2},
\end{align*}
where the first inequality is by the definition of the total-variation distance, and the second inequality is from Lemma~\ref{lemma:prob-mixing}. To summarize, we have
\begin{align*}
    & \big\|h(O_t, \btheta_{1}) - h(O_t, \btheta_{2})\big\| \\
    & \le 
    \bigg[ 
    U_{\delta} L_{l}
    +
    (2+ 2 \lambda^{-2} + 3 \lambda^{-1}) B U_r|\cA|L \bigg(1 + \lceil \log_{\rho}m^{-1} \rceil + \frac{1}{1-\rho}\bigg)
    \bigg] 
    \cdot \| \btheta_1 - \btheta_2 \|
    \\
    & =
    L_{h}\| \btheta_1 - \btheta_2 \|,
\end{align*}
where $L_h$ denotes the coefficient above.

Similarly, for $\EE_{\btheta_1}[ h(O', \btheta_1)]$ and $\EE_{\btheta_2}[ h(O', \btheta_2)]$, we have first
\begin{align*}
    & \| \EE_{\btheta_1}[ h(O', \btheta_1)]
    -
    \EE_{\btheta_2}[ h(O', \btheta_2)] \| \\
    & \le 
    \| \EE_{\btheta_1}[ h(O', \btheta_1)]
    -
    \EE_{\btheta_1}[ h(O', \btheta_2)] \|
    +
    \| \EE_{\btheta_1}[ h(O', \btheta_2)]
    -
    \EE_{\btheta_2}[ h(O', \btheta_2)] \|
    \\
    & \le 
    \EE_{\btheta_1} [ \| h(O', \btheta_1)
    -
    h(O', \btheta_2) \| ]
    +
    \| \EE_{\btheta_1}[ h(O', \btheta_2)]
    -
    \EE_{\btheta_2}[ h(O', \btheta_2)] \|
    \\
    & \le 
    L_{h} \|\btheta_1 - \btheta_2 \|
    +
    \| \EE_{\btheta_1}[ h(O', \btheta_2)]
    -
    \EE_{\btheta_2}[ h(O', \btheta_2)] \|
    \\
    & \le 
    L_{h} \|\btheta_1 - \btheta_2 \|
    +
    2 U_r B d_{TV}(\mu_{\btheta_1} \otimes \pi_{\btheta_1}, \mu_{\btheta_2}\otimes \pi_{\btheta_2})
    \\
    & \le 
    \bigg[ L_{h} 
    +
    2 U_r B
    |\cA|L \bigg(1 + \lceil \log_{\rho}m^{-1} \rceil + \frac{1}{1-\rho}\bigg) \bigg] 
    \|\btheta_1 - \btheta_2 \|
    \\
    & \le 2L_h \|\btheta_1 - \btheta_2 \|,
\end{align*}
where the first inequality is due to the triangle inequality; the second one is due to the convexity of $\|\cdot\|$ norm; the third inequality is from the Lipschitz-ness of $h(O,\btheta)$ we just showed above; the fourth one is due to the property of the total variation distance; the fifth one is due to Proposition~\ref{prop:optimal-lipschitz}. The last inequality is just to absorb the coefficient into $L_h$ for less notation clutter.

So far, we have proved the Lipschitz-ness of all the terms in $\Gamma(O, \btheta_1) - \Gamma(O, \btheta_2)$. We can also show that each term is bounded: from \eqref{eqn:G-theta} in Section~\ref{subsec:proof-of-lemma-actor-markovian}, we can see that $\nabla J(\btheta)$ is $G_{\btheta}$-bounded and also $h(O,\btheta) - \EE_{\btheta}[h(O', \btheta)]$ is $2U_\delta B$-bounded since $h(O, \btheta)$ is bounded by $U_\delta B$ for any $O$ and $\btheta$.

To sum up, $\nabla J(\btheta)$ is $G_{\btheta}$-bounded and $L_J$-Lipschitz; $h(O,\btheta) - \EE_{\btheta}[h(O', \btheta)]$ is $3L_h$-Lipschitz and $2U_{\delta}B$-bounded.
By the triangle inequality, we have
\begin{align*}
    \Gamma(O_t, \btheta_{t}) - \Gamma(O_t, \btheta_{t-\tau})
    & = 
    \big \la \nabla J(\btheta_t) - \nabla J(\btheta_{t-\tau}), h(O_t,\btheta_t) - \EE_{\btheta_t}
    \big[ 
    h(O', \btheta_t)
    \big]
    \big \ra
    \\ & \qquad 
    +
    \big \la \nabla J(\btheta_{t-\tau}), 
    \big(h(O_t,\btheta_t) - \EE_{\btheta_t}
    [ 
    h(O', \btheta_t)
    ]
    \big)\\
    &\qquad -
    \big(
    h(O_t,\btheta_{t-\tau}) - \EE_{\btheta_{t-\tau}}
    [ 
    h(O', \btheta_{t-\tau})
    ]
    \big)
    \big \ra
    \\
    & \le 
    (2 U_{\delta} B L_J 
    +
    3G_{\btheta} L_h)\norm{\btheta_{t} - \btheta_{t-\tau}}.
\end{align*} 
This completes the proof.
\end{proof}

\subsection{Proof of Lemma \ref{lemma:Gamma-term2}}
\begin{proof}[Proof of Lemma \ref{lemma:Gamma-term2}]
By the definition of $\Gamma(O, \btheta)$ in \eqref{def:actor-term},
\begin{align} \label{eqn:Gamma-term2}
    \EE \big[\Gamma(O_t, \btheta_{t-\tau}) - \Gamma(\tilde{O}_t, \btheta_{t-\tau}) \big]
    & =
    \EE\big[ \big\la\nabla J(\btheta_{t-\tau}),h(O_t, \btheta_{t-\tau}) - h(\tilde{O}_t, \btheta_{t-\tau})\big\ra \big]
    \notag\\
    & =
    \EE\Big[ 
    \big\la\nabla J(\btheta_{t-\tau}),h(O_t, \btheta_{t-\tau})\big\ra
    -
    \big\la\nabla J(\btheta_{t-\tau}), h(\tilde{O}_t, \btheta_{t-\tau})\big\ra
    \Big]
    \notag\\
    & \le 
    4 U_{\delta} B G_{\btheta} d_{TV}\big(\PP(O_t = \cdot | s_{t- \tau +1}, \btheta_{t- \tau}), \PP(\tilde{O}_t = \cdot | s_{t- \tau +1}, \btheta_{t- \tau})\big),
\end{align}
where the inequality is by the definition of total variation. By Lemma  \ref{lemma:auxiliary-chain} we have 
\begin{align*}
    &d_{TV}\big(\PP(O_t \in \cdot|s_{t- \tau +1}, \btheta_{t- \tau}), \PP(\tilde{O}_t \in \cdot|s_{t- \tau +1}, \btheta_{t- \tau})\big) \notag \\
    & = 
    d_{TV}\big(\PP((s_t,a_t) \in \cdot|s_{t- \tau +1}, \btheta_{t- \tau}), \PP((\tilde{s}_t, \tilde{a}_t) \in \cdot|s_{t- \tau +1}, \btheta_{t- \tau})\big)
    \notag \\
    & \le 
    d_{TV}\big(\PP(s_t \in \cdot|s_{t- \tau +1}, \btheta_{t- \tau}), \PP(\tilde{s}_t \in \cdot|s_{t- \tau +1}, \btheta_{t- \tau})\big)
    +
    \frac{1}{2} |\cA|L \EE \norm{\btheta_{t} - \btheta_{t-\tau}}
    \notag \\
    & \le 
    d_{TV}\big(\PP(O_{t-1} \in \cdot|s_{t- \tau +1}, \btheta_{t- \tau}), \PP(\tilde{O}_{t-1} \in \cdot|s_{t- \tau +1}, \btheta_{t- \tau})\big)
    +
    \frac{1}{2} |\cA|L \EE \norm{\btheta_{t} - \btheta_{t-\tau}}. 
\end{align*}
Repeat the inequality above over $t$ to $t-\tau+1$ we have 
\begin{align}
    d_{TV}\big(\PP(O_t \in \cdot|s_{t- \tau +1}, \btheta_{t- \tau}), \PP(\tilde{O}_t \in \cdot|s_{t- \tau +1}, \btheta_{t- \tau})\big)
    \le
    \frac{1}{2} |\cA| L 
    \sum_{i=t-\tau}^{t} \EE \norm{\btheta_i - \btheta_{t-\tau}}. \label{eqn:DVO}
\end{align}
Plugging \eqref{eqn:DVO} into \eqref{eqn:Gamma-term2} we get
\begin{align*}
    \EE\big[\Gamma(O_t, \btheta_{t-\tau}) - \Gamma(\tilde{O}_t, \btheta_{t-\tau})\big]
    & \le
    2U_{\delta} B G_{\btheta} |\cA| L \sum_{i=t-\tau}^{t} \norm{\btheta_i - \btheta_{t-\tau}}.
\end{align*}

\end{proof}

\subsection{Proof of Lemma \ref{lemma:Gamma-term3}}

\begin{proof}[Proof of Lemma \ref{lemma:Gamma-term3}]
\begin{align*}
    \EE \big[\Gamma\big(\tilde{O}_t, \btheta_{t-\tau}\big) - \Gamma(O'_t, \btheta_{t-\tau})\big]
    & = 
    \EE \big[
    \big \la 
    \nabla J (\btheta_{t-\tau}),
    h(\tilde{O}_t, \btheta_{t-\tau})
    -
    h(O'_t, \btheta_{t-\tau})
    \big \ra
    \big]
    \\
    & = 
    \EE \big[
    \big \la 
    \nabla J (\btheta_{t-\tau}),
    h(\tilde{O}_t, \btheta_{t-\tau})
    \big \ra
    -
    \big \la 
    \nabla J (\btheta_{t-\tau}),
    h(O'_t, \btheta_{t-\tau})
    \big \ra
    \big]
    \\
    & \le
    4U_{\delta}B G_{\btheta}
    d_{TV}\big(\PP(\tilde{O}_{t}= \cdot | s_{t- \tau +1}, \btheta_{t- \tau}) , \mu_{\btheta_{t-\tau}} \otimes \pi_{\btheta_{t-\tau}} \otimes \cP\big) \\ 
    & \le 
    4U_{\delta}B G_{\btheta} m \rho^{\tau - 1}.
\end{align*}
The first inequality is by the definition of total variation norm and the second inequality holds because, by the ergodicity in  Assumption \ref{assum:ergodicity}, it holds that 
\begin{align*}
    d_{TV}\big(\PP(\tilde{s}_{t}= \cdot | s_{t- \tau +1}, \btheta_{t- \tau}) , \mu_{\btheta_{t-\tau}}\big) \le m \rho^{\tau - 1},
\end{align*}
and thus
\begin{align*}
     &d_{TV}\big(\PP(\tilde{O}_{t}= \cdot | s_{t- \tau +1}, \btheta_{t- \tau}) , \mu_{\btheta_{t-\tau}} \otimes \pi_{\btheta_{t-\tau}} \otimes \cP \big) \\
   &= 
    d_{TV}\big(\PP((\tilde{s}_t, \tilde{a}_t)= \cdot | s_{t- \tau +1}, \btheta_{t- \tau}) , \mu_{\btheta_{t-\tau}} \otimes \pi_{\btheta_{t-\tau}} \big) \\
    &= 
    d_{TV}\big(\PP(\tilde{s}_{t}= \cdot | s_{t- \tau +1}, \btheta_{t- \tau}) , \mu_{\btheta_{t-\tau}} \big)
    \\
    & \le m \rho^{\tau - 1}.
\end{align*}
The equations above are derived following the same procedure in \eqref{eqn:peeling}, because $\PP(\tilde{O}_{t}= \cdot | s_{t- \tau +1}, \btheta_{t- \tau}) = \PP(\tilde{s}_{t}= \cdot | s_{t- \tau +1}, \btheta_{t- \tau}) \otimes \pi_{\btheta_{t-\tau}} \otimes \cP$.
\end{proof}

\subsection{Proof of Lemma \ref{lemma:eta-term1}}
\begin{proof} [Proof of Lemma \ref{lemma:eta-term1}]
By the definition of $\Xi(O, \eta, \btheta)$ in \eqref{eq:def_notation_Ot}, we have
\begin{align*}
    \big|\Xi(O, \eta, \btheta_1) - \Xi(O, \eta, \btheta_2)\big|
    & = 
    \big|(\eta - \eta_1^*)(r - \eta_1^*) - (\eta - \eta_2^*)(r - \eta_2^*) \big| \\
    & \le 
    \big|(\eta - \eta_1^*)(r - \eta_1^*) - (\eta - \eta_1^*)(r - \eta_2^*) \big| \\
    &\qquad+
    \big|(\eta - \eta_1^*)(r - \eta_2^*) - (\eta - \eta_2^*)(r - \eta_2^*) \big| \\ 
    & \le 
    4 U_r |\eta_1^* - \eta_2^*| \\ 
    & = 
    4 U_r \big|J(\btheta_1) - J(\btheta_2) \big| \\ 
    & \le 
    4 U_r  C_J \norm{\btheta_1 - \btheta_2}.
\end{align*}
\end{proof}

\subsection{Proof of Lemma \ref{lemma:eta-term2}}
\begin{proof}[Proof of Lemma \ref{lemma:eta-term2}]
By definition,
\begin{align*}
    \big|\Xi(O, \eta_1, \btheta) - \Xi(O, \eta_2, \btheta) \big| 
    & = 
    \big|
    (\eta_1 - \eta^*)( r - \eta^*)
    -
    (\eta_2 - \eta^*)( r - \eta^*)
    \big| \\
    & \le
    2 U_r |\eta_1 - \eta_2|.
\end{align*}
\end{proof}

\subsection{Proof of Lemma \ref{lemma:eta-term3}}

\begin{proof} [Proof of Lemma \ref{lemma:eta-term3}]
By the Cauchy-Schwartz inequality and the definition of total variation norm, we have
\begin{align*}
    \EE \big[ \Xi(O_t, \eta_{t-\tau}, \btheta_{t-\tau})
    -
    \Xi(\tilde{O}_t, \eta_{t-\tau}, \btheta_{t-\tau}) \big]
    & = 
    (\eta_{t-\tau} - \eta^*_{t-\tau})
    \EE [r(s_t, a_t) - r(\tilde{s}_t, \tilde{a_t}) 
    ].
\end{align*}
Since 
\begin{align*}
    \EE [r(s_t, a_t) - r(\tilde{s}_t, \tilde{a_t}) 
    ]
    & \le 
    2 U_r d_{TV}\big(\PP(O_t \in \cdot|s_{t- \tau +1}, \btheta_{t- \tau}), \PP(\tilde{O}_t \in \cdot|s_{t- \tau +1}, \btheta_{t- \tau})\big),
\end{align*}
the total variation between $O_t$ and $\tilde{O}_t$ has appeared in \eqref{eqn:DVO}, in the proof of Lemma \ref{lemma:Gamma-term2}, which is
\begin{align*}
    d_{TV}\big(\PP(O_t \in \cdot|s_{t- \tau +1}, \btheta_{t- \tau}), \PP(\tilde{O}_t \in \cdot|s_{t- \tau +1}, \btheta_{t- \tau})\big)
    \le
    \frac{1}{2} |\cA| L 
    \sum_{i=t-\tau}^{t} \EE \norm{\btheta_i - \btheta_{t-\tau}}. 
\end{align*}
Plugging this bound, we have 
\begin{align*}
    \big|\EE[\Xi(O_t, \eta_{t-\tau}, \btheta_{t-\tau})
    -
    \Xi(\tilde{O}_t, \eta_{t-\tau}, \btheta_{t-\tau})]\big|
    \le
    2 U_{r}^2 |\cA| L \sum_{i=t-\tau}^{t} \EE \norm{\btheta_i - \btheta_{t-\tau}}.
\end{align*}
\end{proof}

\subsection{Proof of Lemma \ref{lemma:eta-term4}}

\begin{proof}[Proof of Lemma \ref{lemma:eta-term4}]
We first note that according to the definition,
\begin{align*}
    \EE[ \Xi( O'_{t}, \eta_{t-\tau},\btheta_{t-\tau})| \btheta_{t- \tau}] = 0,
\end{align*}
 where $O'_{t} = (s'_t, a'_t, s'_{t+1})$ is the tuple generated by $s'_t \sim \mu_{\btheta_{t-\tau}}, a'_t \sim \pi_{\btheta_{t-\tau}}, s'_{t+1} \sim \cP$. By the ergodicity in  Assumption \ref{assum:ergodicity}, it holds that 
\begin{align*}
    d_{TV}\big(\PP(\tilde{s}_{t}= \cdot | s_{t- \tau +1}, \btheta_{t- \tau}) , \mu_{\btheta_{t-\tau}} \big) \le m \rho^{\tau - 1}.
\end{align*}
It can be shown that 
\begin{align*}
    \EE[\Xi(\tilde{O}_t, \eta_{t-\tau}, \btheta_{t-\tau}) ]
    &=
    \EE\big[\Xi\big(\tilde{O}_t, \eta_{t-\tau}, \btheta_{t-\tau}\big) -
    \Xi(O'_t, \eta_{t-\tau}, \btheta_{t-\tau})\big] \\
    & =
    \EE\big[
    (\eta_{t-\tau} - \eta^*_{t-\tau})
    \big(r(\tilde{s}_t, \tilde{a}_t) - r(s',a') \big)
    \big] 
    \\
    & \le 4 U_r^2  d_{TV}\big(\PP\big(\tilde{O}_{t}= \cdot | s_{t- \tau +1}, \btheta_{t- \tau}\big) , \mu_{\btheta_{t-\tau}} \otimes \pi_{\btheta_{t-\tau}} \otimes \cP \big) \\
    & \le 4 U_r^2  m \rho^{\tau - 1}.
\end{align*}
The argument used here also appears in the proof of Lemma \ref{lemma:critic-term4} and explained in detail there.
\end{proof}

\subsection{Proof of Lemma \ref{lemma:critic-term1}}
\begin{proof} [Proof of Lemma \ref{lemma:critic-term1}]
\begin{align*}
    & \big|\Lambda(O, \bomega, \btheta_1) - \Lambda(O, \bomega, \btheta_2) \big| \\
    & = 
    \Big|
    \big\la\bomega - \bomega_1^*,g(O, \bomega) - \bar g(\btheta_1, \bomega)\big\ra - \big\la \bomega - \bomega_2^*,g(O, \bomega) - \bar g(\btheta_2, \bomega)\big\ra \Big| \\
    & \le 
    \underbrace{
    \Big|
    \big\la \bomega - \bomega_1^*, g(O, \bomega) - \bar g(\btheta_1, \bomega)\big\ra - \big\la\bomega - \bomega_1^*,g(O, \bomega) - \bar g(\btheta_2, \bomega)\big\ra
    \Big|}_{I_1} \\
    &\qquad+
    \underbrace{\Big|
    \big\la \bomega - \bomega_1^*,g(O, \bomega) - \bar g(\btheta_2, \bomega)\big\ra - \big\la \bomega - \bomega_2^*,g(O, \bomega) - \bar g(\btheta_2, \bomega)\big \ra
    \Big|}_{I_2}.
\end{align*}
For the term $I_2$, we simply use the Cauchy-Schwartz inequality to get $2U_{\delta} \norm{\bomega^*_1 - \bomega^*_2}$. \\
For the term $I_1$, it can be bounded as:
\begin{align*}
    &
    \Big|
    \big\la \bomega - \bomega_1^*,g(O, \bomega) - \bar g(\btheta_1, \bomega)\big\ra - \big\la\bomega - \bomega_1^*,g(O, \bomega) - \bar g(\btheta_2, \bomega)\big\ra
    \Big| \\
    & = 
    \Big| \big\la\bomega - \bomega_1^*,\bar{g}(\btheta_1, \bomega) - \bar{g}(\btheta_2, \bomega)\big\ra \Big| \\
    &\le  
    2R_{\bomega} \big\|\bar{g}(\btheta_1, \bomega) - \bar{g}(\btheta_2, \bomega) \big\| \\
   & \le  
    2 R_{\bomega} \cdot  2 U_{\delta} \cdot d_{TV}(\mu_{\btheta_1} \otimes \pi_{\btheta_1} \otimes \cP, \mu_{\btheta_2}\otimes \pi_{\btheta_2} \otimes \cP) \\
    &\le 
    2 U_{\delta}^2 d_{TV}(\mu_{\btheta_1} \otimes \pi_{\btheta_1} \otimes \cP, \mu_{\btheta_2}\otimes \pi_{\btheta_2} \otimes \cP),
\end{align*}
where the first inequality is due to Cauchy-Schwartz; the second inequality is by the definition of total variation norm; the third inequality is due to the fact $U_{\delta} \ge 2 R_{\bomega}$. Therefore, we have
\begin{align*}
    \big|\Lambda(\btheta_1, \bomega, O) - \Lambda(\btheta_2, \bomega, O) \big|
    & \le 2 U_{\delta}^2 d_{TV}(\mu_{\btheta_1} \otimes \pi_{\btheta_1} \otimes \cP, \mu_{\btheta_2}\otimes \pi_{\btheta_2} \otimes \cP) + 2 U_{\delta} \norm{\bomega_1^* - \bomega_2^*} \\
    & \le 
    2 U_{\delta}^2 |\cA|L \bigg(1 + \lceil \log_{\rho}m^{-1} \rceil + \frac{1}{1-\rho} \bigg) \norm{\btheta_1 - \btheta_2} 
    +
    2 U_{\delta} L_{*} \norm{\btheta_1 - \btheta_2} \\
    & = K_1 \norm{\btheta_1 - \btheta_2},
\end{align*}
where the second inequality is due to Lemma \ref{lemma:prob-mixing} and Proposition \ref{prop:optimal-lipschitz}.
\end{proof}

\subsection{Proof of Lemma \ref{lemma:critic-term2}}

\begin{proof}[Proof of Lemma \ref{lemma:critic-term2}]
By definition,
\begin{align*}
    &\big|\Lambda(O, \bomega_1, \btheta) - \Lambda(O, \bomega_2, \btheta) \big| \\
    & = 
    \Big|
    \big\la \bomega_1 - \bomega^*,  g(O,\bomega_1) - \bar{g}(\bomega_1, \btheta )\big\ra
    -
    \big\la \bomega_2 - \bomega^*,  g(O,\bomega_2) - \bar{g}(\bomega_2, \btheta)\big\ra
    \Big| \\
    & \le
    \Big|
    \big\la \bomega_1 - \bomega^*,  g(O,\bomega_1) - \bar{g}(\bomega_1, \btheta )\big\ra
    -
    \big\la \bomega_1 - \bomega^*,  g(O,\bomega_2) - \bar{g}(\bomega_2, \btheta )\big\ra
    \Big| \\
    & \qquad+ 
    \Big|
    \big\la \bomega_1 - \bomega^*,  g(O,\bomega_2) - \bar{g}(\bomega_2, \btheta )\big\ra
    -
    \big\la \bomega_2 - \bomega^*,  g(O,\bomega_2) - \bar{g}(\bomega_2, \btheta )\big\ra
    \Big| \\
    & \le 
    2 R_{\bomega} \Big\| \big( g(O, \bomega_1) - g(O, \bomega_2) \big)
    -
    \big( \bar{g}(\bomega_1, \btheta ) - \bar{g}(\bomega_2, \btheta ) \big) \Big\|
    + 
    2 U_{\delta} \norm{\bomega_1 - \bomega_2}.
\end{align*}
Note that we have $\norm{g(O, \bomega_1, \btheta) - g(O, \bomega_2, \btheta)} = |( \bphi(s') - \bphi(s) )^{\top} (\bomega_1 - \bomega_2)| \le 2\norm{\bomega_1 - \bomega_2} $ and similarly $\norm{\bar{g}(\bomega_1, \btheta ) - \bar{g}(\bomega_2, \btheta )} \le |\EE \big[ ( \bphi(s') - \bphi(s) )^{\top} (\bomega_1 - \bomega_2) \big] | \le 2 \norm{\bomega_1 - \bomega_2}$. Therefore,
\begin{align*}
    & \big|\Lambda(O, \bomega_1, \btheta) - \Lambda(O, \bomega_2, \btheta) \big|  \\
    & \le 
    2 R_{\bomega} \Big\|\big( g(O, \bomega_1) - g(O, \bomega_2) \big)
    -
    \big( \bar{g}(\bomega_1, \btheta ) - \bar{g}(\bomega_2, \btheta ) \big) \Big\|
    + 
    2 U_{\delta} \norm{\bomega_1 - \bomega_2} \\
    & \le 
    2 R_{\bomega} \cdot 4 \norm{\bomega_1  -\bomega_2} 
    +
    2 U_{\delta} \norm{\bomega_1  -\bomega_2} \\
    & \le 
    6 U_{\delta} \norm{\bomega_1  -\bomega_2}.
\end{align*}
\end{proof}

\subsection{Proof of Lemma \ref{lemma:critic-term3}}

\begin{proof} [Proof of Lemma \ref{lemma:critic-term3}]
By the Cauchy-Schwartz inequality and the definition of total variation norm, we have
\begin{align}
     &\EE[ \Lambda(O_t, \bomega_{t-\tau}, \btheta_{t-\tau})
    -
    \Lambda(\tilde{O}_t, \bomega_{t-\tau}, \btheta_{t-\tau}) ] \notag \\
    &= 
    \EE \big[ 
    \big\la\bomega_{t-\tau} - \bomega_{t-\tau}^{*},g(O_t, \bomega_{t-\tau}) - g(\tilde{O}_t, \bomega_{t-\tau})\big\ra
    \big]
    \notag\\
    &\le 
    2U_{\delta}^{2} d_{TV}\big(\PP(O_t \in \cdot|s_{t- \tau +1}, \btheta_{t- \tau}), \PP(\tilde{O}_t \in \cdot|s_{t- \tau +1}, \btheta_{t- \tau})\big).
    \label{eqn:aux-gap}
\end{align}
The total variation between $O_t$ and $\tilde{O}_t$ has appeared in \eqref{eqn:DVO}, in the proof of Lemma \ref{lemma:Gamma-term2}, which is
\begin{align*}
    d_{TV}\big(\PP(O_t \in \cdot|s_{t- \tau +1}, \btheta_{t- \tau}), \PP(\tilde{O}_t \in \cdot|s_{t- \tau +1}, \btheta_{t- \tau})\big)
    \le
    \frac{1}{2} |\cA| L 
    \sum_{i=t-\tau}^{t} \EE \norm{\btheta_i - \btheta_{t-\tau}}. 
\end{align*}
Plugging this bound into \eqref{eqn:aux-gap}, we have 
\begin{align*}
    \EE \big|\Lambda(O_t, \bomega_{t-\tau}, \btheta_{t-\tau})
    -
    \Lambda(\tilde{O}_t, \bomega_{t-\tau}, \btheta_{t-\tau})\big|
    \le
    U_{\delta}^2 |\cA| L \sum_{i=t-\tau}^{t} \EE \norm{\btheta_i - \btheta_{t-\tau}}.
\end{align*}
\end{proof}

\subsection{Proof of Lemma \ref{lemma:critic-term4}}

\begin{proof}[Proof of Lemma \ref{lemma:critic-term4}]
We first note that according to the definition in Section~\ref{subsec:proof-critic}, 
\begin{align*}
    \EE[ \Lambda( O'_{t}, \bomega_{t-\tau},\btheta_{t-\tau})| s_{t- \tau +1}, \btheta_{t- \tau}] = 0,
\end{align*}
 where $O'_{t} = (s'_t, a'_t, s'_{t+1})$ is the tuple generated by $s'_t \sim \mu_{\btheta_{t-\tau}}, a'_t \sim \pi_{\btheta_{t-\tau}}, s'_{t+1} \sim \cP$. By the ergodicity in  Assumption \ref{assum:ergodicity}, it holds that 
\begin{align*}
    d_{TV}\big(\PP(\tilde{s}_{t}= \cdot | s_{t- \tau +1}, \btheta_{t- \tau}) , \mu_{\btheta_{t-\tau}}\big) \le m \rho^{\tau - 1}.
\end{align*}
It can be shown that 
\begin{align*}
    \EE[\Lambda(\tilde{O}_t, \bomega_{t-\tau}, \btheta_{t-\tau}) ]
    &=
    \EE[\Lambda(\tilde{O}_t, \bomega_{t-\tau}, \btheta_{t-\tau}) -
    \Lambda(O'_t, \bomega_{t-\tau}, \btheta_{t-\tau})] \\
    & =
    \EE
    \big\la\bomega_{t-\tau} - \bomega^{*}_{t-\tau}, g(\tilde{O}_t, \bomega_{t- \tau}) - g(O'_t, \bomega_{t- \tau})\big\ra \\
    & \le 4 R_{\bomega} U_{\delta} d_{TV}\big(\PP(\tilde{O}_{t}= \cdot | s_{t- \tau +1}, \btheta_{t- \tau}) , \mu_{\btheta_{t-\tau}} \otimes \pi_{\btheta_{t-\tau}} \otimes \cP \big) \\
    & \le 2 U_{\delta}^2 d_{TV}\big(\PP(\tilde{s}_{t}= \cdot | s_{t- \tau +1}, \btheta_{t- \tau}) , \mu_{\btheta_{t-\tau}}\big) \\
    & \le 2 U_{\delta}^2 m \rho^{\tau - 1}.
\end{align*}
The third inequality holds because
$2R_{\bomega} < U_{\delta}$ and 
\begin{align*}
    & d_{TV}\big(\PP(\tilde{O}_{t}= \cdot | s_{t- \tau +1}, \btheta_{t- \tau}) , \mu_{\btheta_{t-\tau}} \otimes \pi_{\btheta_{t-\tau}} \otimes \cP \big) \\
   &= 
    d_{TV}\big(\PP((\tilde{s}_t, \tilde{a}_t)= \cdot | s_{t- \tau +1}, \btheta_{t- \tau}) , \mu_{\btheta_{t-\tau}} \otimes \pi_{\btheta_{t-\tau}} \big) \\
    &= 
    d_{TV}\big(\PP(\tilde{s}_{t}= \cdot | s_{t- \tau +1}, \btheta_{t- \tau}) , \mu_{\btheta_{t-\tau}} \big).
\end{align*}
This can be shown following the same procedure in \eqref{eqn:peeling}, because $\PP(\tilde{O}_{t}= \cdot | s_{t- \tau +1}, \btheta_{t- \tau}) = \PP(\tilde{s}_{t}= \cdot | s_{t- \tau +1}, \btheta_{t- \tau}) \otimes \pi_{\btheta_{t-\tau}} \otimes \cP$.
\end{proof}

\end{document}